%% file: main.tex
\documentclass{article}

\usepackage[preprint]{neurips_2024}

\usepackage[utf8]{inputenc} 
\usepackage[T1]{fontenc}    
\usepackage{hyperref}       
\usepackage{url}            
\usepackage{booktabs}       
\usepackage{amsfonts}       
\usepackage{nicefrac}       
\usepackage{microtype}      
\usepackage{makecell}
\usepackage{multirow}

\usepackage{amsmath}

\usepackage{wrapfig}
\usepackage[ruled,vlined,linesnumbered]{algorithm2e}
\usepackage{amsmath}
\usepackage{hyperref}
\usepackage{amsthm}
\usepackage{mathtools}
\usepackage{bm}
\usepackage{graphicx}
\usepackage{makecell}
\usepackage{caption}
\usepackage{subcaption}
\usepackage{float}
\usepackage{comment}

\usepackage{enumitem}
\usepackage{array}
\usepackage{dsfont}
\usepackage{makecell}
\usepackage{multirow}
\usepackage{tabularx}
\usepackage{flafter}

\usepackage{listings}
\usepackage{pbox}
\usepackage{chngpage}
\usepackage{minitoc}
\usepackage[toc,page,header]{appendix}
\usepackage[flushleft]{threeparttable}

\newtheorem*{rep@theorem}{\rep@title}
\newcommand{\newreptheorem}[2]{%
\newenvironment{rep#1}[1]{%
 \def\rep@title{#2 \ref{##1}}%
 \begin{rep@theorem}}%
 {\end{rep@theorem}}}
\makeatother

\newreptheorem{definition}{Definition}
\newtheorem{thm}{Theorem}

\newreptheorem{thm}{Theorem}
\newreptheorem{lemma}{Lemma}

\usepackage[dvipsnames]{xcolor}

\usepackage{pifont}

\definecolor{cycle2}{RGB}{106, 191, 0}
\definecolor{cycle3}{RGB}{191, 0, 0}

\definecolor{amber}{rgb}{1.0, 0.75, 0.0}
\definecolor{awesome}{rgb}{1.0, 0.13, 0.32}
\definecolor{ao(english)}{rgb}{0.0, 0.5, 0.0}

\newcommand{\cmark}{\textcolor{cycle2}{\ding{52}}}
\newcommand{\xmark}{\textcolor{cycle3}{\ding{56}}}

\input{preamble}

\title{Large Graph Generative Models\vspace{-1ex}}

\author{
Yu Wang$^1$, Ryan A. Rossi$^3$, Namyong Park, Huiyuan Chen, Nesreen K. Ahmed$^4$, \\ \textbf{Puja Trivedi$^2$, Franck Dernoncourt$^3$, Danai Koutra$^2$, Tyler Derr$^1$}
\\
$^1$Vanderbilt University ~~~ $^2$University of Michigan ~~~ $^3$Adobe Research ~~~  $^4$ Intel Labs.\\
\footnotesize\texttt{\{yu.wang.1,tyler.derr\}@vanderbilt.edu}, \footnotesize\texttt{\{ryrossi,dernonco\}@adobe.com}, \\
\footnotesize\texttt{\{pujat,dkoutra\}@umich.edu},\\
\footnotesize\texttt{park.namyong@gmail.com, nesreen.k.ahmed@intel.com, hxc501@case.edu},
}

\begin{document}
\maketitle
\doparttoc
\faketableofcontents

\maketitle

\vspace{-3ex}
\begin{abstract}
Large Generative Models (LGMs) such as GPT, Stable Diffusion, Sora, and Suno are trained on a huge amount of language corpus, images, videos, and audio that are extremely diverse from numerous domains. This training paradigm over diverse well-curated data lies at the heart of generating creative and sensible content. However, all previous graph generative models (e.g., GraphRNN, MDVAE, MoFlow, GDSS, and DiGress) have been trained only on one dataset each time, which cannot replicate the revolutionary success achieved by LGMs in other fields. To remedy this crucial gap, we propose a new class of graph generative model called \textsc{Large Graph Generative Model} (LGGM) that is trained on a large corpus of graphs (over 5000 graphs) from 13 different domains. We empirically demonstrate that the pre-trained LGGM has superior zero-shot generative capability to existing graph generative models. Furthermore, our pre-trained LGGM can be easily fine-tuned with graphs from target domains and demonstrate even better performance than those directly trained from scratch, behaving as a solid starting point for real-world customization. Inspired by Stable Diffusion, we further equip LGGM with the capability to generate graphs given text prompts (Text-to-Graph), such as the description of the network name and domain (i.e., "The power-1138-bus graph represents a network of buses in a power distribution system."), and network statistics (i.e., "The graph has a low average degree, suitable for modeling social media interactions."). This Text-to-Graph capability integrates the extensive world knowledge in the underlying language model, offering users fine-grained control of the generated graphs. We release the code, the model checkpoint, and the datasets at \href{https://lggm-lg.github.io/}{https://lggm-lg.github.io/}.
\end{abstract}

\input{introduction}

\input{relatedwork}

\input{method}

\input{experiment}

\input{conclusion}

\bibliographystyle{plain}
\bibliography{reference}

\appendix
\input{appendix}


\end{document}

%% file: preamble.tex
\usepackage{xcolor}
\definecolor{thedarkblue}{RGB}{0,0,120} 
\definecolor{mydarkblue}{rgb}{0,0.08,0.45} 
\definecolor{darkblue}{rgb}{0,0.08,180}
\colorlet{TufteRed}{red!80!black}
\definecolor{theblue}{RGB}{0,0,180}
\colorlet{thered}{TufteRed}
      
\usepackage{microtype}
\usepackage{balance}

\usepackage{booktabs}
\usepackage{tabularx}

\usepackage{amsmath,amssymb,amsthm}

\newcommand{\eat}[1]{\ignorespaces}
\usepackage{comment}

\newcommand{\journal}[1]{} 

\usepackage{tikz}
\usepackage{verbatim}
\usetikzlibrary{arrows}
\usetikzlibrary{shapes,snakes}
\usetikzlibrary{decorations.pathmorphing} 
\usetikzlibrary{fit}					
\usetikzlibrary{backgrounds}	

\usepackage{ragged2e}
\usepackage{multirow}
\usepackage{microtype}
\usepackage{balance}
\usepackage{setspace}

\graphicspath{{./}{./graphics/}}
\newcolumntype{H}{>{\setbox0=\hbox\bgroup}c<{\egroup}@{}}

\newcolumntype{R}[1]{>{\RaggedLeft\arraybackslash}} 
\newcolumntype{L}[1]{>{\RaggedRight\arraybackslash}} 

\AtBeginEnvironment{pmatrix}{\setlength{\arraycolsep}{2pt}}

\DeclareMathOperator*{\argmin}{argmin}

\DeclareMathOperator{\hugeE}{\mbox{\huge\raise-0.3ex\hbox{E}}}
\DeclareMathOperator{\p}{\mathbb{P}}
\DeclareMathOperator{\hugep}{\mbox{\huge\raise-0.3ex\hbox{$\p$}}}



%% file: introduction.tex
\section{Introduction}\label{sec-intro}
Recently, Large Generative Models (LGMs) such as GPT, Stable Diffusion, Sora, and Suno~\cite{achiam2023gpt, videoworldsimulators2024, rombach2022high, touvron2023llama} have achieved revolutionary success in generating creative and sensible content, which significantly increases the productivity of real-world applications~\cite{beard2009firefly, somepalli2023diffusion, zhang2024artbank}. Unlike previous models such as Bert/Bart~\cite{devlin2018bert, lewis2019bart} in Natural Language Processing (NLP) and Unet~\cite{ronneberger2015u} in Image Segmentation that are trained only on small-scale datasets from specific domains over narrow tasks, the key to the success of these LGMs lies in their large training paradigm over the well-curated training data from a wide variety of domains~\cite{bubeck2023sparks, devlin2018bert, radford2021learning, touvron2023llama}. Graph, as a data modality distinct from image, text, and audio, is ubiquitous across numerous fields and presents a new frontier for applications of generative models such as drug discovery~\cite{hoogeboom2022equivariant, igashov2024equivariant}, material design~\cite{liu2024data, menon2022generative} and cyber-security~\cite{gavrilev2023anomaly, liu2024graph}. Given the unprecedented success achieved by LGMs in other domains and the promising practical usage of graph generative models, we naturally ask: 

\begin{center}
\textbf{\textit{Can we propose large generative models for graph-structured data?}}
\end{center}

\begin{figure*}[htbp!]
     \centering
     \includegraphics[width=1.0\textwidth]{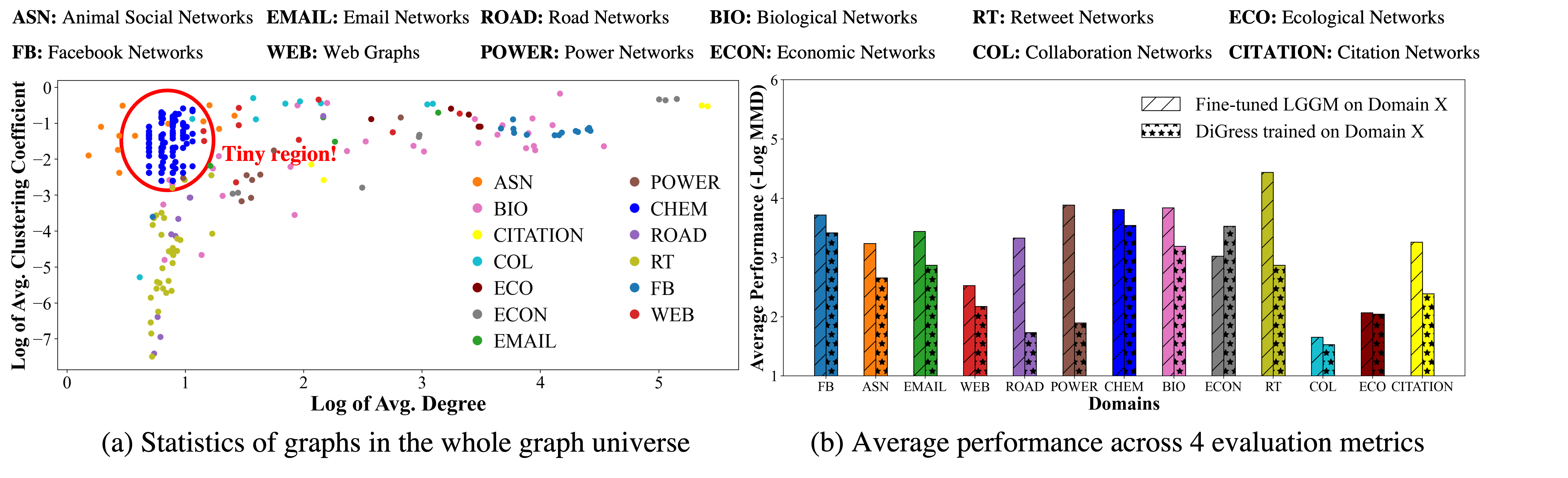}
     \vspace{-3ex}
     \caption{(a): Average degree and clustering coefficient of graphs from 13 domains. The graph universe consists of graphs from distinct domains (e.g., the tiny region of Chemical Graphs), yet there are some common transferrable patterns. (b): Our pre-trained LGGM after fine-tuning on each domain achieves better generative performance than DiGress trained on that same domain.}
    \vspace{-3ex}
     \label{fig-motivation}
\end{figure*}

Although graph generative models have been the long-standing focus of generative-based research~\cite{faez2021deep, guo2022systematic, zhu2022survey}, previous ones have been trained on graphs from only one domain each time. For example, both the representative auto-regressive-based GraphRNN~\cite{you2018graphrnn}, VAE-based GraphVAE~\cite{simonovsky2018graphvae} and diffusion-based DiGress~\cite{vignac2022DiGress} have been trained only on synthetic (Ego, Community, Grid) or chemistry graphs (Enzymes, QM9), the statistics of which only counts a tiny region of the whole graph universe and is different from graphs in other domains, as shown in Figure~\ref{fig-motivation}(a). \textcolor{DarkOrchid}{\textbf{Road Networks}} possess lower average clustering coefficients than \textcolor{RoyalBlue}{\textbf{Facebook Networks (FB)}}. This is because \textcolor{DarkOrchid}{\textbf{Road Networks}}, by design, have square intersections, whereas social relationships in \textcolor{RoyalBlue}{\textbf{Facebook Networks (FB)}} naturally form triangular connections~\cite{rossi2019complex}. Moreover, Tortoise \textcolor{Orange}{\textbf{Animal Social Networks (ASN)}\footnote{Nodes represent tortoises and an edge is set up between two tortoises if they share the same burrow~\cite{sah2019multi}.}} have a lower average degree than \textcolor{Brown}{\textbf{Power Networks}} because tortoises, as solitary creatures, would not share the same burrow~\cite{sah2016inferring}. As a result, graph generative models trained on one domain are hardly generalizable to unseen graphs, as shown by the worse zero-shot generation performance of DiGress in Table~\ref{tab-pretrain-uniform}. More critically, without training on numerous graphs covering the whole graph universe, these small models can never replicate the revolutionary success achieved by LGMs in other fields.

Recognizing the significant gap in developing LGMs for graph-structured data and their potential revolutionary impact similar to LGMs in other fields, we develop the very first \textsc{Large Graph Generative Model} (LGGM) that is pre-trained over 5000 graphs from 13 domains sourcing from the \href{https://networkrepository.com/}{Network Repository} - the interactive data repository collecting graphs from 30 domains~\cite{rossi2015network, rossi2016interactive}. After this pre-training, our LGGM learns some fundamental structural patterns that are transferrable across different domains~\cite{mao2024graph} and henceforth demonstrates significantly better zero-shot generative capability on graphs from unseen domains shown in Table~\ref{tab-pretrain-uniform}. Moreover, the pre-trained LGGM are highly adaptable for fine-tuning on a specific domain, achieving an overall performance increase of 29.59\% compared to the smaller DiGress model trained on the same domain, as depicted in Figure~\ref{fig-motivation}(b). This improvement is even more significant when only limited graphs are available shown by Figure~\ref{fig-ft-varyingnumber}, proving particular advantages in semi-supervised generative settings~\cite{dai2023advdiff, liu2024data, livernoche2023diffusion}. More importantly, our LGGMs support Text-to-Graph generation, which allows finer-level control of the generated graphs (e.g., their domains/names in Table~\ref{tab-text2graph-uniform} and clustering coefficient/average degree in Figure~\ref{fig-text2graph-property}). \textbf{Our contributions are as follows:} 

\begin{itemize}[leftmargin=*]
\item \textbf{Large Graph Generative Model:} We propose a pioneering Large Graph Generative Model (LGGM), trained on thousands of graphs arising from 13 distinct domains. To the best of our knowledge, this work is the very first one exploring the potential of LGMs on graph-structured data. We hope others expand this collection and leverage our work to develop future LGGMs that could eventually replicate the success of Stable Diffusion~\cite{rombach2022high} but in the graph modality.

\item \textbf{Superior Zero-shot and Fine-tuning Generative Capability:} Our pre-trained LGGM delivers exceptional zero-shot generative performance on unseen graphs in Table~\ref{tab-pretrain-uniform} of Section~\ref{sec-pretrain-eval} and shows great adaptability for fine-tuning in Figure~\ref{fig-ft-uniform} of Section~\ref{sec-ft-eval}. Remarkably, the fine-tuned LGGM outperforms DiGress trained from scratch on the same graphs especially under limited data scenarios, behaving as a better starting point for real-world development.

\item \textbf{Text-to-Graph Generation:} We equip the LGGM with the capability to generate graphs given user-specified text prompts, allowing finer-level control of the generated graphs in terms of their domains/names and network statistics.





\end{itemize}

%% file: relatedwork.tex
\newpage
\section{Related work}\label{sec-relate-lgm}
\vspace{-1ex}
\subsection{Large Generative Models (LGMs)}
Recent years have witnessed unprecedented success achieved by LGMs in generating creative and sensible content for a variety of downstream tasks across multiple modalities~\cite{achiam2023gpt, videoworldsimulators2024, cao2024survey, touvron2023llama, zhou2023comprehensive}. For instance, in Natural Language Processing (NLP), large language models trained on next-token prediction can effectively produce human-readable texts for completing question-answering, translation, and more tasks~\cite{qin2023chatgpt, tian2024graph,  wang2024knowledge}. Furthermore, the advancement of multi-modal generative models now supports cross-modality generation, such as converting text into images with Stable Diffusion or vice versa with GIT~\cite{rombach2022high, zhang2023llavar, wang2022git}. The key to their success lies in their ability to effectively utilize the world knowledge obtained during the pre-training stage over a large amount of well-curated data. This world knowledge has been demonstrated to be positively transferrable across numerous domains, delivering promising efficacy with few-shot task demonstrations. Compared with the recent large generative models in NLP/CV such as LLaMA3, Falcon, Stable Diffusion, LLaVA~\cite{liu2024visual,  penedo2023refinedweb, rombach2022high, touvron2023llama}, we alternatively focus on developing large generative models for graphs with the expectation to realize a similar set of advantages achieved by LGMs in other fields, including enhanced zero-shot generalibility, improved fine-tuning performance and cross-modality generation.

\vspace{-1ex}
\subsection{Graph Generative Models}\label{sec-relate-ggm}
\vspace{-1ex}
\begin{wraptable}{r}{7.5cm}
\vspace{-5ex}
\tiny
\setlength{\extrarowheight}{.1pt}
\setlength\tabcolsep{2.5pt}
\centering
\caption{Our LGGM is trained across 13 domains on thousands of graphs with the support for Text-to-Graph (T2G) generation, controlling the domain/property of the generated graphs.}
\begin{tabular}{llcccc}
\hline
\multirow{2}{*}{\textbf{Type}} & \multirow{2}{*}{\textbf{Model}} & \multirow{2}{*}{\textbf{\# Domains}} & \multirow{2}{*}{\textbf{\makecell{Multi-Domain \\Training}}} & \multirow{2}{*}{\textbf{\makecell{T2G \\Domain}}} & \multirow{2}{*}{\textbf{\makecell{T2G \\Property}}} \\
 &   &  &  &  \\
\hline
\multirow{3}{*}{\textbf{\makecell[l]{Auto- \\Regressive}}} & GraphRNN~\cite{you2018graphrnn} & 2 & \xmark & \xmark & \xmark \\
 & EdgeRNN~\cite{bacciu2020edge} & 3 & \xmark & \xmark & \xmark\\
 & MolRNN~\cite{popova2019molecularrnn} & 2 & \xmark & \xmark & \xmark\\
\hline
\multirow{4}{*}{\textbf{VAE}} & MDVAE~\cite{du2022interpretable} & 1 & \xmark & \xmark  & \xmark\\
 & PCVAE~\cite{du2021deep, guo2020property} & 3 & \xmark & \xmark & \xmark\\
 &  (DE)CO-VAE~\cite{guo2021generating} & 1 & \xmark & \xmark & \xmark\\
 &  GraphVAE~\cite{simonovsky2018graphvae} & 1 & \xmark & \xmark & \xmark\\
\hline
\multirow{2}{*}{\textbf{GAN}} & Mol-CycleGAN~\cite{maziarka2020mol} & 1 & \xmark & \xmark  & \xmark\\
 & LGGAN~\cite{fan2019conditional} & 2 & \xmark & \xmark  & \xmark\\
 \hline
\multirow{3}{*}{\textbf{Flow}} & GraphNVP~\cite{madhawa2019graphnvp} & 1 & \xmark & \xmark  & \xmark\\
 & MoFlow~\cite{zang2020moflow} & 1 & \xmark & \xmark  & \xmark\\
 & GraphDF~\cite{luo2021graphdf} & 2 & \xmark & \xmark  & \xmark\\
 \hline
\multirow{3}{*}{\textbf{Diffusion}} & GDSS~\cite{jo2022score} & 3 & \xmark & \xmark  & \xmark\\
 & DiGress~\cite{vignac2022DiGress} & 2 & \xmark & \xmark  & \xmark\\
 & GraphEBM~\cite{liu2021graphebm} & 1 & \xmark & \xmark  & \xmark\\
\hline
\multicolumn{2} {c} {\textbf{LGGM - Ours}} & \textbf{13} & \cmark & \cmark & \cmark\\
\hline
\end{tabular}
\vspace{-2ex}
\label{tab-summarize}

\end{wraptable}

Given the ubiquity of graphs in modeling relational information of real-world objects across many domains~\cite{liu2024review, rossi2015network, rossi2019complex, wang2024knowledge}, graph generative models have been developed to generate realistic graphs for advancing numerous applications~\cite{hoogeboom2022equivariant, kang2024diffattack, liu2024data, livernoche2023diffusion}, such as generating molecular graphs with high drug-likeness and designing imperceptible adversarial attacks. Graph generative models can generally be divided into two categories: statistic-based ones~\cite{goldenberg2010survey, kolaczyk2014statistical} and deep learning-based ones~\cite{guo2022systematic, zhu2022survey}. Statistic-based generative models such as Stochastic Block Models~\cite{lee2019review} and Small World Models~\cite{newman2000models} assume that the real-world graph formation adheres to specific statistical rules, and define various sampling strategies to simulate networks with prescribed properties. However, this approach oversimplifies the complex distribution of real-world graphs and struggles to generalize to those deviating from established norms. This limitation has spurred recent research into deep-learning-based generative models that automatically capture intricate statistics by learning to recover graphs~\cite{simonovsky2018graphvae, vignac2022DiGress, you2018graphrnn, zang2020moflow}. Despite their effectiveness, they all focus on a narrow range of domains and are trained solely on a single domain each time. Next, we briefly review representative deep graph generative models in Table~\ref{tab-summarize}.


GraphRNN~\cite{you2018graphrnn}, a pioneering model in autoregressive generation, employs breadth-first search to establish node ordering and sequentially generates nodes and edges. In addition, variational autoencoders~\cite{du2020interpretable, du2022interpretable, du2021deep, guo2020property, simonovsky2018graphvae} were adopted to enable flexible graph generation tailored to specific properties by regularizing the latent variables. Following that, normalizing flows were used to learn invertible mappings between molecular graphs and latent representations (e.g., GraphNVP~\cite{madhawa2019graphnvp} and MoFlow~\cite{zang2020moflow}). More recently following the success of diffusion-based models in images, graph diffusion-based models like GDSS~\cite{jo2022score} and DiGress \cite{vignac2022DiGress}  have emerged, allowing gradually generating graphs from the noise either in the continuous or the discrete space. However, these deep graph generative models have primarily focused on limited domains such as chemistry, social networks, and synthetic graphs, neglecting a vast array of unexplored graphs in other fields. More critically, these models have been historically trained on a single domain each time, mirroring earlier approaches like Unet~\cite{ronneberger2015u} and Bert/Bart~\cite{devlin2018bert, lewis2019bart} in CV and NLP, thus limiting their ability to learn fundamental knowledge transferrable across different domains. In contrast, our work focuses on learning a large generative graph model that is pre-trained on thousands of graphs from 13 domains and we demonstrate that, through extensive experiments, the proposed LGGM successfully replicates the revolutionary achievements gained by the recent LGMs in other fields~\cite{achiam2023gpt, videoworldsimulators2024, rombach2022high, touvron2023llama}.

\newpage

%% file: method.tex
\section{Large Graph Generative Models}\label{sec-lggms}
\subsection{Notation}\label{sec-notation}
\vspace{-1ex}
Let $\mathbb{G}$ be a random variable of universal graphs, governed by its underlying distribution $P(\mathbb{G})$. Given that real-world graphs originate from various domains, we introduce $\mathbb{G}^{c}$ to represent a random variable for graphs from domain $c$, with its distribution as $P(\mathbb{G}^c)$. Assuming the universal graph space encompasses $\mathcal{C}$ distinct domains, i.e., $\mathcal{G}=\cup_{c \in \mathcal{C}}\mathcal{G}^{c}$ with each set of graphs from domain $c$ as $\mathcal{G}^{c}$, then $P(\mathbb{G}^{c})/P(\mathbb{G})$ is domain-specific/agnostic distribution. To ease the introduction of training and evaluation setting in Section~\ref{sec-experiment}, we further divide each domain-specific set of graphs $\mathcal{G}^{c}$ into training, validation and testing subsets, notated as $\mathcal{G}^{c} = \mathcal{G}^{\text{Train}, c}\cup\mathcal{G}^{\text{Val}, c}\cup\mathcal{G}^{\text{Test}, c}$. We represent each graph $G=(\mathbf{X}^G, \mathbf{E}^G)$ with $\mathbf{X}^G\in\mathbb{R}^{n_G\times d_{X}}/\mathbf{E}^G\in\mathbb{R}^{n_G\times n_G\times d_{E}}$ as the one-hot encoding matrix representing node/edge categories with $n_G$ being the number of nodes in graph $G$ and $d_{X}/d_{E}$ being the number of node/edge categories, considering the edge existence as a particular edge category. In Text-to-Graph generation, each graph $G$ is paired with a textual description $S$ from the textual distribution $P(\mathbb{S})$ and their joint distribution is $P(\mathbb{G}, \mathbb{S})$.


\subsection{Large Graph Corpus}\label{sec-corpus}
\vspace{-1ex}
Training LGGM requires a substantial, well-curated collection of graphs from multiple domains. We select graphs from the Network Repository across 13 distinct yet representative domains covering a wide variety of real-world scenarios, including Facebook (FB), Animal Social (ASN), Email, Web, Road, Power, Chemical (CHEM), Biological (BIO), Economic (ECON), Retweet (RT), Collaboration (COL), Ecological (ECO), Citation, as shown in Figure~\ref{fig-framework}(a). Given that many real-world graphs (e.g., social networks and road networks) comprise thousands or even millions of nodes and edges, and that state-of-the-art diffusion models, e.g., DiGress and GDSS, are limited to handling networks with only hundreds of nodes, we further sample subgraphs for certain domains to address scalability challenges. Specifically, we generate 2/3-hop ego subgraphs centered on multiple randomly chosen nodes followed by taking their induced subgraphs. We apply this strategy iteratively across all the initially collected graphs until hitting the preset budget. Appendix~\ref{app-graph-data} presents the graph statistics.

\subsection{Pre-Training and Graph Generation of LGGM}\label{sec-pretrain}
\vspace{-1ex}
Our LGGM is designed based on discrete denoising diffusion~\cite{austin2021structured, chen2023efficient, vignac2022DiGress}, which is composed of a diffusion forward process based on a transition matrix and a reverse prediction process based on minimizing the cross-entropy loss between the ground-truth graphs and the predicted clean graphs.

During the forward process, for each graph $G$ sampled from the joint distribution $P(\mathbb{G})$, we obtain its noisy version $G^{t} = (\mathbf{X}^{t}, \mathbf{E}^{t})$ at step $t$ by sampling from the conditional categorical distribution:
\begin{equation}\label{eq-forward}
q(\mathbb{G}^{t} | \mathbb{G}^{t-1}) = (\mathbf{X}^{t-1} \mathbf{Q}^{t}_{X}, \mathbf{E}^{t-1} \mathbf{Q}^{t}_{E}) \quad \text{and} \quad q(\mathbb{G}^{t} | \mathbb{G}^{0}) = (\mathbf{X} \bar{\mathbf{Q}}^{t}_{X}, \mathbf{E} \bar{\mathbf{Q}}^{t}_{E}),
\end{equation}
where $\mathbf{Q}^{t}_X \in \mathbb{R}^{d_{X}\times d_{X}}$ and $\mathbf{Q}^{t}_E \in \mathbb{R}^{d_{E}\times d_{E}}$ are node/edge transition matrices and $\mathbb{G}^{0} = \mathbb{G}$ is the original data distribution of graphs. Depending on whether our generative downstream tasks require generalization to unseen domains or not, we can either use different transition matrices for graphs from different domains, i.e., domain-specific transition matrix $\mathbf{Q}_{X}^{t, c} = \alpha^{t}\mathbf{I} + (1 - \alpha^{t})\mathbf{1}\mathbf{m}_{X}^{c}, \mathbf{m}_{X}^c = \frac{1}{|\mathcal{G}^{\text{Train}, c}|}\sum_{G\in \mathcal{G}^{\text{Train}, c}}\mathbf{X}^G, \forall c \in C$ or unify transition matrices across different domains. For the unified transition matrices, we can trivially use the uniform transition matrix, i.e. $\mathbf{Q}_X^{t, c}=\alpha^{t}\mathbf{I} + (1 - \alpha^t)(\mathbf{1}_{d_{\mathbf{X}}}\mathbf{1}^{\top}_{d_{\mathbf{X}}})/{d_{\mathbf{X}}}$, or compute the marginal transition matrix across all graphs from all domains $\mathbf{Q}_{X}^{t} = \alpha^{t}\mathbf{I} + (1 - \alpha^{t})\mathbf{1}\mathbf{m}_{X}, \mathbf{m}_{X}=\frac{1}{|\mathcal{G}^{\text{Train}}|}\sum_{G\in\mathcal{G}^{\text{Train}}}\mathbf{X}^{G}$. And $\mathbf{Q}^{t}_{E}$ can be computed similarly. We validate the advantages of LGGMs under both of these two transition strategies in Appendix~\ref{app-expr-full}.

In the reverse process, a parametrized neural network is trained to predict the clean graph given the noisy graph sampled following Eq.~\eqref{eq-forward} by optimizing the following loss:
\begin{equation}\label{eq-loss}
    \boldsymbol{\Theta}^{\star} = \argmin_{\boldsymbol{\Theta}} \mathcal{L} = \mathbb{E}_{G \sim P(\mathbb{G})}\mathbb{E}_{t \sim \mathcal{T}}\mathbb{E}_{G^t \sim q(\mathbb{G}^{t}|\mathbb{G})}(-\log p_{\boldsymbol{\Theta}}(G|G^t)).
\end{equation}

Following~\cite{vignac2022DiGress}, we combine the learned $P_{\boldsymbol{\Theta}^{\star}}(\mathbb{G}|\mathbb{G}^t)$ and the closed-form posterior $P(\mathbb{G}^{t-1}|\mathbb{G}^{t}, \mathbb{G})$ to perform backward generation by sampling from the following distribution:
\begin{equation}\label{eq-generate}
    P(\mathbb{G}^{t-1}|\mathbb{G}^{t}) \propto \sum_{\mathbb{G}}P(\mathbb{G}^{t-1}|\mathbb{G}^{t}, \mathbb{G})P_{\boldsymbol{\Theta}^{*}}(\mathbb{G}|\mathbb{G}^{t}).
\end{equation}

\begin{figure*}[htbp!]
     \centering
     \includegraphics[width=1.0\textwidth]{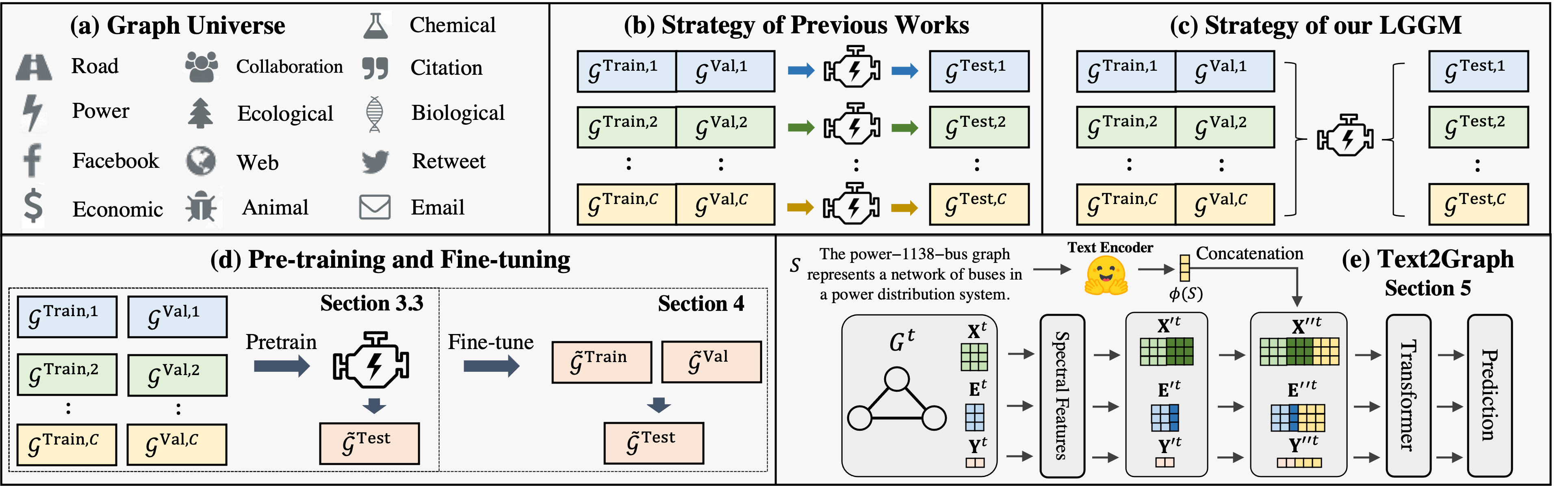}
     \caption{The overview of LGGM framework and experimental settings. (a): Graph universe including our collected 13 distinct yet representative domains. (b)-(c): Compared with all previous graph generative models that have been trained only on one domain each time, our LGGM is trained on thousands of graphs from 13 domains. (d): We pre-train/fine-tune LGGM in Section~\ref{sec-pretrain}/\ref{sec-ft}. (e): Given the text prompt $S$ and the current generated graph at $t$, we concatenate its textual embedding obtained from a pre-trained language model with the node/edge/graph embeddings after spectral feature extraction and forward them through the Graph Transformer to predict the clean graph.}
     \label{fig-framework}
\end{figure*}

\section{Fine-tuning LGGM}\label{sec-ft}

In many real-world applications, the graphs of interest $\widetilde{\mathcal{G}}$ may highly likely come from completely unseen domains, i.e., $\widetilde{\mathcal{G}} \cap \mathcal{G} = \emptyset$, and their corresponding distribution may also be significantly different from the pre-trained one, i.e., $P(\widetilde{\mathbb{G}})\ne P(\mathbb{G})$ as shown by comparing CHEM and FB Networks in Figure~\ref{fig-motivation}(a). In this case, we further fine-tune our pre-trained LGGM based on the observed graphs $\widetilde{\mathcal{G}}$ from the unseen domains:
\begin{equation}\label{eq-ft-loss}
    \boldsymbol{\Theta}^{\star\star} = \argmin_{\boldsymbol{\Theta}} \mathcal{L} = \mathbb{E}_{\widetilde{G} \sim P(\widetilde{\mathbb{G}})}\mathbb{E}_{t \sim \mathcal{T}}\mathbb{E}_{\widetilde{G}^{t} \sim q(\widetilde{\mathbb{G}}^{t}|\widetilde{\mathbb{G}})}(-\log p_{\boldsymbol{\Theta}}(\widetilde{G}|\widetilde{G}^t)),
\end{equation}
where $\boldsymbol{\Theta}^{\star\star}$ is initialized as $\boldsymbol{\Theta}^{\star}$ from the pretaining phase in Eq~\eqref{eq-loss}. After fine-tuning, our LGGM can effectively adapt to unseen distributions by using both the prior knowledge from the pre-training stage and the specific knowledge of new graphs from the unseen domains, as verified in Figure~\ref{fig-ft-uniform}.


Despite the superior capabilities of LGGM in generating graphs after both pre-training and fine-tuning processes, they essentially mimic the random sampling from the learned distribution $P(\mathbb{G})$ that is prescribed by the training data without any fine-level customization. To control the characteristics of the generated graphs, we further propose the very first Text-to-Graph LGGM to generate graphs based on textual description. In this way, users could specify their desired properties of the graphs through natural language description, thereby guiding the graph generation in a more tailored manner.

\section{Text-to-Graph LGGM}\label{sec-Text2Graph}
Given the textual description $S$ about the network to be generated, our goal here is to learn $P(\mathbb{G}^{t - 1}|\mathbb{G}^t, \mathbb{S})$, which is further decomposed as:
\begin{equation}\label{eq-generate-text}
    P(\mathbb{G}^{t-1}|\mathbb{G}^{t}, \mathbb{S}) \propto \sum_{\mathbb{G}}P(\mathbb{G}^{t-1}|\mathbb{G}^{t}, \mathbb{G}, \mathbb{S})P(\mathbb{G}|\mathbb{G}^{t}, \mathbb{S}).
\end{equation}
Theorem~\ref{thm-text2graph} proves that if the transition matrices $\mathbf{Q}_X^t, \mathbf{Q}_E^t$ in Eq.~\eqref{eq-forward} are independent of the textual description $S$, the first term $P(\mathbb{G}^{t - 1}|\mathbb{G}^t, \mathbb{G}, \mathbb{S})$ can then be simplified as $P(\mathbb{G}^{t - 1}|\mathbb{G}^t, \mathbb{G})$ with the analytical form computation~\cite{vignac2022DiGress}. For the second term, we approximate it by a neural network, i.e., $P(\mathbb{G}|\mathbb{G}^{t}, \mathbb{S}) = P_{\boldsymbol{\Theta}^{\blacktriangle}}(\mathbb{G}|\mathbb{G}^{t}, \mathbb{S})$ with $\boldsymbol{\Theta}^{\blacktriangle}$ being optimized by:
\begin{equation}\label{eq-text2graph-loss}
    \boldsymbol{\Theta}^{\blacktriangle} = \argmin_{\boldsymbol{\Theta}} \mathcal{L} = \mathbb{E}_{(G, S) \sim P(\mathbb{G}, \mathbb{S})}\mathbb{E}_{t \sim \mathcal{T}}\mathbb{E}_{G^{t} \sim q(\mathbb{G}^{t}|\mathbb{G})}(-\log p_{\boldsymbol{\Theta}}(G|G^t, \phi(S))),
\end{equation}
where $\phi$ is a pre-trained textual encoder. Figure~\ref{fig-framework}(e) shows the architecture of LGGM-Text2Graph, which firstly integrates the textual embedding $\phi(S)$ into the node/edge/graph-level latent embeddings after spectral feature extraction of the current generated graph and further predicts the clean graph. Theorem~\ref{thm-elbo} proves that modeling $P(\mathbb{G}^{t - 1}|\mathbb{G}^t, \mathbb{S})$ with $P_{\boldsymbol{\Theta}^{\blacktriangle}}(\mathbb{G}^{t - 1}|\mathbb{G}^t, \mathbb{S})$ leads to higher evidence lower bound of the likelihood $\log P(\mathbb{G}^0, \mathbb{S})$.

\newpage

Training $p_{\boldsymbol{\Theta}}(G|G^t, \phi(S))$ in Eq~\eqref{eq-text2graph-loss} requires the joint distribution between graphs and their corresponding textual descriptions, i.e., $P(\mathbb{G}, \mathbb{S})$. Given users' specific interests in the graphs to generate, we explore two main categories of textual prompts to guide graph generation: domain/name (e.g., Power Network, power-1138-bus) and structural characteristics (e.g., average degree, clustering coefficient). For example, zoologists interested in the dynamics of tortoise interactions might seek to generate Animal Social Networks~\cite{sosa2021animal}, and social scientists studying social anomalies might prioritize generating social interactions with dense and unexpected connections~\cite{ma2021comprehensive}. Since this work is a pioneering effort in Text-to-Graph generation and no prior collection of user prompts for this purpose exists, following previous works, e.g., LLaVA~\cite{liu2024visual, zhang2023enhanced}, we ask GPT3.5/4 to emulate the human drafting of prompts to obtain pairs of (user prompt, graph). For preparing the graphs with user prompts about their domains/names, we obtain the domain/name information of each graph directly from the Network Repository~\cite{rossi2015network} and prompt GPT3.5 to generate the human-readable description paired with the corresponding graph. See more details in Appendix~\ref{app-graph-name-data}. For preparing the graphs with user prompts about their average clustering coefficient/degree, instead of using graphs from Network Repository that only count partially of the entire graph universe (i.e., no existing graphs there cover the area with high average degree and low average clustering coefficient in Figure~\ref{fig-motivation}(a)), we use the Watts–Strogatz small-world graph model~\cite{watts1998collective} to synthesize graphs covering the full spectrum of the graph universe. After that, we calculate the average degree and clustering coefficient for each graph and prompt GPT4 to generate textual descriptions about these networks using their statistics. See more details in Appendix~\ref{app-graph-property-data}. We also employ t-SNE visualization~\cite{van2008visualizing} to analyze the generated textual descriptions, as shown in Figure~\ref{fig-tsne-prompt}. This visualization indicates that texts describing graphs from various domains or with distinct statistics tend to form separate clusters, a necessary condition for the successful control of the generated graphs.

\vspace{-1ex}
\section{Experiments}\label{sec-experiment}
\vspace{-1ex}
\subsection{Experimental Setup}\label{sec-expr-setup}
In this section, we conduct four experiments over the graphs collected from 13 domains to demonstrate the effectiveness of LGGMs in four different aspects, the details of which are summarized as follows:

\begin{itemize}[leftmargin = *]
    \item \textbf{Pre-training Evaluation in Table~\ref{tab-pretrain-uniform} in Section~\ref{sec-pretrain-eval}:} To demonstrate the superior zero-shot performance of LGGM in generating unseen graphs compared to conventional graph generative models, we adopt the out-of-distribution evaluation where we iteratively treat each domain X as the unseen one and train the LGGM using training graphs from all other domains, and evaluate its performance on the testing graphs from the unseen domain X. The variant of LGGM in this experiment is called LGGM-X where X represents the unseen domain.

    \item \textbf{Fine-tuning Evaluation in Figure~\ref{fig-ft-uniform} in Section~\ref{sec-ft-eval}:} To demonstrate the high adaptability for fine-tuning LGGM, we further fine-tune the above pre-trained LGGM. Specifically, we take LGGM-X pre-trained on graphs from all other domains but domain X, and then fine-tune it on the training graphs from domain X. After that we evaluate it on the testing graphs from domain X. The variant of LGGM in this experiment is called Fine-tuned LGGM on X. 

    \item \textbf{Text-to-Graph Generation in Table~\ref{tab-text2graph-uniform} and Figure~\ref{fig-text2graph-property} in Section~\ref{sec-Text-to-Graph-eval}:} To control the graph generation, we consider two types of user prompt information: the domain/name and the graph properties, i.e., we train LGGM on training graphs from all domains with user prompts either describing the graph domains/names or graph statistics. We call these two variants of LGGM as LGGM-T2G$^{\text{D}}$ and LGGM-T2G$^{\text{UP}}$, respectively.

    \item \textbf{Fine-tuned LGGM compared with DiGress trained directly on X in Figure~\ref{fig-motivation}(b)/\ref{fig-ft-varyingnumber} in Section~\ref{sec-train-from-scratch}:} When having access to graphs of domain X, users could directly train existing graph generative models and generate graphs for the domain X. To demonstrate the practical usage of LGGMs, we further compare the fine-tuned LGGM on X with DiGress directly trained on X. In addition, we also compare their performance under limited data scenarios~\cite{gavrilev2023anomaly, liu2024data}.    
\end{itemize}

Figure~\ref{fig-experiment-setup} in Appendix~\ref{app-expr-setup} comprehensively illustrates each of the above training paradigms. Due to the page limitation, we present the evaluation metrics and model hyperparameters in Appendix~\ref{app-expr}. Moreover, we only present results under the uniform transition strategy in the main paper while leaving the one under domain-specific transition strategy in Appendix~\ref{app-expr-full}. It is important to note that the benefits of LGGMs are consistent across both of these two transition strategies.

%% file: experiment.tex
\begin{table*}[t!]
\scriptsize
\setlength\tabcolsep{5.5pt}
\caption{Comparing Zero-shot Generative Performance on unseen Graphs in held-out domain X between DiGress trained on QM9 and LGGM-X trained on all except the held-out domain X. Result "ALL" is computed by averaging across 12 domains and the best result for each domain is in \textbf{bold}.}
\label{tab-pretrain-uniform}
\centering
\begin{tabular}{llcccc|llcccc}
\toprule
\textbf{Domain} & \textbf{Method} & \textbf{DEG} & \textbf{CC} & \textbf{Spec} & \textbf{Orb} & \textbf{Domain} & \textbf{Method} & \textbf{DEG} & \textbf{CC} & \textbf{Spec} & \textbf{Orb} \\
\toprule
\multirow{2}{*}{\textsc{FB}} & DiGress & \textbf{0.3376} & \textbf{0.6298} & \textbf{0.0797} & \textbf{0.3593} & \multirow{2}{*}{BIO} & DiGress & 0.2712 & 0.5202 & 0.1127 & 0.3188 \\
 & LGGM-X & 0.4723 & 0.6843 & 0.2924 & 0.7555 & & LGGM-X & \textbf{0.1081} & \textbf{0.2696} & \textbf{0.0900} & \textbf{0.2053}\\
 \midrule
\multirow{2}{*}{\textsc{ASN}} & DiGress & 0.1496 & 0.3258 & 0.1506 & 0.4420 & \multirow{2}{*}{\textsc{ECON}} & DiGress & 0.2987 & 0.4841 & 0.2162 & 0.3834 \\
 & LGGM-X & \textbf{0.0281} & \textbf{0.2440} & \textbf{0.0830} & \textbf{0.0618} & & LGGM-X & \textbf{0.1213} & \textbf{0.0920} & \textbf{0.1120} & \textbf{0.1086} \\
 \midrule
\multirow{2}{*}{\textsc{Email}} & DiGress & 0.2192 & 0.6012 & \textbf{0.0702} & 0.3416 & \multirow{2}{*}{\textsc{RT}} & DiGress & 0.4164 & \textbf{0.1327} & 0.4147 & 0.5957 \\
 & LGGM-X & \textbf{0.0751} & \textbf{0.2364} & 0.0768 & \textbf{0.3089} & & LGGM-X & \textbf{0.0525} & 0.1429 & \textbf{0.1330} & \textbf{0.2219} \\
 \midrule
\multirow{2}{*}{\textsc{Web}} & DiGress & 0.2556 & 0.6186 & 0.1877 & 0.6045 & \multirow{2}{*}{\textsc{Col}} & DiGress & 0.2473 & 0.5826 & 0.2314 & 0.7679 \\
 & LGGM-X &  \textbf{0.0648} & \textbf{0.3961} & \textbf{0.0549} & \textbf{0.1127} & & LGGM-X & \textbf{0.0736} & \textbf{0.5769} & \textbf{0.0895} & \textbf{0.0988}\\
 \midrule
\multirow{2}{*}{\textsc{ROAD}} & DiGress & 0.3705 & 0.8226 & 0.2801 & 0.7198 & \multirow{2}{*}{\textsc{Eco}} & DiGress & 0.5431 & 0.7915 & \textbf{0.2338} & 0.6045 \\
 & LGGM-X & \textbf{0.0713} & \textbf{0.2193} & \textbf{0.0987} & \textbf{0.2986} & & LGGM-X & \textbf{0.4753} & \textbf{0.3904} & 0.3194 & \textbf{0.3934}  \\
 \midrule
\multirow{2}{*}{\textsc{Power}} & DiGress & 0.3726 & 0.4582 & 0.3270 & 1.4732 & \multirow{2}{*}{\textsc{Citation}} & DiGress & 0.2527 & 0.7790 & 0.1315 & \textbf{0.4966}\\
 & LGGM-X & \textbf{0.0119} & \textbf{0.1293} & \textbf{0.0373} & \textbf{0.0754} & & LGGM-X & \textbf{0.1348} & \textbf{0.7257} & \textbf{0.1160} & 0.4981 \\
 \midrule
 \multirow{2}{*}{\textsc{All}} & DiGress & 0.3112 & 0.5622 & 0.2030 & 0.5923\\
 & LGGM-X & \textbf{0.1408} & \textbf{0.3422} & \textbf{0.1253} & \textbf{0.2616} \\
 \bottomrule
\end{tabular}
\begin{tablenotes}
      \scriptsize
      \centering
      \item \textbf{DEG}, \textbf{CC}, \textbf{Spec}, \textbf{Orb}: MMD of Degree, Clustering Coefficient, Eigenvalues, and Orbits, more details are in Appendix~\ref{app-metrics}.
\end{tablenotes}
\vspace{-2ex}
\end{table*}

\begin{figure*}[t!]
        \centering
        \begin{subfigure}[b]{0.475\textwidth}  
            \centering 
            \includegraphics[width=\textwidth]{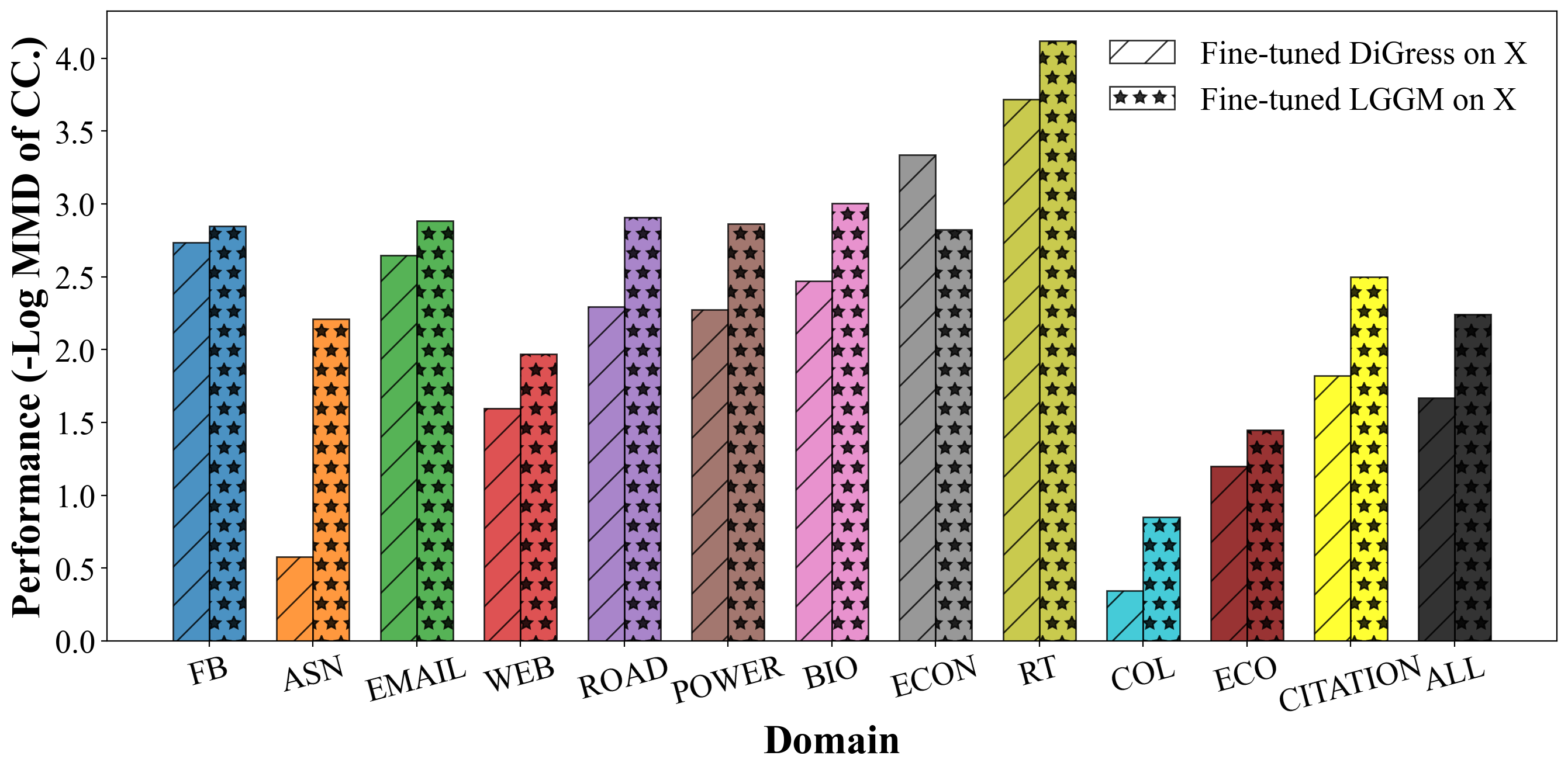}
            \caption{Performance of MMD of CC.}
            \label{fig-finetune-cc}
        \end{subfigure}
        \begin{subfigure}[b]{0.475\textwidth}   
            \centering 
            \includegraphics[width=\textwidth]{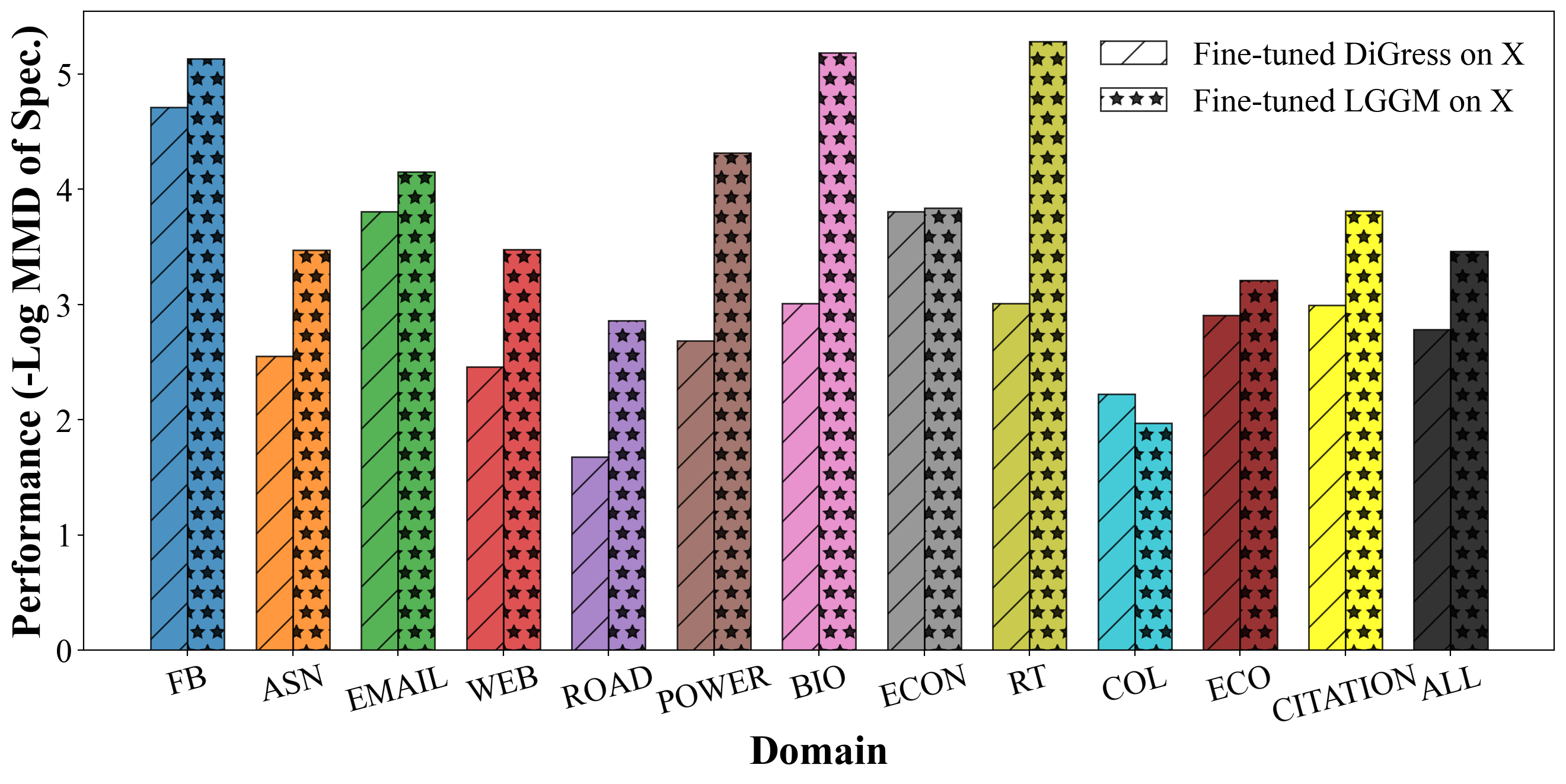}
            \caption{Performance of MMD of Spec.}
            \label{fig-finetune-spec}
        \end{subfigure}
        \caption{Performance comparison between Fine-tuned LGGM and Fine-tuned DiGress.} 
        \vspace{-1ex}
        \label{fig-ft-uniform}
    \end{figure*}

\vspace{-2ex}
\subsection{Pre-training Evaluation}\label{sec-pretrain-eval}
Table~\ref{tab-pretrain-uniform} compares the performance of our model, LGGM-X, pre-trained on all graph domains except the held-out domain X, with DiGress trained on the QM9 dataset. Both of them are evaluated over graphs from the unseen domain X. Overall, LGGM-X outperforms DiGress across all evaluation metrics shown by the "ALL" result. This superiority suggests that training on graphs from diverse domains captures transferable structural patterns and enhances the generalization of the model to unseen domains. The only exception from this trend occurs with Facebook Networks (FB) where our LGGM-X performs uniformly worse than DiGress across all evaluation metrics. This is because Facebook Networks (FB) only count a tiny region among the whole graph universe. As illustrated in Figure~\ref{fig-motivation}(a), the average clustering coefficient of FB graphs ranges from 0.301 to 0.407, a narrow segment within the broader global graph spectrum spanning from 0 to 1. This narrow range poses a challenge for the generalized LGGM-X to specialize in learning the graph data distribution specific to the FB domain. Furthermore, we conduct the same experiment but under the domain-specific transition strategy in Table~\ref{app-tab-pretrain-marginal} in Appendix~\ref{app-expr-full}, and similarly, LGGM-X generally outperforms DiGress. 

\vspace{-1ex}
\subsection{Fine-tuning Evaluation}\label{sec-ft-eval}
\vspace{-1ex}
In addition to the superior zero-shot generative performance of pre-trained LGGM-X, many real-world applications already possess exemplary graphs that can be leveraged, e.g., different types of anomaly behaviors in social networks/e-commerce platforms, and molecules with predefined chemical structures in drug discovery. In these scenarios, users can fine-tune LGGM-X with these domain-specific graphs, adapting the broadly trained model to specialize in generating graphs tailored to target domains. Figure~\ref{fig-ft-uniform} compares the generative performance of fine-tuned DiGress on X that is originally pre-trained on QM9 and fine-tuned LGGM-X on X that is originally pre-trained on all but domain X. We can see that LGGM-X consistently outperforms DiGress for graphs from most of the domains, which further validates the adaptability of LGGM after fine-tuning on a specific domain.

\begin{table*}[t!]
\scriptsize
\setlength\tabcolsep{4.5pt}
\caption{Comparing the Graph Generative Performance of LGGM with/without Text Conditions. Best and runner-up results are \textbf{bolded} and \underline{underlined}.}
\vspace{-1.5ex}
\label{tab-text2graph-uniform}
\centering
\begin{tabular}{llcccc|llcccc}
\toprule
\textbf{Domain} & \textbf{Method} & \textbf{DEG} & \textbf{CC} & \textbf{Spec} & \textbf{Orb} & \textbf{Domain} & \textbf{Method} & \textbf{DEG} & \textbf{CC} & \textbf{Spec} & \textbf{Orb} \\
\toprule
\multirow{3}{*}{\textsc{FB}} & LGGM & \underline{0.0321} & 0.4994 & 0.0763 & 0.3117 & \multirow{3}{*}{\textsc{BIO}} & LGGM & 0.2661 & 0.3120 & 0.1135 & 0.3835\\
& LGGM-T2G$^{\text{D}}$ & 0.1561 & \underline{0.1639} & \underline{0.0924} & \underline{0.0417} & & LGGM-T2G$^{\text{D}}$ & \underline{0.0099} & \underline{0.1286} & \underline{0.0303} & \underline{0.1366}\\
 & LGGM-T2G$^{\text{UP}}$ & \textbf{0.0050} & \textbf{0.0545} & \textbf{0.0070} & \textbf{0.0251} & & LGGM-T2G$^{\text{UP}}$ & \textbf{0.0028} & \textbf{0.0287} & \textbf{0.0236} & \textbf{0.0174}\\
 \midrule
\multirow{3}{*}{\textsc{ASN}} & LGGM & 0.1511 & 0.4325 & 0.1875 & 0.3896 & \multirow{3}{*}{\textsc{ECON}} & LGGM & 0.3828 & 0.1533 & 0.2039 & 0.2583\\
& LGGM-T2G$^{\text{D}}$ & \underline{0.0318} & \underline{0.2821} & \underline{0.0606} & \underline{0.0631} & & LGGM-T2G$^{\text{D}}$ & \underline{0.0666} & \underline{0.0594} & \underline{0.0650} & \underline{0.0586}\\
 & LGGM-T2G$^{\text{UP}}$ & \textbf{0.0211} & \textbf{0.1191} & \textbf{0.0462} & \textbf{0.0195} & & LGGM-T2G$^{\text{UP}}$ & \textbf{0.0132} & \textbf{0.0257} & \textbf{0.0053} & \textbf{0.0191} \\
 \midrule
\multirow{3}{*}{\textsc{Email}} & LGGM & 0.2156 & 0.2450 & 0.0666 & 0.2757 & \multirow{3}{*}{\textsc{RT}} & LGGM & 0.4395 & 0.2225 & 0.4337 & 0.6641 \\
& LGGM-T2G$^{\text{D}}$ & \underline{0.0469} & \underline{0.0982} & \underline{0.0484} & \underline{0.0505} & & LGGM-T2G$^{\text{D}}$ & \underline{0.0468} & \underline{0.0955} & \underline{0.0729} & \underline{0.0393}\\
 & LGGM-T2G$^{\text{UP}}$ & \textbf{0.0073} & \textbf{0.0379} & \textbf{0.0127} & \textbf{0.0437} & & LGGM-T2G$^{\text{UP}}$ & \textbf{0.0286} & \textbf{0.0933} & \textbf{0.0400} & \textbf{0.0312}\\
 \midrule
\multirow{3}{*}{\textsc{Web}} & LGGM & 0.2725 & 0.2672 & 0.1900 & 0.4368 & \multirow{3}{*}{\textsc{Col}} & LGGM  & 0.3565  & 0.3554 & 0.2451 & 0.7874\\
& LGGM-T2G$^{\text{D}}$ & \underline{0.0255} & \textbf{0.0737} & \underline{0.0354} & \underline{0.1856} & & LGGM-T2G$^{\text{D}}$ & \underline{0.0395} & \underline{0.3110} & \underline{0.1146} & \underline{0.1823}\\
& LGGM-T2G$^{\text{UP}}$ & \textbf{0.0105} & \underline{0.0941} & \textbf{0.0206} & \textbf{0.0451} & & LGGM-T2G$^{\text{UP}}$ & \textbf{0.0265} & \textbf{0.2813} & \textbf{0.0895} & \textbf{0.0899}\\
 \midrule
\multirow{3}{*}{\textsc{ROAD}} & LGGM & 0.4825 & 0.5373 & 0.3398 & 0.7542 & \multirow{3}{*}{\textsc{Eco}} & LGGM & 0.5466 & 0.6003 & 0.2257 & 0.7089 \\
& LGGM-T2G$^{\text{D}}$ & \textbf{0.0088} & \underline{0.1225} & \underline{0.0399} & \underline{0.0155} & & LGGM-T2G$^{\text{D}}$ & \underline{0.2160} & \underline{0.2917} & \underline{0.1203} & \underline{0.2569}\\
 & LGGM-T2G$^{\text{UP}}$ & \underline{0.0177} & \textbf{0.0437} & \textbf{0.0336} & \textbf{0.0086} & & LGGM-T2G$^{\text{UP}}$ &  \textbf{0.0293} & \textbf{0.2885} & \textbf{0.0416} & \textbf{0.2556}  \\
 \midrule
\multirow{3}{*}{\textsc{Power}} & LGGM & 0.4394 & 0.4646 & 0.3473 & 1.3186 & \multirow{3}{*}{\textsc{Citation}} & LGGM & 0.2624 & 0.5374 & 0.1295 & 0.3419\\
& LGGM-T2G$^{\text{D}}$ & \underline{0.0162} & \underline{0.1131} & \underline{0.0479} & \underline{0.1786} & & LGGM-T2G$^{\text{D}}$ & \underline{0.0101} & \underline{0.1025} & \underline{0.0315} & \underline{0.0651}\\
& LGGM-T2G$^{\text{UP}}$ & \textbf{0.0062} & \textbf{0.0570} & \textbf{0.0111} & \textbf{0.0084} & & LGGM-T2G$^{\text{UP}}$ & \textbf{0.0072} & \textbf{0.0849} & \textbf{0.0115} & \textbf{0.0287}\\
 \midrule
\multirow{3}{*}{\textsc{All}} & LGGM & 0.3206 & 0.3856 & 0.2132 & 0.5526  \\
 & LGGM-T2G$^{\text{D}}$ & \underline{0.0562} & \underline{0.1535} & \underline{0.0633} & \underline{0.1061}  \\
 & LGGM-T2G$^{\text{UP}}$ & \textbf{0.0146} & \textbf{0.1007} & \textbf{0.0286} & \textbf{0.0494}  \\
 \bottomrule
\end{tabular}
\vspace{-2ex}
\end{table*}

\begin{figure}[t!]
     \centering
     \vspace{-1ex}
     \includegraphics[width=1\textwidth]{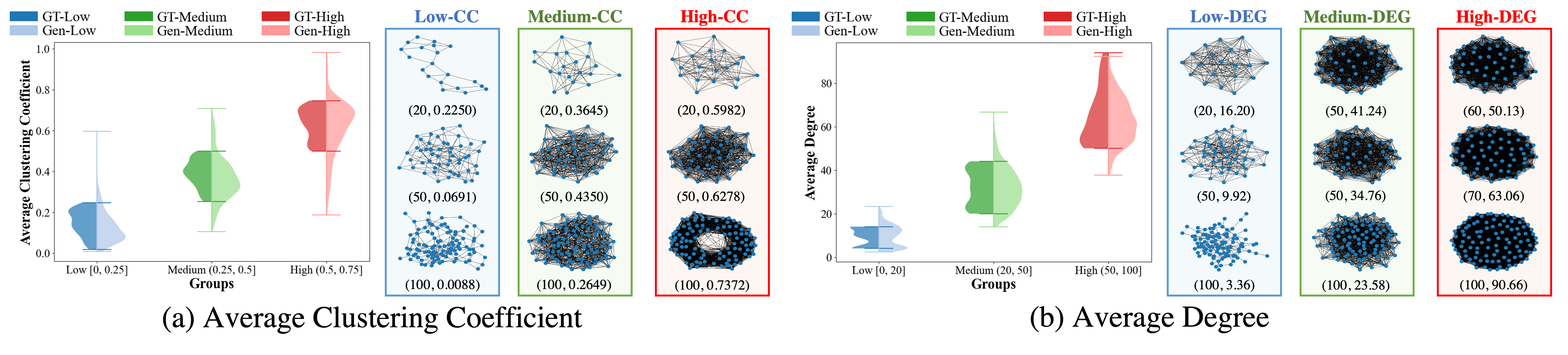}
     \vspace{-3ex}
     \caption{Text-to-Graph Generation with Prescribed Graph Properties. (a) Controlling Average Clustering Coefficient; (b) Controlling Average Degree. GT-Ground Truth Graphs and Gen-Generated Graphs. Below each graph, the number of nodes and key statistical measures are displayed.}
     \label{fig-text2graph-property}
     \vspace{-2ex}
\end{figure}

\vspace{-4ex}
\subsection{Text-to-Graph Generation}\label{sec-Text-to-Graph-eval}
\vspace{-1ex}
Here we integrate Text-to-Graph (T2G) generation into LGGMs. We introduce two variants: LGGM-T2G$^{\text{D}}$, which utilizes domain labels such as "Power Networks" as textual descriptions, and LGGM-T2G$^{\text{UP}}$, which utilizes user prompts from GPT3.5, like "The power-1138-bus graph represents a network of buses in a power distribution system". Table~\ref{tab-text2graph-uniform} compares the basic LGGM trained without text conditions, against LGGM-T2G$^{\text{D}}$ and LGGM-T2G$^{\text{UP}}$. Firstly, we observe a significant performance improvement from LGGM to LGGM-T2G$^{\text{D}}$/LGGM-T2G$^{\text{UP}}$. The inclusion of text descriptions acts as a unique identifier that enables LGGM-T2G to specialize in generating graphs aligning with corresponding domains. Moreover, the network-level user prompts in LGGM-T2G$^{\text{UP}}$ provide a finer-level control compared to the domain-level descriptions in LGGM-T2G$^{\text{D}}$, further boosting the performance. Furthermore, we shuffle the domain names paired with each graph for LGGM-T2G$^{\text{D}}$ in the testing phase and observe the performance decrease in Table~\ref{app-table-text2graph-shuffle-uniform} as expected.

LGGM-T2G can also control the properties of the generated graphs. Here we first synthesize ground-truth graphs with clustering coefficients between [0, 0.75] and average degrees between [0, 100]. We divide these ground-truth graphs into three groups, low/medium/high, and prompt GPT4 to generate user instructions describing these two graph properties (Appendix~\ref{app-graph-property-data}). Then we combine these three groups of graphs with their instructions to train LGGM-T2G and evaluate whether the properties of the generated graphs align with the instructions. In Figure~\ref{fig-text2graph-property}(a)/(b), we can see a clear alignment between the statistical properties of the ground-truth graphs and those of the generated graphs, both in terms of the average CC and DEG. Furthermore, we visualize the generated graphs for these three groups in Figure~\ref{fig-text2graph-property}(a)/(b). We can see graphs in low-CC groups possess many squares while the ones in high-CC groups contain many triangles, aligning with the intuition of CC.

\begin{figure*}[t!]
     \centering
     \begin{subfigure}[b]{0.245\textwidth}
         \centering
         \includegraphics[width=\textwidth]{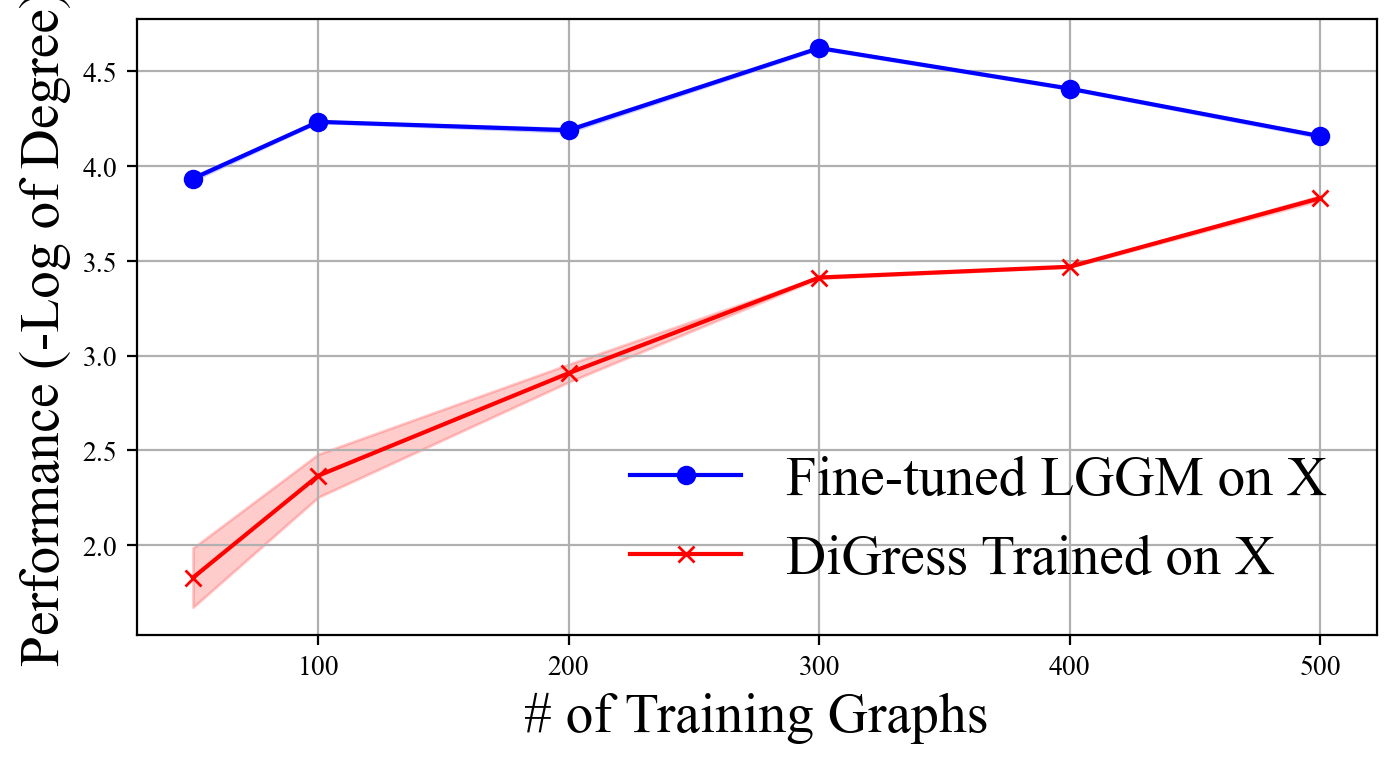}
         \caption{Road-DEG.}
         \label{fig-sens-road-deg-marginal}
     \end{subfigure}
     \hfill
     \begin{subfigure}[b]{0.245\textwidth}
         \centering
         \includegraphics[width=\textwidth]{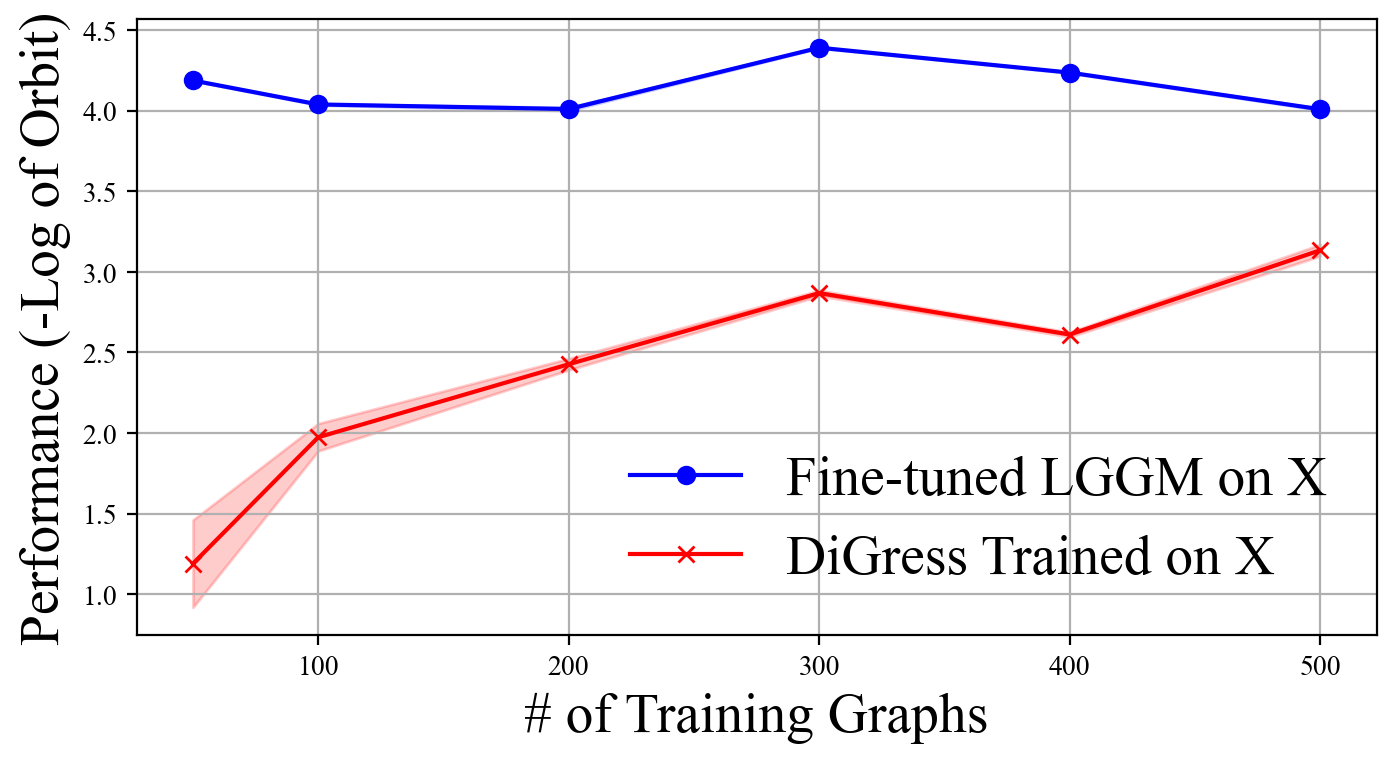}
         \caption{Road-Orbit.}
         \label{fig-sens-road-orbit-marginal}
     \end{subfigure}
     \hfill
     \begin{subfigure}[b]{0.245\textwidth}
         \centering
         \includegraphics[width=\textwidth]{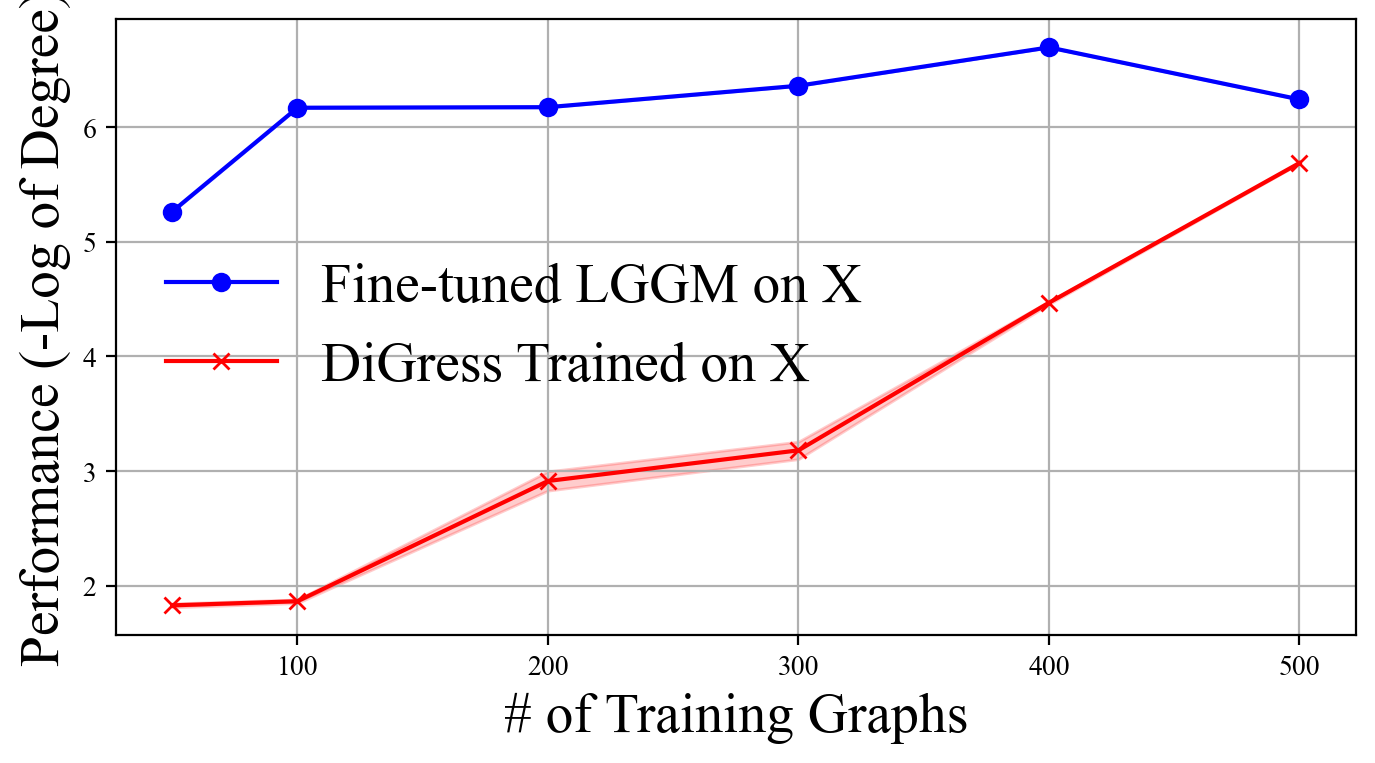}
         \caption{Retweet-DEG.}
         \label{fig-sens-rt-deg-marginal}
     \end{subfigure}
     \hfill
     \begin{subfigure}[b]{0.245\textwidth}
         \centering
         \includegraphics[width=\textwidth]{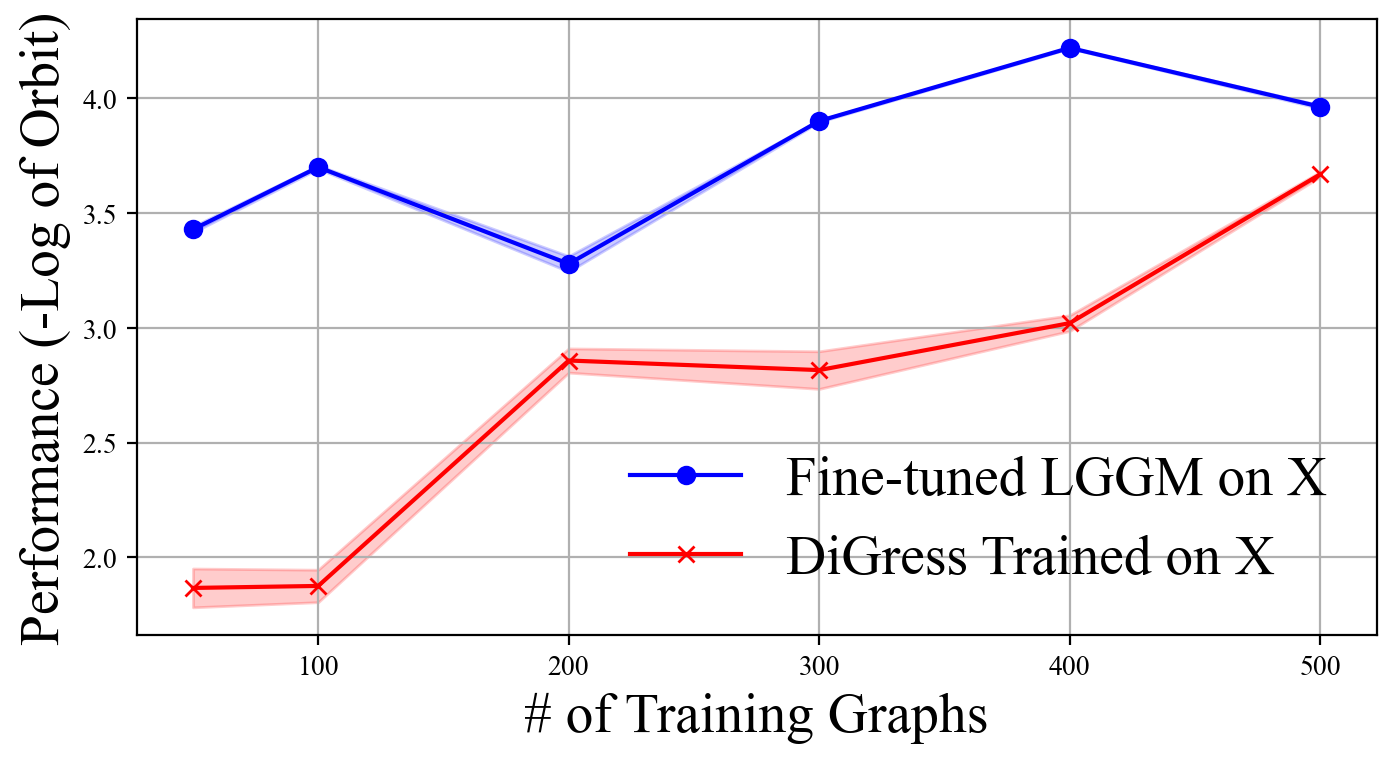}
         \caption{Retweet-Orbit.}
         \label{fig-sens-rt-orbit-marginal}
     \end{subfigure}
    \caption{With fewer training graphs, Fine-tuned LGGM becomes more advantageous than DiGress.}
    \label{fig-ft-varyingnumber}
    \vspace{-2ex}
\end{figure*}

\vspace{-3ex}
\subsection{Practical Usage of Fine-tuned LGGM}\label{sec-train-from-scratch}
To demonstrate the practical usage of LGGM in generating graphs for real-world deployment, we further compare the fine-tuned LGGM with DiGress trained directly on each domain in Figure~\ref{fig-motivation}(b). We can see that even using the same graphs for training, due to the additional knowledge incorporated during the pre-training phase of LGGM, it exhibits significantly better generative performance for most domains. Moreover, this advantage becomes even more pronounced when fewer graphs are available. Figure~\ref{fig-ft-varyingnumber} illustrates the enhanced performance of fine-tuned LGGM versus DiGress trained on X, with a widening margin as the number of training graphs in X decreases. This is particularly useful since many graph generative applications involve semi-supervised settings, e.g., generating anomaly software and design of drugs, the amount of which only count 0.05\%-0.5\%~\cite{bajorath2002integration} and 0.01\%~\cite{oak2019malware} among the whole potential candidates, respectively.

%% file: conclusion.tex
\vspace{-1ex}
\section{Research Problems Enabled by LGGMs}\label{sec-future}
\vspace{-1ex}
As the proposed LGGMs are the first to explore the potential of large generative models in graphs, it will spark numerous transformative research opportunities~\cite{leskovec2010kronecker}, which are summarized below:
\begin{itemize}[leftmargin=*]
    \item \textbf{Simulations, Extrapolation, Anonymization}: Since our LGGM-T2G can generate graphs with pre-defined properties, we can simulate graphs with various properties and extrapolate new insights from these simulated graphs, e.g., evaluate conventional and newly designed graph algorithms/ models~\cite{palowitch2022graphworld}. Moreover, for sensitive real-world graphs, such as those involving proprietary drug designs and corporate networks, we can maintain confidentiality by sharing only the model, which can then simulate similar graphs without disclosing private information.

    \item \textbf{Data Augmentation}: LGGMs can be used for data augmentation when only limited graphs are available for applications like graph anomaly detection and molecular tasks~\cite{liu2024data, qin2023class}. 

    \item \textbf{Graph Compression}: LGGM allows for the compression of graphs across multiple domains by merely storing model parameters instead of the original graphs. 
\end{itemize}

\vspace{-2ex}
\section{Limitations, Future Directions, and Conclusion}\label{sec-conclusion}
\textbf{Limitations and Future Directions}: Like LGMs in other fields~\cite{wang2023survey, zhao2023domain}, our LGGMs are not specialized in generating graphs for specific domains. One future direction could be exploring strategies such as Retrieval-Augmented Generation to enhance domain proficiency~\cite{wang2024knowledge}. Additionally, our evaluation of LGGMs has focused solely on their generative capabilities, without examining their potential usage in downstream tasks. A promising future direction is to assess their practical utility in application-oriented manners, e.g., higher quality of generated graphs for better data augmentation.

\textbf{Conclusion}: Motivated by the recent successes of Large Generative Models (LGMs) across fields of Vision, Language, Video, and Audio, and recognizing the promising practical usage of graph generative models, we introduce, for the very first time, Large Graph Generative Models (LGGMs). These models are trained on over 5,000 graphs sourced from 13 distinct domains from the well-known Network Repository. We empirically verify the superiority of our LGGMs in three aspects. Firstly, our pre-trained LGGM-X models demonstrate exceptional zero-shot generative capabilities. Secondly, LGGMs show remarkable adaptability for fine-tuning, and the fine-tuned LGGM is even more powerful than previous graph generative models trained from scratch. Lastly, our models facilitate Text-to-Graph generation, enabling users to specify domain/network names/statistics through prompts to control the generated graphs. Looking ahead, we identify several potential transformative research problems in Section~\ref{sec-future}. To foster further innovation and community collaboration, we release the complete resources of LGGMs, including code, data, and model checkpoints. We invite the community to use these tools to explore new possibilities in graph generation and beyond.

%% file: appendix.tex
\renewcommand \thepart{} 
\renewcommand \partname{}

\newpage
\part{Appendix} 
\parttoc

\clearpage

\newpage
\section{Notations}\label{app-notations}

This section summarizes the notations used throughout this paper.

\begin{table}[htbp!]
\caption{Notations used throughout this paper.} 
\setlength{\extrarowheight}{.095pt}
\setlength\tabcolsep{2pt}
\centering
\label{tab-symbols}
\begin{tabular}{cc}
\midrule
\textbf{Notations} & \textbf{Definitions or Descriptions}\\
\midrule
$\mathbb{G}, \mathbb{G}^c$ & Random variable of universal graphs and graphs from domain $c$\\
$\mathcal{G}, \mathcal{G}^c$ & Set of universal graphs and graphs from domain c\\
$\mathcal{G}^{\text{Train, Val, Test, c}}$ & Set of training/validation/testing graphs from domain $c$\\
$P(\mathbb{G}), P(\mathbb{G}^c)$ & Distribution of universal graphs and graphs from domain c\\
$G=(\mathbf{X}^G, \mathbf{E}^G)$ & Graph $G$ with node/edge category matrices $\mathbf{X}^G, \mathbf{E}^G$\\
$n_G$ & Number of nodes in graph $G$\\
$d_{X}/d_{E}$ & Number of node/edge categories\\
$\mathbf{Q}_{X}^{t}, \mathbf{Q}_{E}^{t}$ & Node/edge transition matrices\\
$\bar{\mathbf{Q}}_{X}^{t}, \bar{\mathbf{Q}}_{E}^{t}$ & Node/edge accumulative transition matrices\\
$\mathbf{m}_{X}^c, \mathbf{m}_{E}^c$ & Distribution of node/edge categories of graphs from domain c\\
$t, \mathcal{T}$ & Diffusion step $t$ and the set of total steps $\mathcal{T}$\\
$\widetilde{\mathbb{G}}$ & Distribution of graphs from unseen domains\\
$\widetilde{\mathcal{G}}$ & Set of graphs from unseen domains\\
$S, \phi(S)$ & Text with its embedding from the pre-trained textual encoder $\phi$\\
$P(\mathbb{G}, \mathbb{S})$ & Joint distribution of graphs and their textual descriptions\\
$\boldsymbol{\Theta}$ & Parameters of Neural Networks\\
$\boldsymbol{\Theta}^{\star}$ & Optimal Parameters of Neural Networks after pre-training\\
$\boldsymbol{\Theta}^{\star\star}$ & Optimal Parameters of Neural Networks after fine-tuning\\
$\boldsymbol{\Theta}^{\blacktriangle}$ & Optimal Parameters of Neural Networks after Text2Graph Generation\\
\textsc{FB} & Facebook Networks\\
\textsc{ASN} & Animal Social Networks\\
\textsc{Email} & Email Networks\\
\textsc{Web} & Web Graphs\\
\textsc{Road} & Road Networks\\
\textsc{Power} & Power Networks\\
\textsc{CHEM} & Chemical Networks\\
\textsc{BIO} & Biological Networks\\
\textsc{Econ} & Economic Networks\\
\textsc{RT} & Retweet Networks\\
\textsc{COL} & Collaboration Networks\\
\textsc{ECO} & Ecological Networks\\
\textsc{Citation} & Citation Networks\\
LGGM-X & Pre-trained LGGM on all other domains except X\\
Fine-tuned LGGM on X & Fine-tuned LGGM-X on domain X\\
LGGM-T2G & LGGM trained on graphs paired with texts\\
LGGM-T2G$^{\text{D}}$ & LGGM trained on graphs with texts on domains\\
LGGM-T2G$^{\text{UP}}$ & LGGM trained on graphs with user prompts on domains/names\\
LGGM & LGGM trained on all graphs from all domains\\
\bottomrule
\end{tabular}
\end{table}

\newpage 
\section{Proof of Theorems}\label{app-theorem}

\begin{thm}\label{thm-text2graph}
If the transition matrices $\mathbf{Q}_{X}^{t}, \mathbf{Q}_{E}^{t}$ in Eq.~\eqref{eq-forward} are independent of the textual description $\mathbb{S}$, then we have $P(\mathbb{G}^{t - 1}|\mathbb{G}^{t}, \mathbb{G}, \mathbb{S}) \propto P(\mathbb{G}^t|\mathbb{G}^{t - 1})P(\mathbb{G}^{t - 1}|\mathbb{G})$ and correspondingly, we have the analytical formed solution, i.e., $P(\mathbf{X}^{t - 1}|\mathbf{X}^{t}, \mathbf{X}, S) \propto \mathbf{X}^{t}(\mathbf{Q}_X^t)^{\top} \odot \mathbf{X}\bar{\mathbf{Q}}^{t - 1}_{X}$, $P(\mathbf{E}^{t - 1}|\mathbf{E}^{t}, \mathbf{E}, S) \propto \mathbf{E}^{t}(\mathbf{Q}_E^t)^{\top} \odot \mathbf{E}\bar{\mathbf{Q}}^{t - 1}_{E}$ following~\cite{vignac2022DiGress}.
\end{thm}

\begin{proof}
Applying the Bayes rule, we have:

\begin{align}\label{eq-bayes}
    P(\mathbb{G}^{t-1}|\mathbb{G}^{t}, \mathbb{G}, \mathbb{S}) &\propto P(\mathbb{G}^{t-1}, \mathbb{G}^{t}, \mathbb{G}, \mathbb{S}) \propto P(\mathbb{G}^{t}|\mathbb{G}^{t - 1}, \mathbb{G}, \mathbb{S})P(\mathbb{G}^{t - 1}, \mathbb{G}, \mathbb{S}) \\
    & \propto P(\mathbb{G}^t|\mathbb{G}^{t - 1}, \mathbb{G}, \mathbb{S})P(\mathbb{G}^{t - 1}|\mathbb{G}, \mathbb{S})P(\mathbb{G}, \mathbb{S}).
\end{align}
Given the independence of the transition matrix on the textual description $S$ and also the noise is Markovian~\cite{vignac2022DiGress}, we have $P(\mathbb{G}^{t}|\mathbb{G}^{t - 1}, \mathbb{G}, \mathbb{S}) = P(\mathbb{G}^t|\mathbb{G}^{t - 1})$, $P(\mathbb{G}^{t - 1}|\mathbb{G}, \mathbb{S}) = P(\mathbb{G}^{t - 1}|\mathbb{G})$, and also the irrelevance of $P(\mathbb{G}, \mathbb{S})$ to $P(\mathbb{G}^{t - 1}|\mathbb{G}^{t}, \mathbb{G}, \mathbb{S})$, we then end up with:
\begin{equation}
    P(\mathbb{G}^{t-1}|\mathbb{G}^{t}, \mathbb{G}, \mathbb{S}) \propto P(\mathbb{G}^{t}|\mathbb{G}^{t - 1})P(\mathbb{G}^{t - 1}|\mathbb{G}).
\end{equation}

Since the distribution of graphs can be decomposed into the distribution of node and edge categories, following~\cite{vignac2022DiGress}, we similarly have:
\begin{equation}
    P(\mathbf{X}^{t - 1}|\mathbf{X}^{t}, \mathbf{X}, S)\propto P(\mathbf{X}^t|\mathbf{X}^{t - 1})P(\mathbf{X}^{t - 1}|\mathbf{X}) = \mathbf{X}^t(\mathbf{Q}_X^t)^{\top}\odot \mathbf{X}\bar{\mathbf{Q}}_X^{t - 1},
\end{equation}
\begin{equation}
    P(\mathbf{E}^{t - 1}|\mathbf{E}^{t}, \mathbf{E}, S)\propto P(\mathbf{E}^t|\mathbf{E}^{t - 1})P(\mathbf{E}^{t - 1}|\mathbf{E}) = \mathbf{E}^t(\mathbf{Q}_E^t)^{\top}\odot \mathbf{E}\bar{\mathbf{Q}}_E^{t - 1}.
\end{equation}
\end{proof}

\begin{thm}\label{thm-elbo}
Given the decomposition in Eq.~\eqref{eq-generate-text} that $P(\mathbb{G}^{t-1}|\mathbb{G}^{t}, \mathbb{S}) \propto \sum_{\mathbb{G}}P(\mathbb{G}^{t-1}|\mathbb{G}^{t}, \mathbb{G}, \mathbb{S})P(\mathbb{G}|\mathbb{G}^{t}, \mathbb{S})$, optimizing $\boldsymbol{\Theta}$ according to Eq.~\eqref{eq-text2graph-loss} essentially optimizes the variational lower bound of the log-likelihood $P_{\boldsymbol{\Theta}}(\mathbb{G}^0, \mathbb{S})$.
\end{thm}

\begin{proof}
We start directly from the log-likelihood of the joint distribution of $P_{\Theta}(\mathbb{G}^0, \mathbb{S})$:
\begin{align}
    \log P_{\boldsymbol{\Theta}}(\mathbb{G}^0, \mathbb{S}) 
    &= \log \int P_{\boldsymbol{\Theta}}(\mathbb{G}^0, \mathbb{S}, \mathbb{G}^1, ..., \mathbb{G}^T)d(\mathbb{G}^1, \mathbb{G}^2, ..., \mathbb{G}^T)\label{eq-vlb-first} \\
    & = \log \int \frac{P_{\boldsymbol{\Theta}}(\mathbb{G}^0, \mathbb{S}, \mathbb{G}^1, ..., \mathbb{G}^T)}{q(\mathbb{G}^1, \mathbb{G}^2, ..., \mathbb{G}^T)}q(\mathbb{G}^1, \mathbb{G}^2, ..., \mathbb{G}^T)d(\mathbb{G}^1, \mathbb{G}^2, ..., \mathbb{G}^T) \\
    & = \log \mathbb{E}_{q(\mathbb{G}^1, \mathbb{G}^2, ..., \mathbb{G}^T)}\frac{P_{\boldsymbol{\Theta}}(\mathbb{G}^0, \mathbb{S}, \mathbb{G}^1, ..., \mathbb{G}^T)}{q(\mathbb{G}^1, \mathbb{G}^2, ..., \mathbb{G}^T)} \\
    & \ge \mathbb{E}_{q(\mathbb{G}^1, \mathbb{G}^2, ..., \mathbb{G}^T)}\log \frac{P_{\boldsymbol{\Theta}}(\mathbb{G}^0, \mathbb{S}, \mathbb{G}^1, ..., \mathbb{G}^T)}{q(\mathbb{G}^1, \mathbb{G}^2, ..., \mathbb{G}^T)} ~~~~~~~~\text{by Jensen’s inequality}\\
    & = \mathbb{E}_{q(\mathbb{G}^1, \mathbb{G}^1, ..., \mathbb{G}^T)}\log \frac{P(\mathbb{G}^{T}, \mathbb{S})\prod_{t = 1}^{T}P_{\boldsymbol{\Theta}}(\mathbb{G}^{t - 1}|\mathbb{G}^{t}, \mathbb{S})}{q(\mathbb{G}^1)\prod_{t = 2}^{T}q(\mathbb{G}^{t}|\mathbb{G}^{t - 1})}~~~~\text{by Markovian}\\
    & = \mathbb{E}_{q(\mathbb{G}^0, \mathbb{G}^1, ..., \mathbb{G}^T)}[\log P(\mathbb{G}^T, \mathbb{S}) + \sum_{t = 1}^{T}\log \frac{P_{\boldsymbol{\Theta}}(\mathbb{G}^{t - 1}|\mathbb{G}^{t}, \mathbb{S})}{q(\mathbb{G}^t|\mathbb{G}^{t - 1})} ] + \text{const}.\label{eq-vlb-last}
\end{align}

According to the decomposition in Eq~\eqref{eq-generate}, optimizing $\boldsymbol{\Theta}$ according to Eq.~\eqref{eq-text2graph-loss} leads to optimizing $P_{\boldsymbol{\Theta}}(\mathbb{G}^{t - 1}|\mathbb{G}^t, \mathbb{S})$, which corresponds to the second term in Eq.~\eqref{eq-vlb-last} and subsequently optimizes the variational lower bound of the log-likelihood $P_{\boldsymbol{\Theta}}(\mathbb{G}^0, \mathbb{S})$ according to the derivation from Eq.~\eqref{eq-vlb-first} to Eq.~\eqref{eq-vlb-last}. Therefore, training Text-to-Graph LGGM according to Eq.~\eqref{eq-text2graph-loss} enables the model to generate graphs such that the pairs of texts and graphs end up with higher likelihoods.

\end{proof}

\newpage
\section{Data Preparation}\label{app-data}
\subsection{Pre-processed Graphs for Training LGGMs}\label{app-graph-data}
We select graphs from the Network Repository across 13 distinct yet representative domains covering a wide variety of real-world scenarios, including Facebook (FB), Animal Social (ASN), Email, Web, Road, Power, Chemical (CHEM), Biological (BIO), Economic (ECON), Retweet (RT), Collaboration (COL), Ecological (ECO), Citation. Due to the scalability issue with diffusion-based graph generative models, we further sample subgraphs for certain domains, and Table~\ref{tab:network_summary} presents the comprehensive statistics of the sampled subgraphs, which are used for training LGGMs. We can see that graphs from different domains are statistically different.

\begin{table}[htbp!]
\tiny
\setlength{\extrarowheight}{.1pt}
\setlength\tabcolsep{4pt}
\centering
\caption{Summary of Graph Statistics. Facebook (FB), Animal Social (ASN), Email, Web, Road, Power, Chemical (CHEM), Biological (BIO), Economic (ECON), Retweet (RT), Collaboration (COL), Ecological (ECO), Citation.}
\begin{tabular}{lccccccccc}
\hline
\textbf{Category} & \textbf{\makecell{Num\\ Nodes}} & \textbf{\makecell{Num\\ Edges}} & \textbf{\makecell{Avg\\ Degree}} & \textbf{\makecell{Avg\\ Clustering}} & \textbf{\makecell{Max\\ Nodes}} & \textbf{\makecell{Min\\ Nodes}} & \textbf{\makecell{Max\\ Edges}} & \textbf{\makecell{Min\\ Edges}} & \textbf{\makecell{Num\\ Graphs}} \\ \hline
ASN & $52.47 \pm 40.13$ & $77.59 \pm 80.95$ & $2.62 \pm 1.52$ & $0.395 \pm 0.178$ & 283 & 3 & 515 & 2 & 267 \\
BIO & $191.14 \pm 43.47$ & $965.71 \pm 878.35$ & $9.16 \pm 7.69$ & $0.276 \pm 0.199$ & 258 & 109 & 4392 & 96 & 504 \\
CHEM & $36.46 \pm 20.49$ & $64.61 \pm 26.23$ & $3.75 \pm 0.63$ & $0.421 \pm 0.223$ & 125 & 2 & 149 & 1 & 646 \\
Citation & $235.91 \pm 27.25$ & $1287.16 \pm 1087.00$ & $10.17 \pm 8.14$ & $0.369 \pm 0.224$ & 270 & 175 & 4474 & 188 & 504 \\
COL & $174.26 \pm 53.82$ & $312.56 \pm 176.33$ & $3.41 \pm 1.24$ & $0.497 \pm 0.203$ & 247 & 52 & 996 & 68 & 504 \\
ECO & $100.67 \pm 30.10$ & $1490.00 \pm 673.87$ & $27.72 \pm 7.00$ & $0.406 \pm 0.082$ & 128 & 54 & 2106 & 353 & 6 \\
ECON & $144.18 \pm 35.82$ & $3258.76 \pm 3540.28$ & $39.76 \pm 37.80$ & $0.419 \pm 0.296$ & 219 & 90 & 11142 & 188 & 504 \\
Email & $146.67 \pm 35.86$ & $681.55 \pm 500.28$ & $9.79 \pm 7.26$ & $0.389 \pm 0.211$ & 213 & 82 & 2909 & 216 & 504 \\
Power & $132.22 \pm 20.29$ & $289.32 \pm 183.02$ & $4.35 \pm 2.31$ & $0.161 \pm 0.164$ & 187 & 81 & 1332 & 133 & 512 \\
Road & $265.25 \pm 94.31$ & $276.46 \pm 79.61$ & $2.70 \pm 2.08$ & $0.078 \pm 0.134$ & 411 & 32 & 456 & 137 & 504 \\
RT & $104.11 \pm 35.23$ & $110.99 \pm 46.44$ & $2.11 \pm 0.37$ & $0.028 \pm 0.038$ & 175 & 35 & 295 & 34 & 558 \\
FB & $219.45 \pm 47.05$ & $1863.44 \pm 701.53$ & $16.36 \pm 6.17$ & $0.315 \pm 0.083$ & 259 & 48 & 3898 & 46 & 504 \\
Web & $173.32 \pm 24.86$ & $462.21 \pm 336.46$ & $5.09 \pm 3.06$ & $0.404 \pm 0.196$ & 231 & 119 & 1607 & 149 & 504 \\ \hline
\end{tabular}
\label{tab:network_summary}
\end{table}

\subsection{Preparation of Graphs and Textual Description About Their Domains/Names}\label{app-graph-name-data}
Here we thoroughly discuss the process of obtaining graphs and their corresponding text prompts describing their domains/names. As given by the Network Repository, we directly download graphs along with their domains/names. We then prompt GPT3.5 to generate user prompts describing the graph given its domain/name. The concrete prompt template we use here is shown in Listing~\ref{list-Domain-Name} with exemplary generated user prompts shown in Listing~\ref{list-Domain-Name-Example}. Moreover, we apply the sentence transformer to obtain text embeddings of the generated prompts for each network and perform t-SNE visualization. As shown in Figure~\ref{fig-tsne-domain}, we see prompts for graphs from different domains from different clusters. More importantly, textual similarity can somewhat reflect their network similarity. For example, prompts for road and power networks are very close, and they both belong to infrastructure. Moreover, Facebook Networks, Email Networks, Collaboration Networks, Web Graphs are very close since all these four belong to some sub-variants of social networks. \textit{This inherent relationship between the textual similarity and structural similarity between two graphs demonstrates that the world knowledge encoded in the text could somehow provide useful preference for the graphs to be generated.}

\begin{figure*}[htbp!]
     \centering
     \begin{subfigure}[b]{0.325\textwidth}
         \centering
         \includegraphics[width=\textwidth]{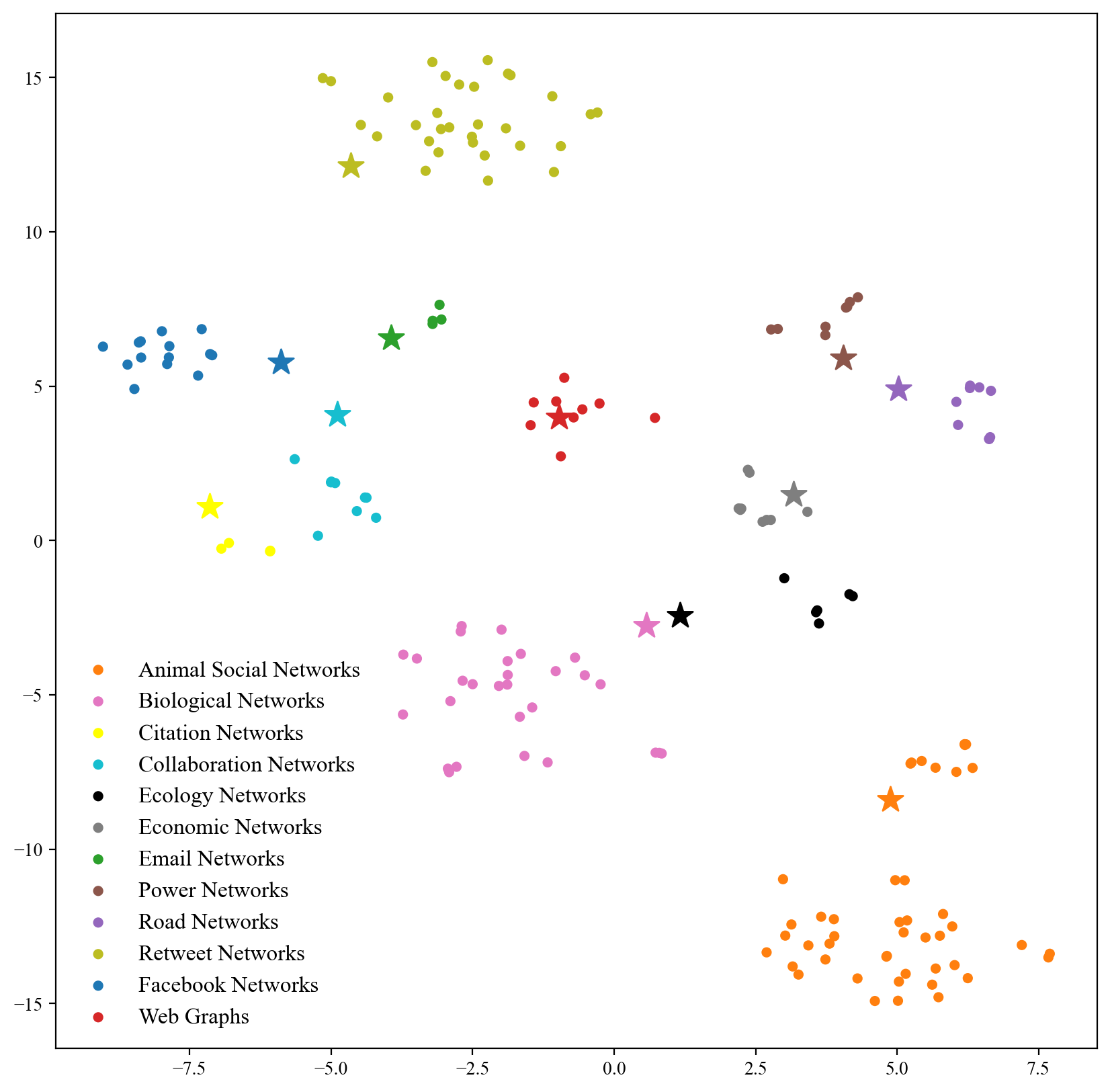}
         \caption{Domain and Name}
         \label{fig-tsne-domain}
     \end{subfigure}
     \hfill
     \begin{subfigure}[b]{0.325\textwidth}
         \centering
         \includegraphics[width=\textwidth]{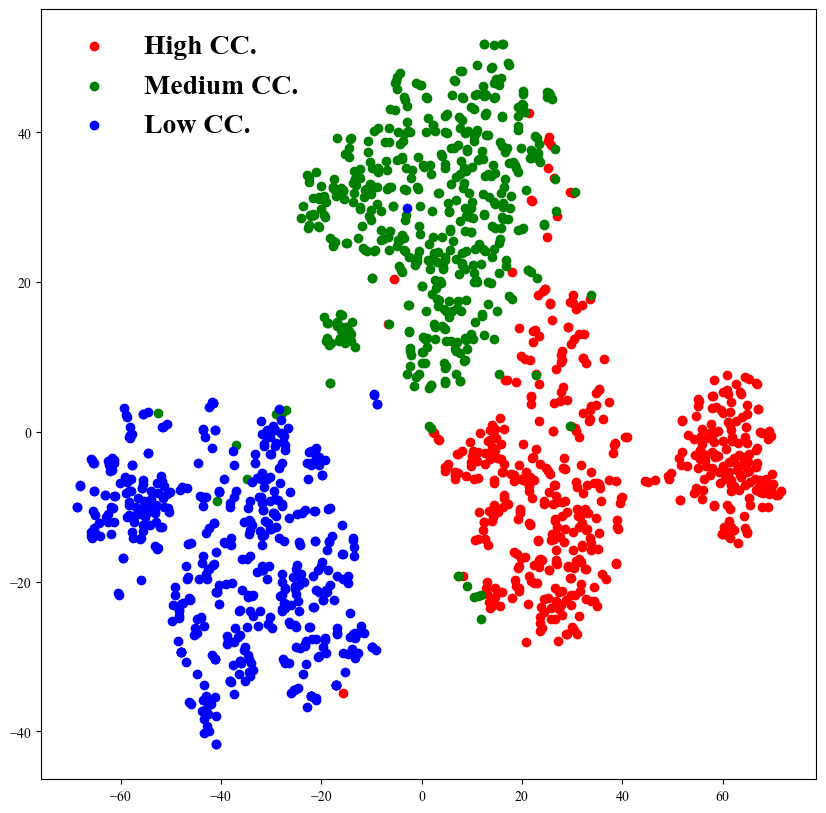}
         \caption{Average Clustering Coefficient}
         \label{fig-tsne-cc}
     \end{subfigure}
     \hfill
     \begin{subfigure}[b]{0.325\textwidth}
         \centering
         \includegraphics[width=\textwidth]{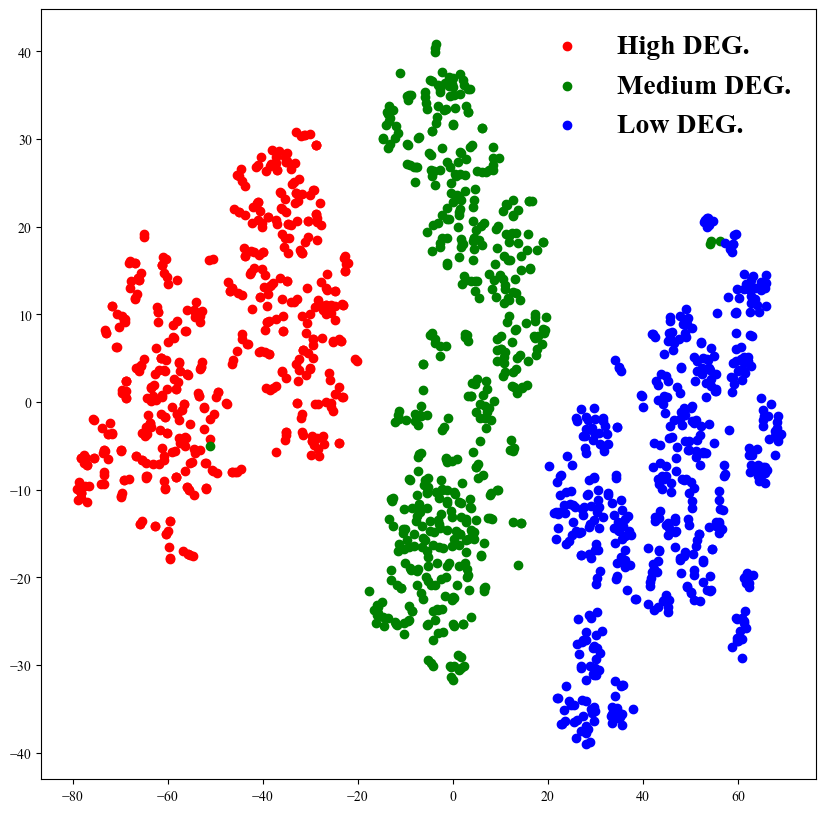}
         \caption{Average Degree}
         \label{fig-tsne-deg}
     \end{subfigure}
    \caption{t-SNE visualization of textual description about network (a) domain/name (b) average clustering coefficient (c) average degree.}
    \label{fig-tsne-prompt}
\end{figure*}

\newpage
\colorlet{shadecolor}{gray!10}
\UseRawInputEncoding
\lstset{breaklines=true, columns=fullflexible, backgroundcolor=\color{shadecolor}}
\begin{lstlisting}[basicstyle=\footnotesize, title={Prompt Template for Generating Textual Description about Network Domain/Name.}, caption={Prompt Template for Generating Textual Description about Network Domain/Name.}, label={list-Domain-Name}]
PROMPT: Given a graph called {NETWORK NAME} that is from the DOMAIN {DOMAIN NAME}.
Note:
* Do not generate more than 20 words.
\end{lstlisting}

\colorlet{shadecolor}{gray!10}
\UseRawInputEncoding
\lstset{breaklines=true, columns=fullflexible, backgroundcolor=\color{shadecolor}}
\begin{lstlisting}[basicstyle=\footnotesize, title={Examples of Textual Description about Network Domain/Name.}, caption={Examples of Textual Description about Network Domain/Name.}, label={list-Domain-Name-Example}]
==============================================================================
*DOMAIN: Animal Social Networks
*NAME: reptilia-tortoise-network-sl	
*TEXT: The reptilia-tortoise-network-sl graph represents the social connections among tortoises in the reptile community.
==============================================================================
*DOMAIN: Power Networks
*NAME: power-eris1176
*TEXT: The power-eris1176 graph represents the interconnected nodes and edges of a power network system
==============================================================================
*DOMAIN: Economic Networks
*NAME: econ-poli
*TEXT: The econ-poli graph represents the interconnectedness of economic and political factors in a network.
==============================================================================
*DOMAIN: Ecology Networks
*NAME: eco-evergla
*TEXT: The eco-evergla graph represents the interconnectedness of species in the Everglades ecosystem.
==============================================================================
*DOMAIN: Email Networks
*NAME: email-enron-only
*TEXT: The email-enron-only graph represents the network of email communication within the Enron corporation.
==============================================================================
*DOMAIN: Road Networks
*NAME: road-roadNet-CA
*TEXT: The road-roadNet-CA graph represents the road network in California.
==============================================================================
*DOMAIN: Retweet Networks
*NAME: rt_occupywallstnyc
*TEXT: The graph rt_occupywallstnyc represents retweet relationships in the Occupy Wall Street movement in New York City.
==============================================================================
*DOMAIN: Facebook Networks
*NAME: socfb-Haverford76
*TEXT: The socfb-Haverford76 graph represents the social connections among users in the Haverford College community on Facebook.
==============================================================================
*DOMAIN: Web Graphs
*NAME: web-wiki-chameleon
*TEXT: The web-wiki-chameleon graph represents the interconnections between web pages, Wikipedia articles, and chameleon species.
==============================================================================
*DOMAIN: Biological Networks
*NAME: bio-WormNet-v3-benchmark
*TEXT: The bio-WormNet-v3-benchmark graph represents a biological network related to worms.
==============================================================================
*DOMAIN: Citation Networks
*NAME: cit-DBLP
*TEXT: cit-DBLP is a graph representing the citation relationships between research papers in the field of computer science.
==============================================================================
*DOMAIN: Collaboration Networks
*NAME: ca-netscienc
*TEXT: The ca-netscienc graph represents a collaboration network in the field of science.
==============================================================================
\end{lstlisting}

\subsection{Preparing Graphs and Their Textual Description about Graph Property}\label{app-graph-property-data}
Here we thoroughly discuss the process of obtaining graphs and their corresponding text prompts describing their properties. Our goal is to demonstrate that Text2Graph LGGM can control the statistics of the generated graphs in the full spectrum. However, the graphs obtained directly from the Network Repository do not cover the whole topological space (e.g., Figure~\ref{fig-motivation}(a) shows that no networks have a higher average degree while low clustering coefficient). Therefore, we plan to synthesize graphs covering the whole space by Watts-Strogatz Small-world Graph Model. We vary the number of nodes between [10, 110], the number of initial neighbors between [5, number of nodes], and also the probability of rewiring each edge between [0, 1] to ensure the generated graphs span across the full spectrum. After that, we group the generated graphs into low, medium, and high groups in terms of their clustering coefficient and average degree. We implement this using~\href{https://networkx.org/documentation/stable/reference/generated/networkx.generators.random_graphs.watts_strogatz_graph.html}{NetworkX}.

After we synthesize graphs and divide them into three groups, we generate user prompts paired with these graphs next. Specifically, we prompt GPT4 following the templates in Listing~\ref{list-Property}/\ref{list-Property-Example}. To ensure the compatibility between the synthesis graphs and the generated user prompts. We further replace the number output by GPT4 describing the network property with the real statistic calculated from each network.

\colorlet{shadecolor}{gray!10}
\UseRawInputEncoding
\lstset{breaklines=true, columns=fullflexible, backgroundcolor=\color{shadecolor}}
\begin{lstlisting}[basicstyle=\footnotesize, title={Prompt Template for Generating Textual Description about Network Property.}, caption={Prompt Template for Generating Textual Description about Network Property.}, label={list-Property}]
==============================================================================
PROMPT: Please generate a short sentence about the graph, including its clustering coefficient information.

Note:
* Do not generate more than 20 words.
* Make sure the generated sentence includes the level of clustering coefficient, you can either specify it via words like ['low', 'medium', 'high']. or specify it via numbers like [(0, 0.25), (0.25, 0.5), (0.5, 0.75)]"
* You can also sometimes specify a concrete application scenario of the generated network.
* Please be accurate but also diverse
==============================================================================
PROMPT: Please generate a short sentence about the graph, including its average degree information.

Note:
* Do not generate more than 20 words.
* Make sure the generated sentence includes the level of average degree, you can either specify it via words like ['low', 'medium', 'high']. or specify it via numbers like [(0, 20), (20, 50), (50, 100)]"
* You can also sometimes specify a concrete application scenario of the generated network.
* Please be accurate but also diverse
==============================================================================
\end{lstlisting}

\colorlet{shadecolor}{gray!10}
\UseRawInputEncoding
\lstset{breaklines=true, columns=fullflexible, backgroundcolor=\color{shadecolor}}
\begin{lstlisting}[basicstyle=\footnotesize, title={Examples of Textual Description about Network Property.}, caption={Examples of Textual Description about Network Property.}, label={list-Property-Example}]
==============================================================================
* This graph has a high clustering coefficient, suggesting strong node clustering.
* Please generate a network with a clustering coefficient around 0.61, indicating strong clustering.
* This retirement community's social interaction graph displays a high clustering coefficient of 0.73, indicative of close relationships.
==============================================================================
* With an average degree of 35, this network is ideal for studying urban transportation patterns.
* The graph's moderate connectivity level helps in understanding the structure of small to medium-sized music bands.
* An average degree of 41 makes this network suitable for simulating the collaboration in local artisan markets.
==============================================================================
\end{lstlisting}

\newpage
\section{Experimental Setting}\label{app-expr}
\vspace{-1ex}
\subsection{Evaluation Metrics}\label{app-metrics}
\vspace{-1ex}
Following~\cite{thompson2022evaluation, you2018graphrnn}, we evaluate the graph generation performance by the standard Maximum Mean Discrepancy (MMD) between generated and reference graphs $\mathcal{G}_{g}, \mathcal{G}_{r}$:
\begin{equation}
    \text{MMD}(\mathcal{G}_g, \mathcal{G}_r) = \frac{1}{m^2}\sum_{i,j=1}^{m}k(\mathbf{x}_{i}^r, \mathbf{x}_{j}^r) + \frac{1}{n^2}\sum_{i,j=1}^nk(\mathbf{x}_i^g, \mathbf{x}_j^g) - \frac{2}{nm}\sum_{i=1}^{n}\sum_{j=1}^mk(\mathbf{x}_i^g, \mathbf{x}_j^r),
\end{equation}
where $k(\cdot, \cdot)$ is a general kernel function and specifically we use RBF kernel following~\cite{you2018graphrnn}:
\begin{equation}
    k(\mathbf{x}_i, \mathbf{x}_j) = \exp(-d(\mathbf{x}_i, \mathbf{x}_j)/2\sigma^2),
\end{equation}

where $d(\cdot, \cdot)$ computes pairwise distance following~\cite{vignac2022DiGress} and MMD is evaluated over the distributions of degree (DEG), clustering coefficients (CC), eigenvalues of normalized Laplacian matrix (Spec) and orbits counts representing the distribution of all substructures of size 4 (Orb).

\vspace{-1ex}
\subsection{Hyperparameter Details}\label{app-hyperparameter}
\vspace{-1ex}
For all experiments, we select the best configuration according to the generation performance on validation graphs and report the final performance on generating testing graphs. We adopt the default hyperparameter settings from DiGress~\cite{vignac2022DiGress} with the following exceptions: we generate 100 graphs per domain for each evaluation and set the training epochs at 300 to ensure convergence. Additionally, we implement gradient accumulation, using a mini-batch size of 12 across 4 accumulations, resulting in an effective batch size of 48. For Text-to-Graph Generation, the textual encoder used to obtain textual description embeddings is "all-MiniLM-L6-v2". All experiments are performed on a machine with A100-80G GPU RAM and 128GB RAM.

\vspace{-1ex}
\subsection{Paradigm Setup}\label{app-expr-setup}
\vspace{-1ex}
Figure~\ref{fig-experiment-setup} comprehensively visualizes the training/evaluation paradigms of the four experiments, the details of which are discussed in Section~\ref{sec-expr-setup}.
\begin{figure*}[htbp!]
     \centering
     \includegraphics[width=0.98\textwidth]{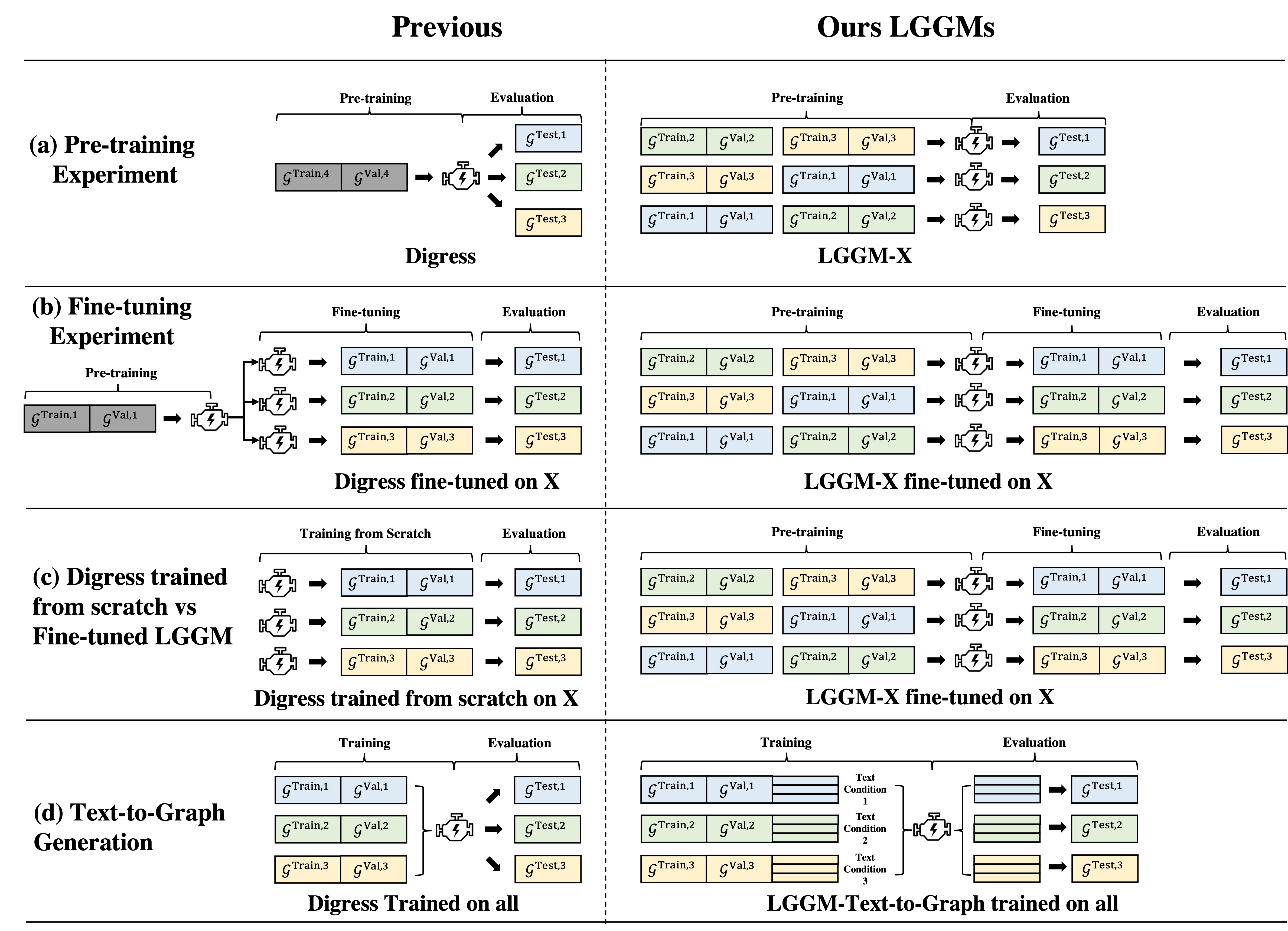}
     \caption{Comprehensive Overview of the Experimental Setup for our LGGMs.}
     \label{fig-experiment-setup}
\end{figure*}

\newpage
\section{Full Experimental Results}\label{app-expr-full}
\subsection{Out-of-domain Performance Comparison between DiGress and LGGM}\label{app-pretrain}
\subsubsection{Domain Specific Transition Strategy}
\begin{table*}[htbp!]
\scriptsize
\setlength\tabcolsep{5.5pt}
\caption{Comparing Zero-shot Generation Performance on Unseen Graphs in domain X between DiGress trained on QM9 and LGGM-X pre-trained on all domains except the held-out domain X.}
\label{app-tab-pretrain-marginal}
\centering
\begin{tabular}{llcccc|llcccc}
\toprule
\textbf{Domain} & \textbf{Method} & \textbf{DEG} & \textbf{CC} & \textbf{Spec} & \textbf{Orb} & \textbf{Domain} & \textbf{Method} & \textbf{DEG} & \textbf{CC} & \textbf{Spec} & \textbf{Orb} \\
\toprule
\multirow{2}{*}{\textsc{FB}} & DiGress & \textbf{0.2695} & \textbf{0.3452} & \textbf{0.0649} & \textbf{0.1489} & \multirow{2}{*}{BIO} & DiGress & 0.2419 & \textbf{0.2993} & \textbf{0.1101} & \textbf{0.2978} \\
 & LGGM-X & 0.4962 & 0.7625 & 0.3408 & 0.7982 &  & LGGM-X & \textbf{0.2117} & 0.6365 & 0.1690 & 0.5156\\
 \midrule
\multirow{2}{*}{\textsc{ASN}} & DiGress & 0.1793 & 0.4721 & 0.1751 & 0.5654 & \multirow{2}{*}{\textsc{ECON}} & DiGress & 0.2811 & 0.2042 & 0.2028 & 0.2633 \\
 & LGGM-X & \textbf{0.0220} & \textbf{0.4044} & \textbf{0.1274} & \textbf{0.0505} & & LGGM-X & \textbf{0.1916} & \textbf{0.0917} & \textbf{0.1219} & \textbf{0.0640} \\
 \midrule
\multirow{2}{*}{\textsc{Email}} & DiGress & \textbf{0.2312} & \textbf{0.5444} & \textbf{0.0674} & \textbf{0.2650} & \multirow{2}{*}{\textsc{RT}} & DiGress & 0.4466 & 0.4170 & 0.4483 & 0.4551 \\
 & LGGM-X & 0.2618 & 0.8650 & 0.3013 & 1.0459 & & LGGM-X & \textbf{0.0721} & \textbf{0.0517} & \textbf{0.2331} & \textbf{0.4085} \\
 \midrule
\multirow{2}{*}{\textsc{Web}} & DiGress & 0.2575 & \textbf{0.5955} & 0.1907 & 0.9282 & \multirow{2}{*}{\textsc{Col}} & DiGress & 0.2393 & \textbf{0.5341} & 0.2247 & 0.7619 \\
 & LGGM-X &  \textbf{0.1491} & 0.9436 & \textbf{0.1154} & \textbf{0.4016} & & LGGM-X & \textbf{0.1493} & 0.9200 & \textbf{0.1786} & \textbf{0.2057}\\
 \midrule
\multirow{2}{*}{\textsc{ROAD}} & DiGress & 0.4111 & 0.6653 & 0.3084 & 0.6530 & \multirow{2}{*}{\textsc{Eco}} & DiGress & 0.4580 & 0.4546 & 0.2144 & 0.4417 \\
 & LGGM-X & \textbf{0.0379} & \textbf{0.1191} & \textbf{0.0759} & \textbf{0.0401} & & LGGM-X & \textbf{0.2049} & \textbf{0.2760} & \textbf{0.0691} & \textbf{0.2107}  \\
 \midrule
\multirow{2}{*}{\textsc{Power}} & DiGress & 0.5292 & \textbf{0.6083} & 0.3556 & 1.2124 & \multirow{2}{*}{\textsc{Citation}} & DiGress & 0.3159 & \textbf{0.3664} & 0.1299 & \textbf{0.2278}\\
 & LGGM-X & \textbf{0.0343} & 0.6290 & \textbf{0.0649} & \textbf{0.0228} & & LGGM-X & \textbf{0.1314} & 0.8908 & \textbf{0.1188} & 0.6391 \\
 \midrule
\multirow{2}{*}{\textsc{All}} & DiGress & 0.3217 & \textbf{0.4589} & 0.2077 & 0.5184  \\
 & LGGM-X & \textbf{0.1635} & 0.5492 & \textbf{0.1597} & \textbf{0.3669}  \\
 \bottomrule
\end{tabular}
\end{table*}

\subsubsection{Uniform Transition Strategy}
\begin{table*}[htbp!]
\scriptsize
\setlength\tabcolsep{5.5pt}
\caption{Comparing Zero-shot Generation Performance on Unseen Graphs in domain X between DiGress trained on QM9 and LGGM-X pre-trained on all domains except the held-out domain X.}
\label{app-tab-pretrain-uniform}
\centering
\begin{tabular}{llcccc|llcccc}
\toprule
\textbf{Domain} & \textbf{Method} & \textbf{DEG} & \textbf{CC} & \textbf{Spec} & \textbf{Orb} & \textbf{Domain}  & \textbf{Method} & \textbf{DEG} & \textbf{CC} & \textbf{Spec} & \textbf{Orb} \\
\toprule
\multirow{2}{*}{\textsc{FB}} & DiGress & \textbf{0.3376} & \textbf{0.6298} & \textbf{0.0797} & \textbf{0.3593} & \multirow{2}{*}{BIO} & DiGress & 0.2712 & 0.5202 & 0.1127 & 0.3188 \\
 & LGGM-X & 0.4723 & 0.6843 & 0.2924 & 0.7555 &  & LGGM-X & \textbf{0.1081} & \textbf{0.2696} & \textbf{0.0900} & \textbf{0.2053}\\
 \midrule
\multirow{2}{*}{\textsc{ASN}} & DiGress & 0.1496 & 0.3258 & 0.1506 & 0.4420 & \multirow{2}{*}{\textsc{ECON}} & DiGress & 0.2987 & 0.4841 & 0.2162 & 0.3834 \\
 & LGGM-X & \textbf{0.0281} & \textbf{0.2440} & \textbf{0.0830} & \textbf{0.0618} &  & LGGM-X & \textbf{0.1213} & \textbf{0.0920} & \textbf{0.1120} & \textbf{0.1086} \\
 \midrule
\multirow{2}{*}{\textsc{Email}} & DiGress & 0.2192 & 0.6012 & 0.0702 & 0.3416 & \multirow{2}{*}{\textsc{RT}}  & DiGress & 0.4164 & 0.1327 & 0.4147 & 0.5957 \\
 & LGGM-X & \textbf{0.0751} & \textbf{0.2364} & \textbf{0.0768} & \textbf{0.3089} &  & LGGM-X & \textbf{0.0525} & \textbf{0.1429} & \textbf{0.1330} & \textbf{0.2219} \\
 \midrule
\multirow{2}{*}{\textsc{Web}} & DiGress & 0.2556 & 0.6186 & 0.1877 & 0.6045 & \multirow{2}{*}{\textsc{Col}}  & DiGress & 0.2473 & 0.5826 & 0.2314 & 0.7679 \\
 & LGGM-X &  \textbf{0.0648} & \textbf{0.3961} & \textbf{0.0549} & \textbf{0.1127} &  & LGGM-X & \textbf{0.0736} & \textbf{0.5769} & \textbf{0.0895} & \textbf{0.0988}\\
 \midrule
\multirow{2}{*}{\textsc{ROAD}} & DiGress & 0.3705 & 0.8226 & 0.2801 & 0.7198 & \multirow{2}{*}{\textsc{Eco}}  & DiGress & 0.5431 & 0.7915 & 0.2338 & 0.6045 \\
 & LGGM-X & \textbf{0.0713} & \textbf{0.2193} & \textbf{0.0987} & \textbf{0.2986} & & LGGM-X & \textbf{0.4753} & \textbf{0.3904} & 0.3194 & \textbf{0.3934}  \\
 \midrule
\multirow{2}{*}{\textsc{Power}} & DiGress & 0.3726 & 0.4582 & 0.3270 & 1.4732 & \multirow{2}{*}{\textsc{Citation}}  & DiGress & 0.2527 & 0.7790 & 0.1315 & \textbf{0.4966}\\
 & LGGM-X & \textbf{0.0119} & \textbf{0.1293} & \textbf{0.0373} & \textbf{0.0754} &  & LGGM-X & \textbf{0.1348} & \textbf{0.7257} & \textbf{0.1160} & 0.4981 \\
 \midrule
 \multirow{2}{*}{\textsc{All}} & DiGress & 0.3112 & 0.5622 & 0.2030 & 0.5923\\
 & LGGM-X & \textbf{0.1408} & \textbf{0.3422} & \textbf{0.1253} & \textbf{0.2616} \\
 \bottomrule
\end{tabular}
\end{table*}

\newpage
\subsection{Performance Comparison between Fine-tuned DiGress and Fine-tuned LGGM}\label{app-expr-fine-tune}
\subsubsection{Domain Specific Transition Strategy}
\begin{table*}[htbp!]
\scriptsize
\setlength\tabcolsep{5.5pt}
\caption{Comparing Graph Generation Performance between Fine-tuned DiGress and Fine-tuned LGGM on each domain. DiGress-FT: DiGress pre-trained on QM9 and fine-tuned on domain X; LGGM-FT: LGGM pre-trained on all other domains except X and fine-tuned on X under \textbf{Domain Specific Transition Strategy}.}
\label{app-tab-ft-marginal}
\centering
\begin{tabular}{llccccllcccc}
\toprule
\textbf{Domain} & \textbf{Method} & \textbf{DEG} & \textbf{CC} & \textbf{Spec} & \textbf{Orb} & \textbf{Domain} & \textbf{Method} & \textbf{DEG} & \textbf{CC} & \textbf{Spec} & \textbf{Orb} \\
\toprule
\multirow{2}{*}{FB} & DiGress-FT & 0.0159 & 0.0564 & 0.0082 & 0.0298 & \multirow{2}{*}{BIO} & DiGress-FT & 0.0391 & 0.0354 & 0.0347 & 0.0291 \\
 & LGGM-FT & \textbf{0.0065} & \textbf{0.0544} & \textbf{0.0069} & \textbf{0.0282} & & LGGM-FT & \textbf{0.0036} & \textbf{0.0303} & \textbf{0.0102} & \textbf{0.0342} \\
 \midrule
\multirow{2}{*}{ASN} & DiGress-FT & 0.0189 & 0.0775 & 0.0729 & 0.0886 & \multirow{2}{*}{ECON} & DiGress-FT & 0.0301 & 0.0431 & 0.0372 & 0.0392 \\
 & LGGM-FT & \textbf{0.0014} & \textbf{0.0509} & \textbf{0.0161} & \textbf{0.0084} & & LGGM-FT & \textbf{0.0215} & \textbf{0.0330} & \textbf{0.0062} & \textbf{0.0249} \\
 \midrule
\multirow{2}{*}{EMAIL} & DiGress-FT & 0.0208 & 0.0448 & 0.0230 & \textbf{0.0447} & \multirow{2}{*}{RT} & DiGress-FT & 0.0054 & 0.0464 & 0.0051 & 0.0437 \\
 & LGGM-FT & \textbf{0.0166} & \textbf{0.0364} & \textbf{0.0104} & 0.0463 & & LGGM-FT & \textbf{0.0012} & \textbf{0.0075} & \textbf{0.0033} & \textbf{0.0162} \\
 \midrule
\multirow{2}{*}{WEB} & DiGress-FT & 0.0192 & 0.0808 & 0.0664 & 0.1361 & \multirow{2}{*}{COL} & DiGress-FT & 0.0255 & 0.2279 & 0.0788 & 0.0731 \\
  & LGGM-FT & \textbf{0.0116} & \textbf{0.0721} & \textbf{0.0152} & \textbf{0.0656} & & LGGM-FT & \textbf{0.0202} & \textbf{0.1621} & \textbf{0.0571} & \textbf{0.0631} \\
 \midrule
\multirow{2}{*}{ROAD} & DiGress-FT & 0.0907 & 0.1404 & 0.1099 & 0.1097 & \multirow{2}{*}{ECO} & DiGress-FT & 0.1370 & 0.2747 & 0.0476 & 0.2109 \\
 & LGGM-FT & \textbf{0.0088} & \textbf{0.1349} & \textbf{0.0347} & \textbf{0.0125} & & LGGM-FT & \textbf{0.0196} & \textbf{0.2343} & \textbf{0.0291} & \textbf{0.2100} \\
 \midrule
\multirow{2}{*}{POWER} & DiGress-FT & 0.0104 & 0.2197 & 0.1023 & 0.0445 & \multirow{2}{*}{CITATION} & DiGress-FT & 0.0363 & 0.1140 & 0.0469 & 0.0423 \\
 & LGGM-FT & \textbf{0.0008} & \textbf{0.1539} & \textbf{0.0215} & \textbf{0.0081} & & LGGM-FT & \textbf{0.0078} & \textbf{0.0827} & \textbf{0.0137} & \textbf{0.0316} \\
 \midrule
\multirow{2}{*}{All} & DiGress-FT & 0.0374 & 0.1134 & 0.0528 & 0.0743 & \\
 & LGGM-FT & \textbf{0.0010} & \textbf{0.0877} & \textbf{0.0187} & \textbf{0.0458} &  &  &  &  & \\
 \bottomrule
\end{tabular}
\end{table*}

\subsubsection{Uniform Transition Strategy}
\begin{table*}[htbp!]
\scriptsize
\setlength\tabcolsep{5.5pt}
\caption{Comparing Graph Generation Performance between Fine-tuned DiGress and Fine-tuned LGGM on each domain. DiGress-FT: DiGress pre-trained on QM9 and fine-tuned on domain X; LGGM-FT: LGGM pre-trained on all other domains except X and fine-tuned on X under \textbf{Uniform Transition Strategy}.}
\label{app-tab-ft-uniform}
\centering
\begin{tabular}{llccccllcccc}
\toprule
\textbf{Domain} & \textbf{Method} & \textbf{DEG} & \textbf{CC} & \textbf{Spec} & \textbf{Orb} & \textbf{Domain} & \textbf{Method} & \textbf{DEG} & \textbf{CC} & \textbf{Spec} & \textbf{Orb} \\
\toprule
\multirow{2}{*}{FB} & DiGress-FT & \textbf{0.0039} & 0.0650 & 0.0090 & 0.0304 & \multirow{2}{*}{BIO} & DiGress-FT & 0.0274 & 0.0845 & 0.0493 & 0.0312 \\
 & LGGM-FT & \textbf{0.0050} & \textbf{0.0579} & \textbf{0.0059} & \textbf{0.0280} &  & LGGM-FT & \textbf{0.0049} & \textbf{0.0496} & \textbf{0.0056} & \textbf{0.0257} \\
\midrule
\multirow{2}{*}{ASN} & DiGress-FT & 0.0249 & 0.5604 & 0.0779 & 0.0348 & \multirow{2}{*}{ECON}  & DiGress-FT & \textbf{0.0133} & \textbf{0.0355} & 0.0223 & \textbf{0.0360} \\
 & LGGM-FT & \textbf{0.0058} & \textbf{0.1098} & \textbf{0.0311} & \textbf{0.0101} &  & LGGM-FT & 0.0597 & 0.0594 & \textbf{0.0216} & 0.0535 \\
\midrule
\multirow{2}{*}{EMAIL} & DiGress-FT & 0.0134 & 0.0709 & 0.0223 & 0.0694 & \multirow{2}{*}{RT}  & DiGress-FT & 0.0418 & 0.0243 & 0.0495 & 0.0583 \\
 & LGGM-FT & \textbf{0.0120} & \textbf{0.0559} & \textbf{0.0158} & \textbf{0.0444} &  & LGGM-FT & \textbf{0.0032} & \textbf{0.0163} & \textbf{0.0051} & \textbf{0.0227} \\
\midrule
\multirow{2}{*}{WEB} & DiGress-FT & 0.0327 & 0.2025 & 0.0858 & 0.2033 & \multirow{2}{*}{COL}  & DiGress-FT & \textbf{0.0562} & 0.7070 & \textbf{0.1086} & 0.1471 \\
 & LGGM-FT & \textbf{0.0218} & \textbf{0.1398} & \textbf{0.0310} & \textbf{0.1262} &  & LGGM-FT & 0.1074 & \textbf{0.4265} & 0.1398 & \textbf{0.0897} \\
\midrule
\multirow{2}{*}{ROAD} & DiGress-FT & 0.0843 & 0.1010 & 0.1873 & 0.5155 & \multirow{2}{*}{ECO}  & DiGress-FT & 0.1118 & 0.3016 & 0.0548 & 0.2102 \\
 & LGGM-FT & \textbf{0.0081} & \textbf{0.0547} & \textbf{0.0573} & \textbf{0.0228} &  & LGGM-FT & \textbf{0.0204} & \textbf{0.2347} & \textbf{0.0404} & \textbf{0.2100} \\
\midrule
\multirow{2}{*}{POWER} & DiGress-FT & 0.0231 & 0.1029 & 0.0683 & 0.0441 & \multirow{2}{*}{CITATION}  & DiGress-FT & 0.0277 & 0.1622 & 0.0501 & 0.0813 \\
 & LGGM-FT & \textbf{0.0077} & \textbf{0.0570} & \textbf{0.0134} & \textbf{0.0040} &  & LGGM-FT & \textbf{0.0052} & \textbf{0.0821} & \textbf{0.0221} & \textbf{0.0443} \\
 \midrule
\multirow{2}{*}{All} & DiGress-FT & 0.0384 & 0.2015 & 0.0654 & 0.1218 \\
 & LGGM-FT & \textbf{0.0218} & \textbf{0.1120} & \textbf{0.0324} & \textbf{0.0568}\\
\bottomrule
\end{tabular}
\end{table*}

\newpage
\subsection{Performance Comparison between DiGress directly trained on X and Fine-tuned LGGM}\label{app-expr-train-scratch}

\subsubsection{Domain Specific Transition}
\begin{table*}[htbp!]
\scriptsize
\setlength\tabcolsep{5.3pt}
\caption{Comparing Graph Generation Performance between DiGress and Fine-tuned LGGM on each domain. DiGress: DiGress trained directly on domain X; LGGM-FT: LGGM pre-trained on all other domains except X and fine-tuned on X under \textbf{Domain Specific Transition Strategy}.}
\label{app-tab-ft-direct-marginal}
\centering
\begin{tabular}{llccccllcccc}
\toprule
\textbf{Domain} & \textbf{Method} & \textbf{DEG} & \textbf{CC} & \textbf{Spec} & \textbf{Orb} & \textbf{Domain} & \textbf{Method} & \textbf{DEG} & \textbf{CC} & \textbf{Spec} & \textbf{Orb} \\
\toprule
\multirow{2}{*}{FB} & DiGress & 0.0423 & 0.0718 & 0.0243 & 0.0298 & \multirow{2}{*}{BIO} & DiGress & 0.0481 & 0.1286 & 0.0487 & 0.0460 \\
& LGGM-FT & \textbf{0.0065} & \textbf{0.0544} & \textbf{0.0069} & \textbf{0.0282} &  & LGGM-FT & \textbf{0.0036} & \textbf{0.0303} & \textbf{0.0102} & \textbf{0.0342} \\
\midrule
\multirow{2}{*}{ASN} & DiGress & 0.0319 & 0.0835 & 0.0679 & 0.1463 & \multirow{2}{*}{ECON}  & DiGress & 0.0224 & 0.0361 & 0.0084 & 0.0325 \\
& LGGM-FT & \textbf{0.0014} & \textbf{0.0509} & \textbf{0.0161} & \textbf{0.0084} &  & LGGM-FT & \textbf{0.0215} & \textbf{0.0330} & \textbf{0.0062} & \textbf{0.0249} \\
\midrule
\multirow{2}{*}{EMAIL} & DiGress & \textbf{0.0145} & 0.0671 & 0.0143 & 0.0558 & \multirow{2}{*}{RT}  & DiGress & 0.0035 & 0.0111 & 0.0094 & 0.0207 \\
& LGGM-FT & 0.0166 & \textbf{0.0364} & \textbf{0.0104} & \textbf{0.0463} &  & LGGM-FT & \textbf{0.0012} & \textbf{0.0075} & \textbf{0.0033} & \textbf{0.0162} \\
\midrule
\multirow{2}{*}{WEB} & DiGress & 0.0204 & 0.0778 & 0.0695 & 0.1101 & \multirow{2}{*}{COL}  & DiGress & 0.0278 & 0.2192 & 0.0669 & 0.0284 \\
& LGGM-FT & \textbf{0.0116} & \textbf{0.0721} & \textbf{0.0152} & \textbf{0.0656} &  & LGGM-FT & \textbf{0.0202} & \textbf{0.1621} & \textbf{0.0571} & \textbf{0.0631} \\
\midrule
\multirow{2}{*}{ROAD} & DiGress & 0.0333 & \textbf{0.1342} & 0.0932 & 0.0861 & \multirow{2}{*}{ECO}  & DiGress & 0.0268 & 0.2356 & 0.0339 & 0.2100 \\
& LGGM-FT & \textbf{0.0088} & 0.1349 & \textbf{0.0347} & \textbf{0.0125} &  & LGGM-FT & \textbf{0.0196} & \textbf{0.2343} & \textbf{0.0291} & \textbf{0.2100} \\
\midrule
\multirow{2}{*}{POWER} & DiGress & 0.0143 & 0.2050 & 0.0776 & 0.0392 & \multirow{2}{*}{CITATION}  & DiGress & 0.0406 & 0.1790 & 0.0677 & 0.0944 \\
& LGGM-FT & \textbf{0.0008} & \textbf{0.1539} & \textbf{0.0215} & \textbf{0.0081} &  & LGGM-FT & \textbf{0.0078} & \textbf{0.0827} & \textbf{0.0137} & \textbf{0.0316} \\
\midrule
\multirow{2}{*}{All} & DiGress & 0.0272 & 0.1208 & 0.0485 & 0.0749\\
& LGGM-FT & \textbf{0.0100} & \textbf{0.0877} & \textbf{0.0187} & \textbf{0.0458}\\
\bottomrule
\end{tabular}
\end{table*}

\subsubsection{Uniform Transition}

\begin{table*}[htbp!]
\scriptsize
\setlength\tabcolsep{5.3pt}
\caption{Comparing Graph Generation Performance between DiGress and Fine-tuned LGGM on each domain. DiGress: DiGress trained directly on domain X; LGGM-FT: LGGM pre-trained on all other domains except X and fine-tuned on X under \textbf{Uniform Transition Strategy}.}
\label{app-tab-ft-direct-uniform}
\centering
\begin{tabular}{llccccllcccc}
\toprule
\textbf{Domain} & \textbf{Method} & \textbf{DEG} & \textbf{CC} & \textbf{Spec} & \textbf{Orb} & \textbf{Domain} & \textbf{Method} & \textbf{DEG} & \textbf{CC} & \textbf{Spec} & \textbf{Orb} \\
\toprule
\multirow{2}{*}{FB} & DiGress & 0.0177 & 0.0698 & 0.0138 & 0.0296 & \multirow{2}{*}{BIO} & DiGress & 0.0179 & 0.0499 & 0.0441 & 0.0526 \\
& LGGM-FT & \textbf{0.0050} & \textbf{0.0579} & \textbf{0.0059} & \textbf{0.0280} &  & LGGM-FT & \textbf{0.0049} & \textbf{0.0496} & \textbf{0.0056} & \textbf{0.0257} \\
\midrule
\multirow{2}{*}{ASN} & DiGress & 0.0337 & 0.1744 & 0.0482 & 0.0243 & \multirow{2}{*}{ECON} & DiGress & \textbf{0.0229} & \textbf{0.0430} & \textbf{0.0088} & \textbf{0.0427} \\
& LGGM-FT & \textbf{0.0058} & \textbf{0.1098} & \textbf{0.0311} & \textbf{0.0101} & & LGGM-FT & 0.0597 & 0.0594 & 0.0216 & 0.0535 \\
\midrule
\multirow{2}{*}{EMAIL} & DiGress & 0.0259 & 0.0901 & 0.0366 & 0.0743 & \multirow{2}{*}{RT} & DiGress & 0.0336 & 0.0920 & 0.0432 & 0.0572 \\
& LGGM-FT & \textbf{0.0120} & \textbf{0.0559} & \textbf{0.0158} & \textbf{0.0444} &  & LGGM-FT & \textbf{0.0032} & \textbf{0.0163} & \textbf{0.0051} & \textbf{0.0227} \\
\midrule
\multirow{2}{*}{WEB} & DiGress & 0.0239 & \textbf{0.0898} & 0.1033 & 0.2371 & \multirow{2}{*}{COL} & DiGress & \textbf{0.0252} & 0.5156 & \textbf{0.1171} & 0.2060 \\
& LGGM-FT & \textbf{0.0218} & 0.1398 & \textbf{0.0310} & \textbf{0.1262} & & LGGM-FT & 0.1074 & \textbf{0.4265} & 0.1398 & \textbf{0.0897} \\
\midrule
\multirow{2}{*}{ROAD} & DiGress & 0.1553 & 0.2788 & 0.2169 & 0.0542 & \multirow{2}{*}{ECO} & DiGress & 0.0263 & 0.2359 & 0.0439 & 0.2100 \\
& LGGM-FT & \textbf{0.0081} & \textbf{0.0547} & \textbf{0.0573} & \textbf{0.0228} & & LGGM-FT & \textbf{0.0204} & \textbf{0.2347} & \textbf{0.0404} & \textbf{0.2100} \\
\midrule
\multirow{2}{*}{POWER} & DiGress & 0.0348 & 0.3174 & 0.1083 & 0.1393 & \multirow{2}{*}{CITATION} & DiGress & 0.0217 & 0.1566 & 0.0645 & 0.1235 \\
& LGGM-FT & \textbf{0.0077} & \textbf{0.0570} & \textbf{0.0134} & \textbf{0.0040} & & LGGM-FT & \textbf{0.0052} & \textbf{0.0821} & \textbf{0.0221} & \textbf{0.0443} \\
\midrule
\multirow{2}{*}{All} & DiGress & 0.0366 & 0.1761 & 0.0707 & 0.1042 \\
& LGGM-FT & \textbf{0.0218} & \textbf{0.1120} & \textbf{0.0324} & \textbf{0.0568} \\
\bottomrule
\end{tabular}
\end{table*}

\newpage
\subsection{Text-to-Graph Generation}
\subsubsection{Domain Specific Transition}
\begin{table*}[htbp!]
\scriptsize
\setlength\tabcolsep{4.5pt}
\caption{Comparing the performance of graph generation between LGGM trained on graphs from all domains with and without domain/name as textual conditions.}
\label{app-tab-text2graph-marginal}
\centering
\begin{tabular}{llcccc|llcccc}
\toprule
\textbf{Domain} & \textbf{Method} & \textbf{DEG} & \textbf{CC} & \textbf{Spec} & \textbf{Orb} & \textbf{Domain} & \textbf{Method} & \textbf{DEG} & \textbf{CC} & \textbf{Spec} & \textbf{Orb} \\
\toprule
\multirow{3}{*}{\textsc{FB}} & LGGM & 0.2566 & 0.3552 & \underline{0.0587} & 0.1614 & \multirow{3}{*}{\textsc{BIO}} & LGGM & 0.2860 & \underline{0.3275} & \underline{0.1117} & \underline{0.2333}\\
& LGGM-T2G$^{\text{D}}$ & \underline{0.1533} & \underline{0.1894} & 0.0817 & \underline{0.0492} & & LGGM-T2G$^{\text{D}}$ & \underline{0.1313} & 0.5111 & 0.1340 & 0.3736\\
 & LGGM-T2G$^{\text{UP}}$ & \textbf{0.0053} & \textbf{0.0576} & \textbf{0.0076} & \textbf{0.0245} & & LGGM-T2G$^{\text{UP}}$  & \textbf{0.0219} & \textbf{0.0251} & \textbf{0.0126} & \textbf{0.0190}\\
 \midrule
\multirow{3}{*}{\textsc{ASN}} & LGGM & 0.1477 & \underline{0.3003} & 0.1551 & 0.3719 & \multirow{3}{*}{\textsc{ECON}} & LGGM & 0.3540 & 0.3404 & 0.2078 & 0.2740\\
& LGGM-T2G$^{\text{D}}$ & \underline{0.0429} & 0.4742 & \underline{0.0949} & \underline{0.0401} & & LGGM-T2G$^{\text{D}}$ & \underline{0.2346} & \underline{0.1572} & \underline{0.1550} & \textbf{0.0579}\\
 & LGGM-T2G$^{\text{UP}}$ & \textbf{0.0161} & \textbf{0.1312} & \textbf{0.0344} & \textbf{0.0174} & & LGGM-T2G$^{\text{UP}}$ & \textbf{0.0869} & \textbf{0.0601} & \textbf{0.0412} & \underline{0.0592} \\
 \midrule
\multirow{3}{*}{\textsc{Email}} & LGGM & 0.1957 & \underline{0.2629} & \underline{0.0646} & \underline{0.2118} & \multirow{3}{*}{\textsc{RT}} & LGGM & 0.4355 & 0.3924 & 0.4329 & 0.4966 \\
& LGGM-T2G$^{\text{D}}$ & \underline{0.0874} & 0.3238 & 0.1472 & 0.2869 & & LGGM-T2G$^{\text{D}}$ & \underline{0.0050} & \underline{0.0940} & \underline{0.0415} & \underline{0.2870}\\
 & LGGM-T2G$^{\text{UP}}$ & \textbf{0.0077} & \textbf{0.0316} & \textbf{0.0176} & \textbf{0.0365} & & LGGM-T2G$^{\text{UP}}$ & \textbf{0.0034} & \textbf{0.0253} & \textbf{0.0225} & \textbf{0.0869}\\
 \midrule
\multirow{3}{*}{\textsc{Web}} & LGGM & 0.2461 & \underline{0.3570} & 0.1853 & 0.4832 & \multirow{3}{*}{\textsc{Col}} & LGGM  &  0.2616 & \textbf{0.3398} & 0.2305 & 0.7090\\
& LGGM-T2G$^{\text{D}}$ & \underline{0.1253} & 0.9088 & \underline{0.1156} & \underline{0.3884} & & LGGM-T2G$^{\text{D}}$ & \underline{0.1301} & 0.9384 & \underline{0.1963} & \underline{0.2032}\\
 & LGGM-T2G$^{\text{UP}}$ & \textbf{0.0771} & \textbf{0.2720} & \textbf{0.0732} & \textbf{0.1251} & & LGGM-T2G$^{\text{UP}}$ & \textbf{0.0845} & \underline{0.5070} & \textbf{0.1378} & \textbf{0.1531}\\
 \midrule
\multirow{3}{*}{\textsc{ROAD}} & LGGM & 0.4315 & 0.8107 & 0.3192 & 0.6976 & \multirow{3}{*}{\textsc{Eco}} & LGGM & 0.4611 & 0.3108 & 0.1932 & 0.3468 \\
& LGGM-T2G$^{\text{D}}$ & \underline{0.0112} & \underline{0.1611} & \textbf{0.0298} & \underline{0.0120} & & LGGM-T2G$^{\text{D}}$ & \textbf{0.0575} & \underline{0.2976} & \underline{0.0585} & \underline{0.2580}\\
 & LGGM-T2G$^{\text{UP}}$ & \textbf{0.0097} & \textbf{0.1316} & \underline{0.0324} & \textbf{0.0119} & & LGGM-T2G$^{\text{UP}}$ & \underline{0.1070} & \textbf{0.2913} & \textbf{0.0410} & \textbf{0.2556}  \\
 \midrule
\multirow{3}{*}{\textsc{Power}} & LGGM & 0.4411 & \textbf{0.4694} & 0.3384 & 1.3222 & \multirow{3}{*}{\textsc{Citation}} & LGGM & 0.3392 & \underline{0.5009} & \underline{0.1295} & \underline{0.2248}\\
& LGGM-T2G$^{\text{D}}$ & \textbf{0.0194} & 0.6031 & \textbf{0.0286} & \textbf{0.0193} & & LGGM-T2G$^{\text{D}}$ & \underline{0.1636} & 0.8868 & 0.2036 & 0.6142\\
 & LGGM-T2G$^{\text{UP}}$ & \underline{0.0227} & \underline{0.4817} & \underline{0.0330} & \underline{0.0223} & & LGGM-T2G$^{\text{UP}}$ & \textbf{0.0496} & \textbf{0.0914} & \textbf{0.0669} & \textbf{0.0318}\\
 \midrule
\multirow{3}{*}{\textsc{All}} & LGGM & 0.3213 & \underline{0.3973} & 0.2022 & 0.4610  \\
 & LGGM-T2G$^{\text{D}}$ & \underline{0.0968} & 0.4621 & \underline{0.1072} & \underline{0.2158}  \\
 & LGGM-T2G$^{\text{UP}}$ & \textbf{0.0410} & \textbf{0.1755} & \textbf{0.0434} & \textbf{0.0703}  \\
 \bottomrule
\end{tabular}
\end{table*}

\subsubsection{Uniform Transition}
\begin{table*}[htbp!]
\scriptsize
\setlength\tabcolsep{4.5pt}
\caption{Comparing the performance of graph generation between LGGM trained on graphs from all domains with and without domain/name as textual conditions.}
\label{app-tab-text2graph-uniform}
\centering
\begin{tabular}{llcccc|llcccc}
\toprule
\textbf{Domain} & \textbf{Method} & \textbf{DEG} & \textbf{CC} & \textbf{Spec} & \textbf{Orb} & \textbf{Domain} & \textbf{Method} & \textbf{DEG} & \textbf{CC} & \textbf{Spec} & \textbf{Orb} \\
\toprule
\multirow{3}{*}{\textsc{FB}} & LGGM & \underline{0.0321} & 0.4994 & 0.0763 & 0.3117 & \multirow{3}{*}{\textsc{BIO}} & LGGM & 0.2661 & 0.3120 & 0.1135 & 0.3835\\
& LGGM-T2G$^{\text{D}}$ & 0.1561 & \underline{0.1639} & \underline{0.0924} & \underline{0.0417} & & LGGM-T2G$^{\text{D}}$ & \underline{0.0099} & \underline{0.1286} & \underline{0.0303} & \underline{0.1366}\\
 & LGGM-T2G$^{\text{UP}}$ & \textbf{0.0050} & \textbf{0.0545} & \textbf{0.0070} & \textbf{0.0251} & & LGGM-T2G$^{\text{UP}}$ & \textbf{0.0028} & \textbf{0.0287} & \textbf{0.0236} & \textbf{0.0174}\\
 \midrule
\multirow{3}{*}{\textsc{ASN}} & LGGM & 0.1511 & 0.4325 & 0.1875 & 0.3896 & \multirow{3}{*}{\textsc{ECON}} & LGGM & 0.3828 & 0.1533 & 0.2039 & 0.2583\\
& LGGM-T2G$^{\text{D}}$ & \underline{0.0318} & \underline{0.2821} & \underline{0.0606} & \underline{0.0631} & & LGGM-T2G$^{\text{D}}$ & \underline{0.0666} & \underline{0.0594} & \underline{0.0650} & \underline{0.0586}\\
 & LGGM-T2G$^{\text{UP}}$ & \textbf{0.0211} & \textbf{0.1191} & \textbf{0.0462} & \textbf{0.0195} & & LGGM-T2G$^{\text{UP}}$ & \textbf{0.0132} & \textbf{0.0257} & \textbf{0.0053} & \textbf{0.0191} \\
 \midrule
\multirow{3}{*}{\textsc{Email}} & LGGM & 0.2156 & 0.2450 & 0.0666 & 0.2757 & \multirow{3}{*}{\textsc{RT}} & LGGM & 0.4395 & 0.2225 & 0.4337 & 0.6641 \\
& LGGM-T2G$^{\text{D}}$ & \underline{0.0469} & \underline{0.0982} & \underline{0.0484} & \underline{0.0505} & & LGGM-T2G$^{\text{D}}$ & \underline{0.0468} & \underline{0.0955} & \underline{0.0729} & \underline{0.0393}\\
 & LGGM-T2G$^{\text{UP}}$ & \textbf{0.0073} & \textbf{0.0379} & \textbf{0.0127} & \textbf{0.0437} & & LGGM-T2G$^{\text{UP}}$ & \textbf{0.0286} & \textbf{0.0933} & \textbf{0.0400} & \textbf{0.0312}\\
 \midrule
\multirow{3}{*}{\textsc{Web}} & LGGM & 0.2725 & 0.2672 & 0.1900 & 0.4368 & \multirow{3}{*}{\textsc{Col}} & LGGM  & 0.3565  & 0.3554 & 0.2451 & 0.7874\\
& LGGM-T2G$^{\text{D}}$ & \underline{0.0255} & \textbf{0.0737} & \underline{0.0354} & \underline{0.1856} & & LGGM-T2G$^{\text{D}}$ & \underline{0.0395} & \underline{0.3110} & \underline{0.1146} & \underline{0.1823}\\
& LGGM-T2G$^{\text{UP}}$ & \textbf{0.0105} & \underline{0.0941} & \textbf{0.0206} & \textbf{0.0451} & & LGGM-T2G$^{\text{UP}}$ & \textbf{0.0265} & \textbf{0.2813} & \textbf{0.0895} & \textbf{0.0899}\\
 \midrule
\multirow{3}{*}{\textsc{ROAD}} & LGGM & 0.4825 & 0.5373 & 0.3398 & 0.7542 & \multirow{3}{*}{\textsc{Eco}} & LGGM & 0.5466 & 0.6003 & 0.2257 & 0.7089 \\
& LGGM-T2G$^{\text{D}}$ & \textbf{0.0088} & \underline{0.1225} & \underline{0.0399} & \underline{0.0155} & & LGGM-T2G$^{\text{D}}$ & \underline{0.2160} & \underline{0.2917} & \underline{0.1203} & \underline{0.2569}\\
 & LGGM-T2G$^{\text{UP}}$ & \underline{0.0177} & \textbf{0.0437} & \textbf{0.0336} & \textbf{0.0086} & & LGGM-T2G$^{\text{UP}}$ &  \textbf{0.0293} & \textbf{0.2885} & \textbf{0.0416} & \textbf{0.2556}  \\
 \midrule
\multirow{3}{*}{\textsc{Power}} & LGGM & 0.4394 & 0.4646 & 0.3473 & 1.3186 & \multirow{3}{*}{\textsc{Citation}} & LGGM & 0.2624 & 0.5374 & 0.1295 & 0.3419\\
& LGGM-T2G$^{\text{D}}$ & \underline{0.0162} & \underline{0.1131} & \underline{0.0479} & \underline{0.1786} & & LGGM-T2G$^{\text{D}}$ & \underline{0.0101} & \underline{0.1025} & \underline{0.0315} & \underline{0.0651}\\
& LGGM-T2G$^{\text{UP}}$ & \textbf{0.0062} & \textbf{0.0570} & \textbf{0.0111} & \textbf{0.0084} & & LGGM-T2G$^{\text{UP}}$ & \textbf{0.0072} & \textbf{0.0849} & \textbf{0.0115} & \textbf{0.0287}\\
 \midrule
\multirow{3}{*}{\textsc{All}} & LGGM & 0.3206 & 0.3856 & 0.2132 & 0.5526  \\
 & LGGM-T2G$^{\text{D}}$ & \underline{0.0562} & \underline{0.1535} & \underline{0.0633} & \underline{0.1061}  \\
 & LGGM-T2G$^{\text{UP}}$ & \textbf{0.0146} & \textbf{0.1007} & \textbf{0.0286} & \textbf{0.0494}  \\
 \bottomrule
\end{tabular}
\end{table*}

\newpage
\subsection{Sensitive Analysis on Number of Training Data under Domain Specific Transition}\label{app-expr-training-data}

\begin{figure*}[htbp!]
     \centering
     \begin{subfigure}[b]{0.24\textwidth}
         \centering
         \includegraphics[width=\textwidth]{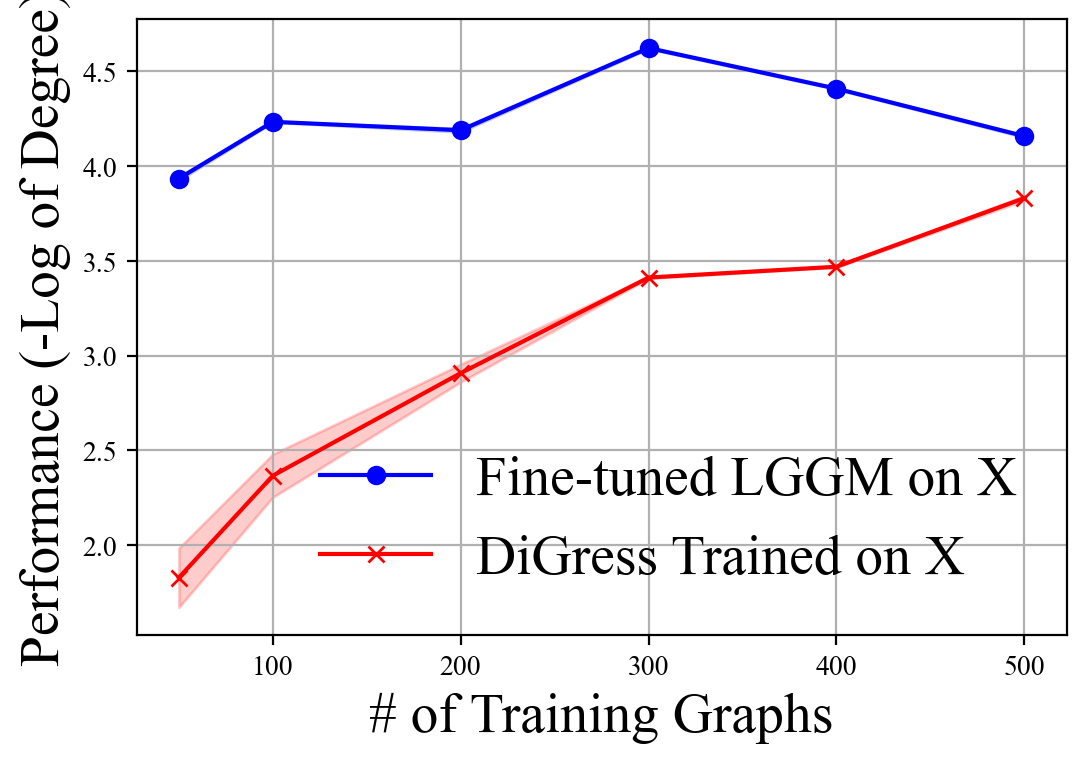}
         \caption{Road-DEG}
         \label{fig-ft-varying-number-road-deg-marginal}
     \end{subfigure}
     \hfill
     \begin{subfigure}[b]{0.24\textwidth}
         \centering
         \includegraphics[width=\textwidth]{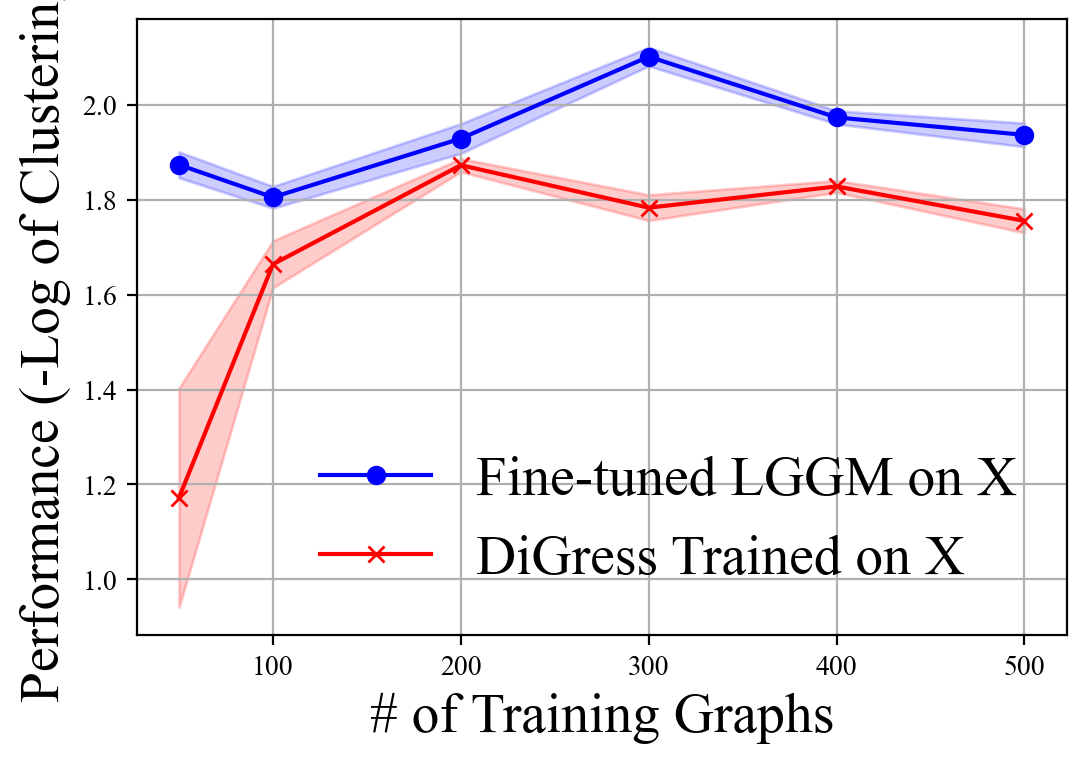}
         \caption{Road-CC}
         \label{fig-ft-varying-number-road-cc-marginal}
     \end{subfigure}
     \hfill
     \begin{subfigure}[b]{0.24\textwidth}
         \centering
         \includegraphics[width=\textwidth]{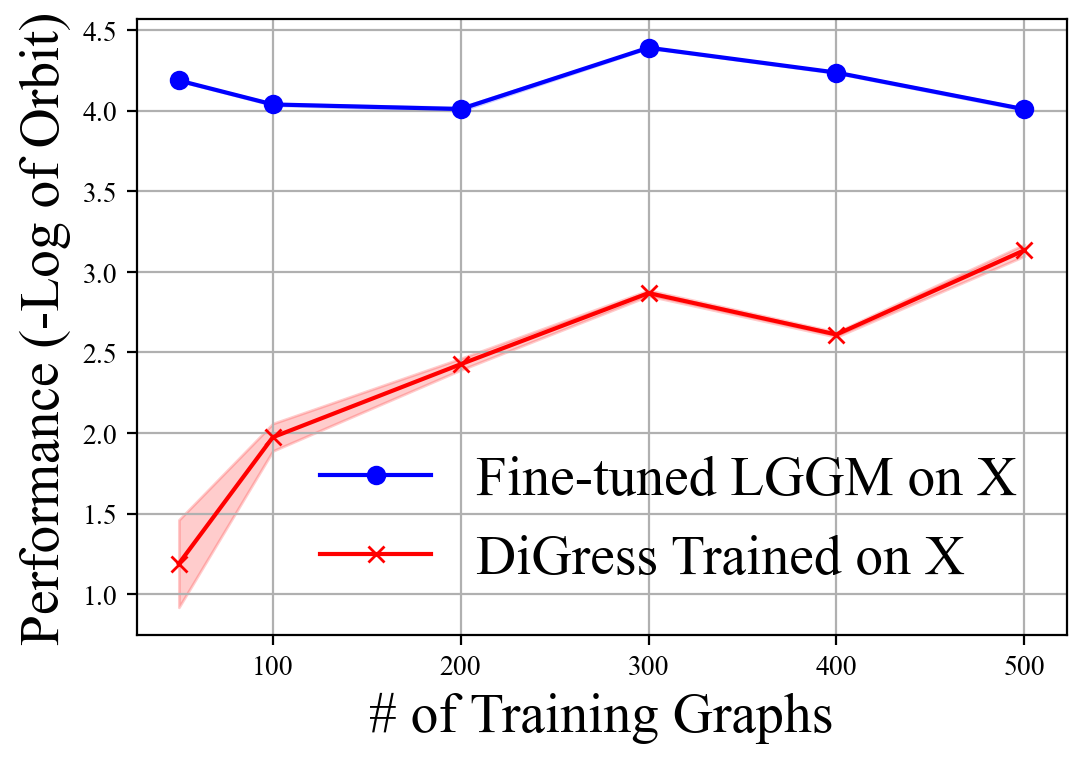}
         \caption{Road-Orb}
         \label{fig-ft-varying-number-road-orb-marginal}
     \end{subfigure}
     \hfill
     \begin{subfigure}[b]{0.24\textwidth}
         \centering
         \includegraphics[width=\textwidth]{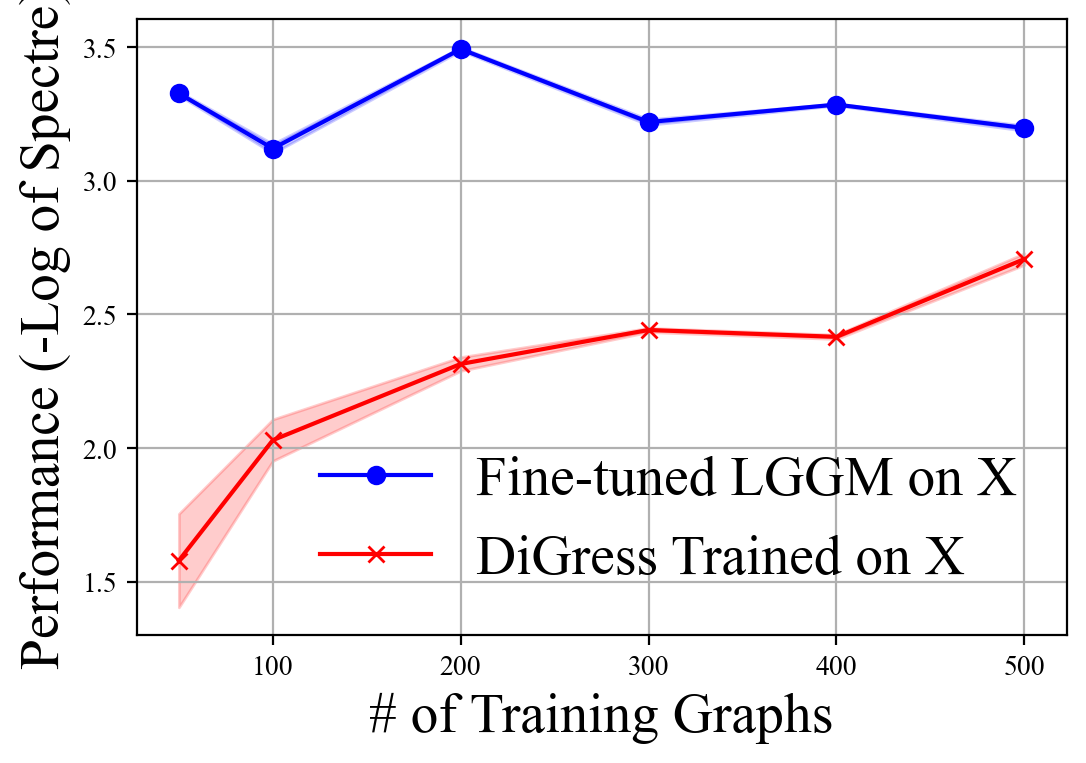}
         \caption{Road-Spec}
         \label{fig-ft-varying-number-road-spec-marginal}
     \end{subfigure}
    \caption{Effect of Number of Training Graphs on Road Networks.}
    \label{fig-ft-varyingnumber-road}
\end{figure*}

\begin{figure*}[htbp!]
     \centering
     \begin{subfigure}[b]{0.24\textwidth}
         \centering
         \includegraphics[width=\textwidth]{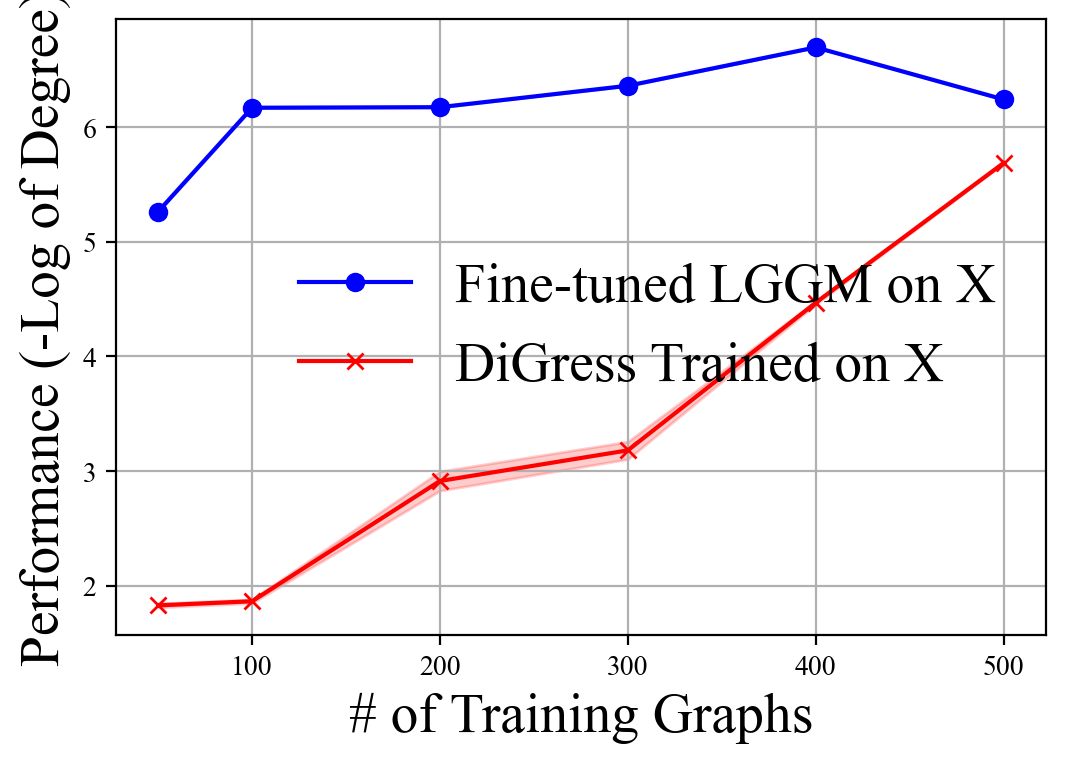}
         \caption{Retweet-DEG}
         \label{fig-ft-varying-number-rt-deg-marginal}
     \end{subfigure}
     \hfill
     \begin{subfigure}[b]{0.24\textwidth}
         \centering
         \includegraphics[width=\textwidth]{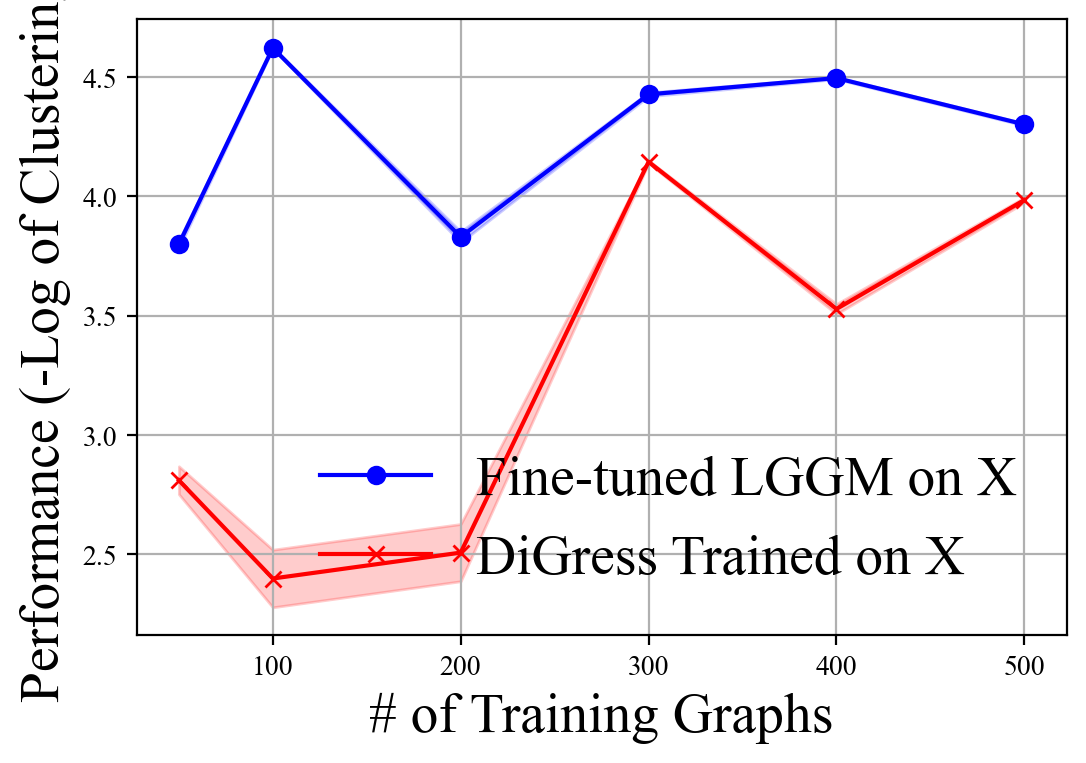}
         \caption{Retweet-CC}
         \label{fig-ft-varying-number-rt-cc-marginal}
     \end{subfigure}
     \hfill
     \begin{subfigure}[b]{0.24\textwidth}
         \centering
         \includegraphics[width=\textwidth]{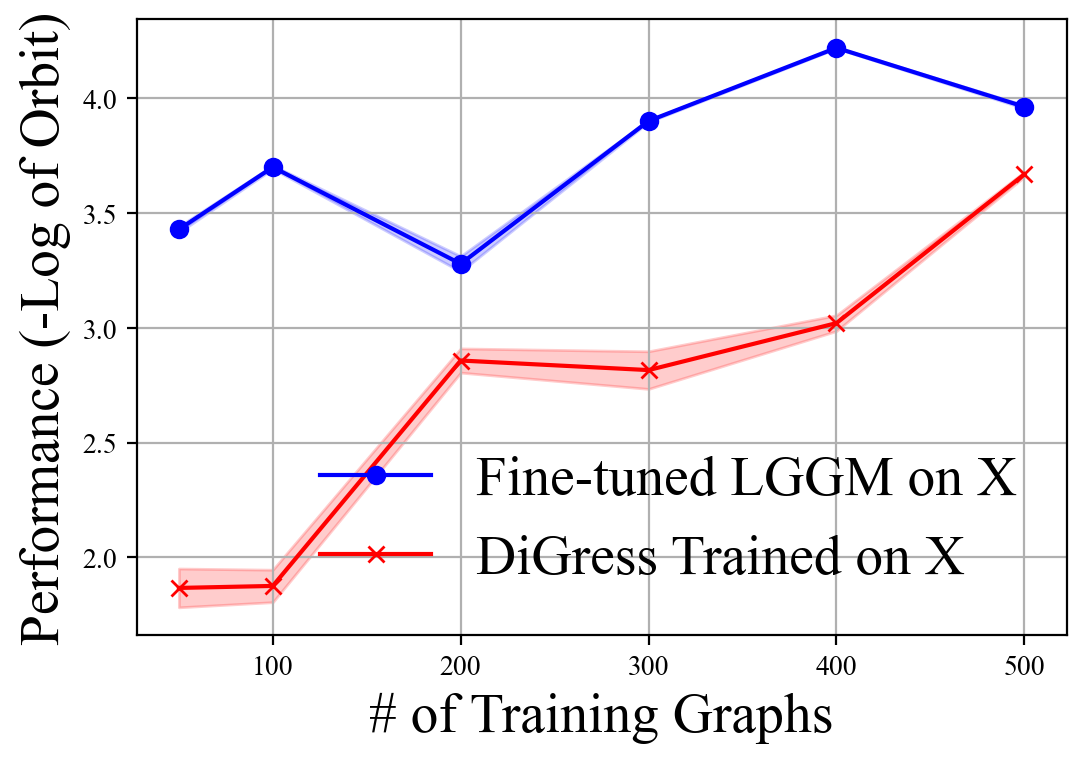}
         \caption{Retweet-Orb}
         \label{fig-ft-varying-number-rt-orb-marginal}
     \end{subfigure}
     \hfill
     \begin{subfigure}[b]{0.24\textwidth}
         \centering
         \includegraphics[width=\textwidth]{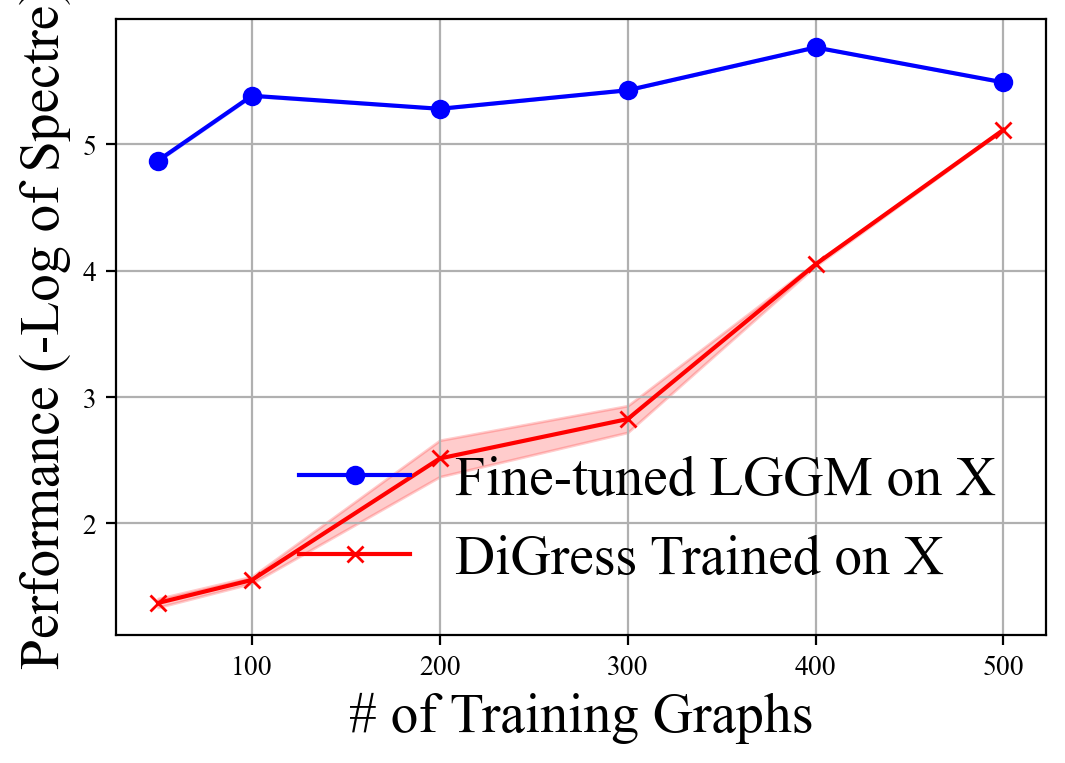}
         \caption{Retweet-Spec}
         \label{fig-ft-varying-number-rt-spec-marginal}
     \end{subfigure}
    \caption{Effect of Number of Training Graphs on Retweet Networks.}
    \label{fig-ft-varyingnumber-rt}
\end{figure*}

\begin{figure*}[htbp!]
     \centering
     \begin{subfigure}[b]{0.24\textwidth}
         \centering
         \includegraphics[width=\textwidth]{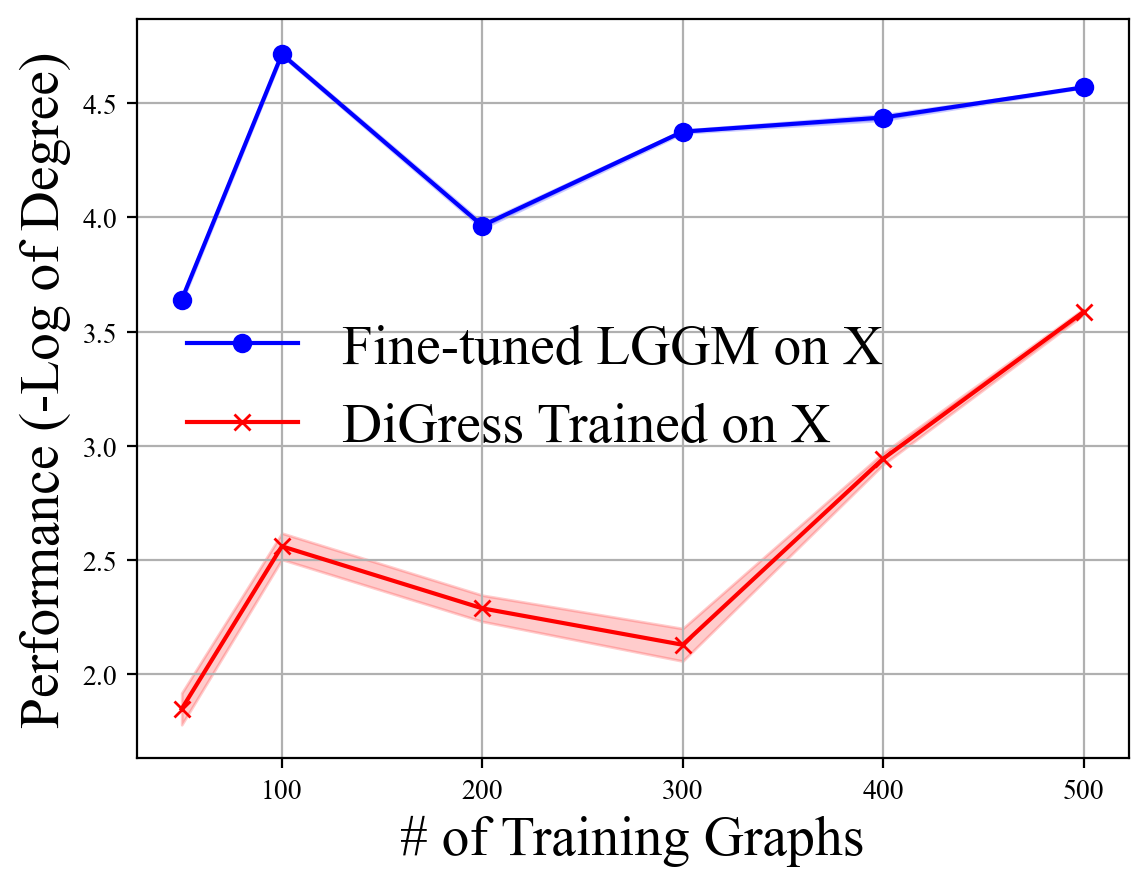}
         \caption{Email-DEG}
         \label{fig-ft-varying-number-email-deg-marginal}
     \end{subfigure}
     \hfill
     \begin{subfigure}[b]{0.24\textwidth}
         \centering
         \includegraphics[width=\textwidth]{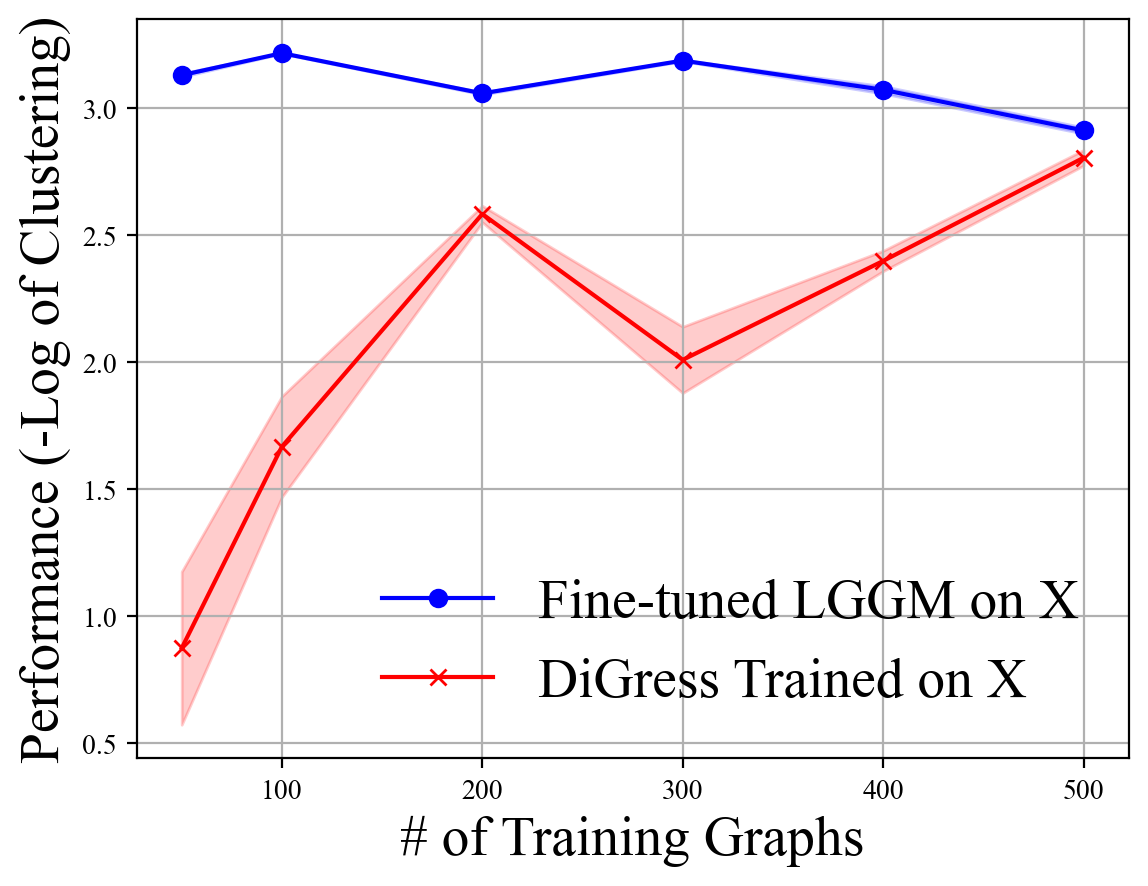}
         \caption{Email-CC}
         \label{fig-ft-varying-number-email-cc-marginal}
     \end{subfigure}
     \hfill
     \begin{subfigure}[b]{0.24\textwidth}
         \centering
         \includegraphics[width=\textwidth]{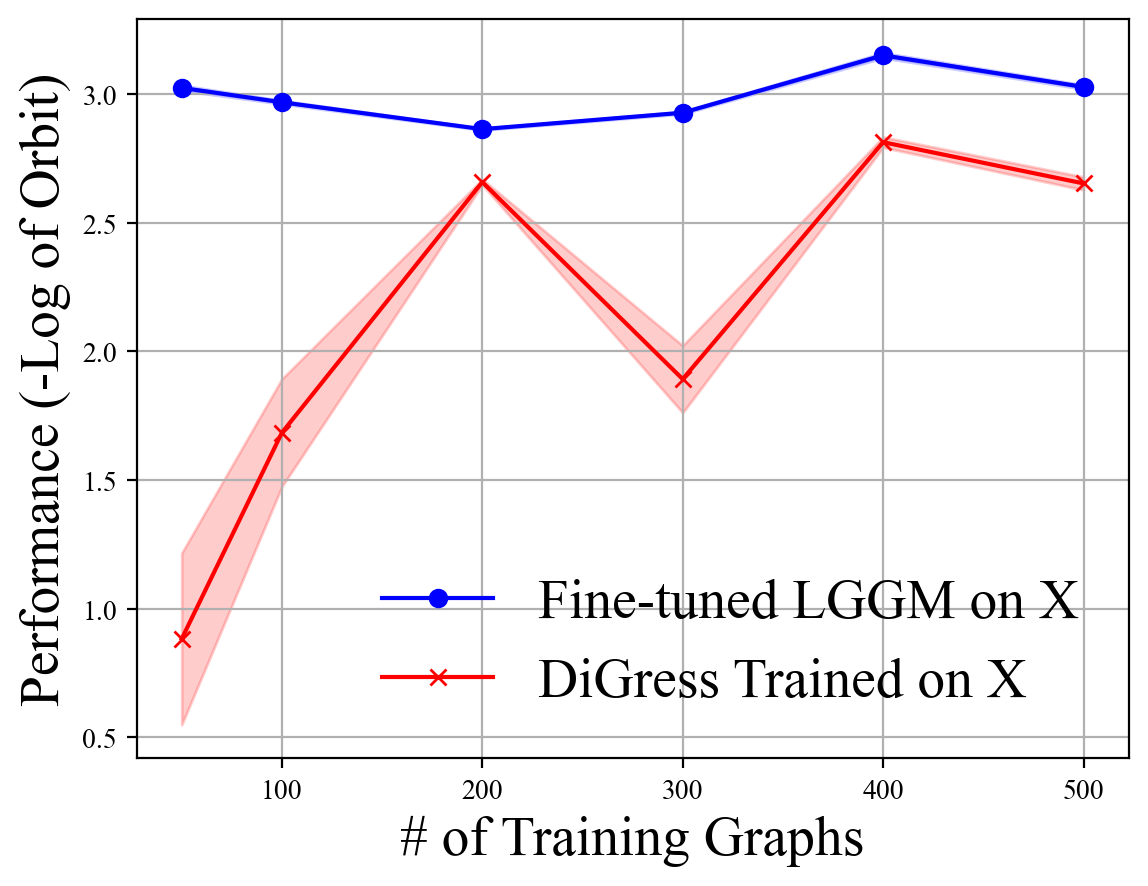}
         \caption{Email-Orb}
         \label{fig-ft-varying-number-email-orb-marginal}
     \end{subfigure}
     \hfill
     \begin{subfigure}[b]{0.24\textwidth}
         \centering
         \includegraphics[width=\textwidth]{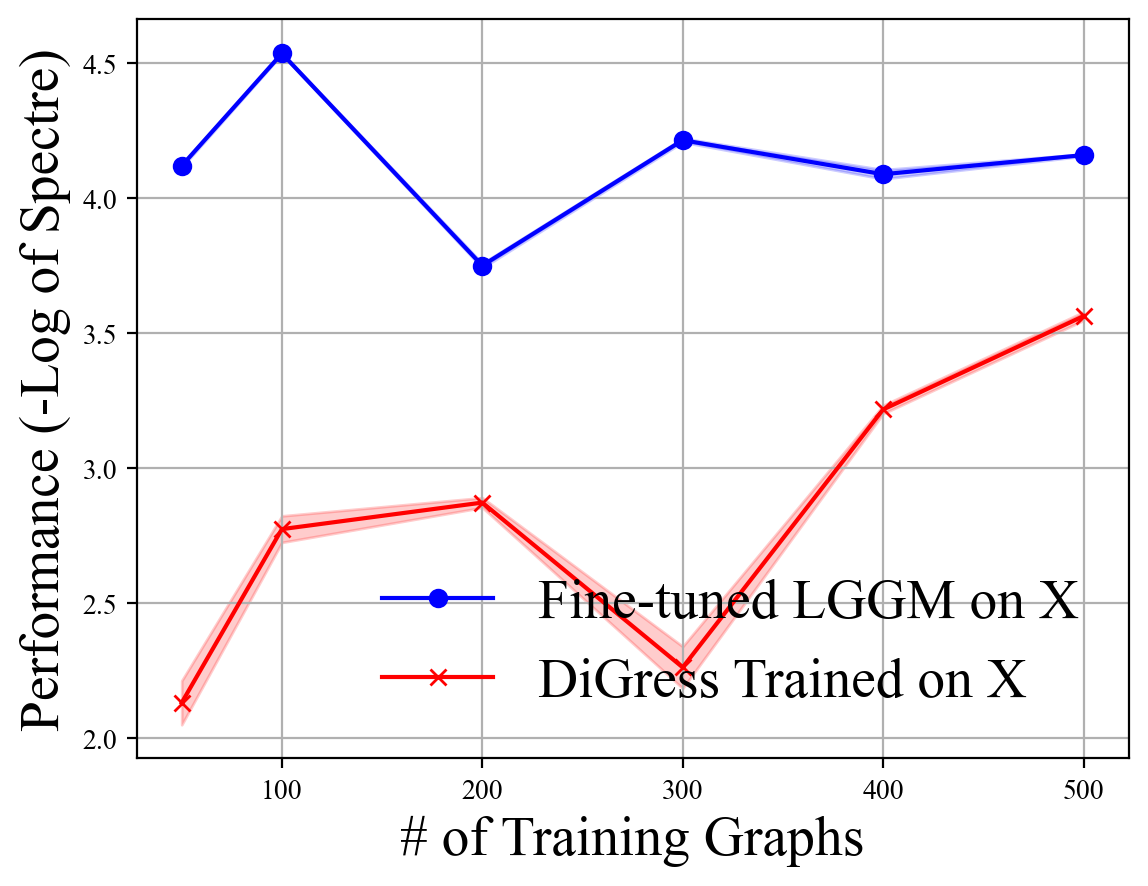}
         \caption{Email-Spec}
         \label{fig-ft-varying-number-email-spec-marginal}
     \end{subfigure}
    \caption{Effect of Number of Training Graphs on Email Networks.}
    \label{fig-ft-varyingnumber-email-marginal}
\end{figure*}

\begin{figure*}[htbp!]
     \centering
     \begin{subfigure}[b]{0.24\textwidth}
         \centering
         \includegraphics[width=\textwidth]{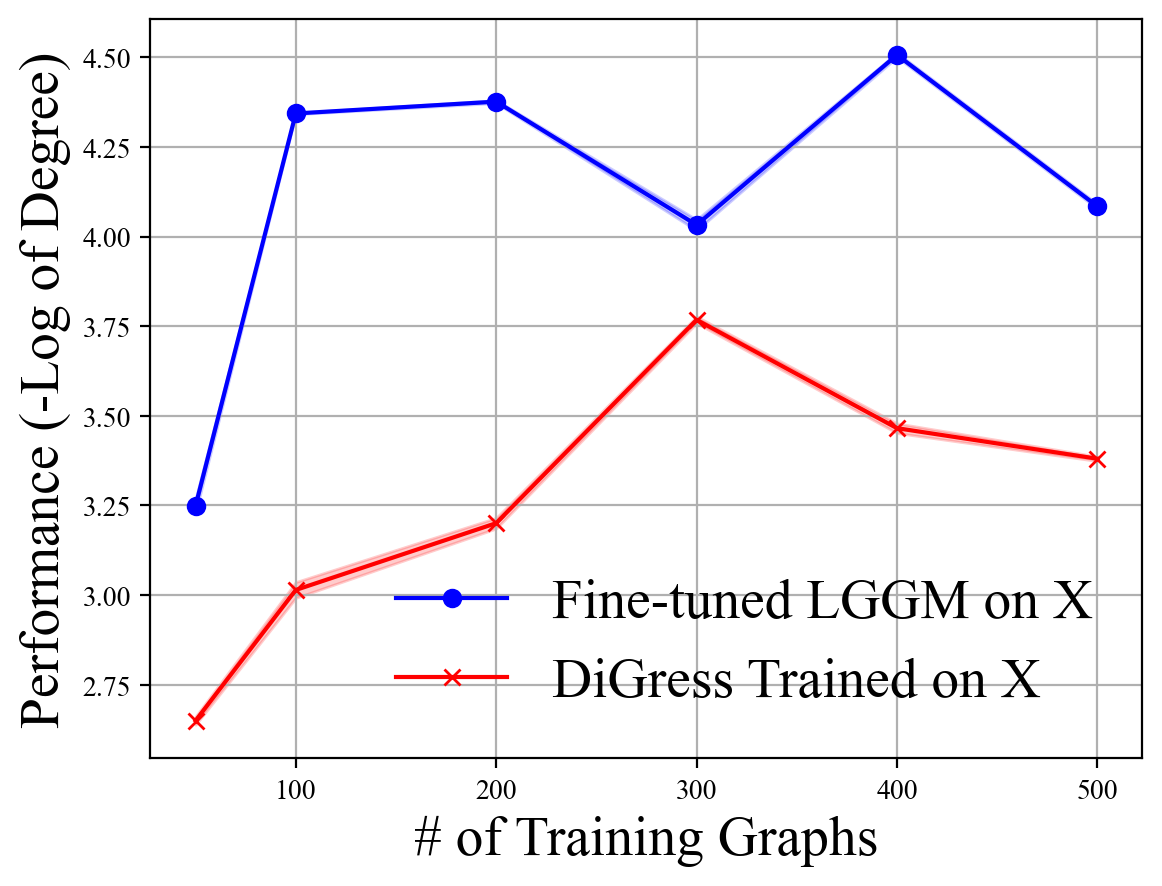}
         \caption{Web-DEG}
         \label{fig-ft-varying-number-web-deg-marginal}
     \end{subfigure}
     \hfill
     \begin{subfigure}[b]{0.24\textwidth}
         \centering
         \includegraphics[width=\textwidth]{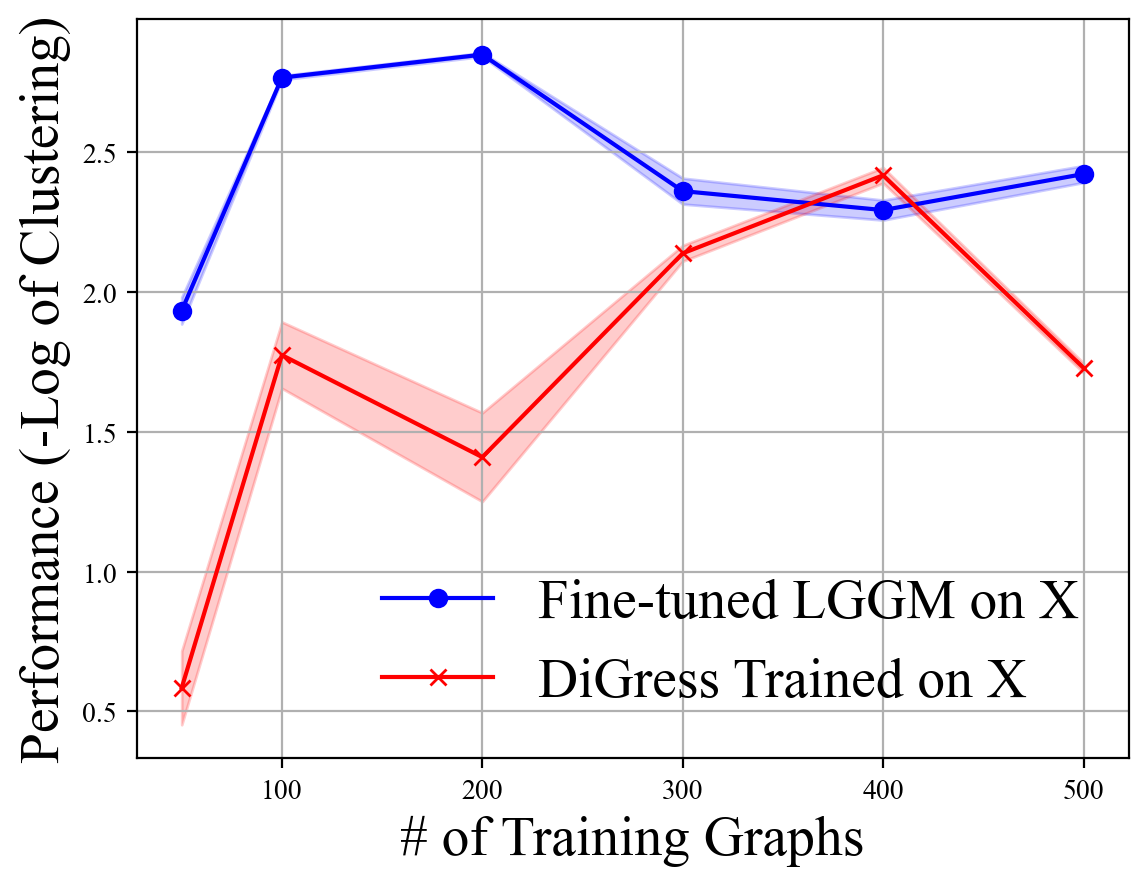}
         \caption{Web-CC}
         \label{fig-ft-varying-number-web-cc-marginal}
     \end{subfigure}
     \hfill
     \begin{subfigure}[b]{0.24\textwidth}
         \centering
         \includegraphics[width=\textwidth]{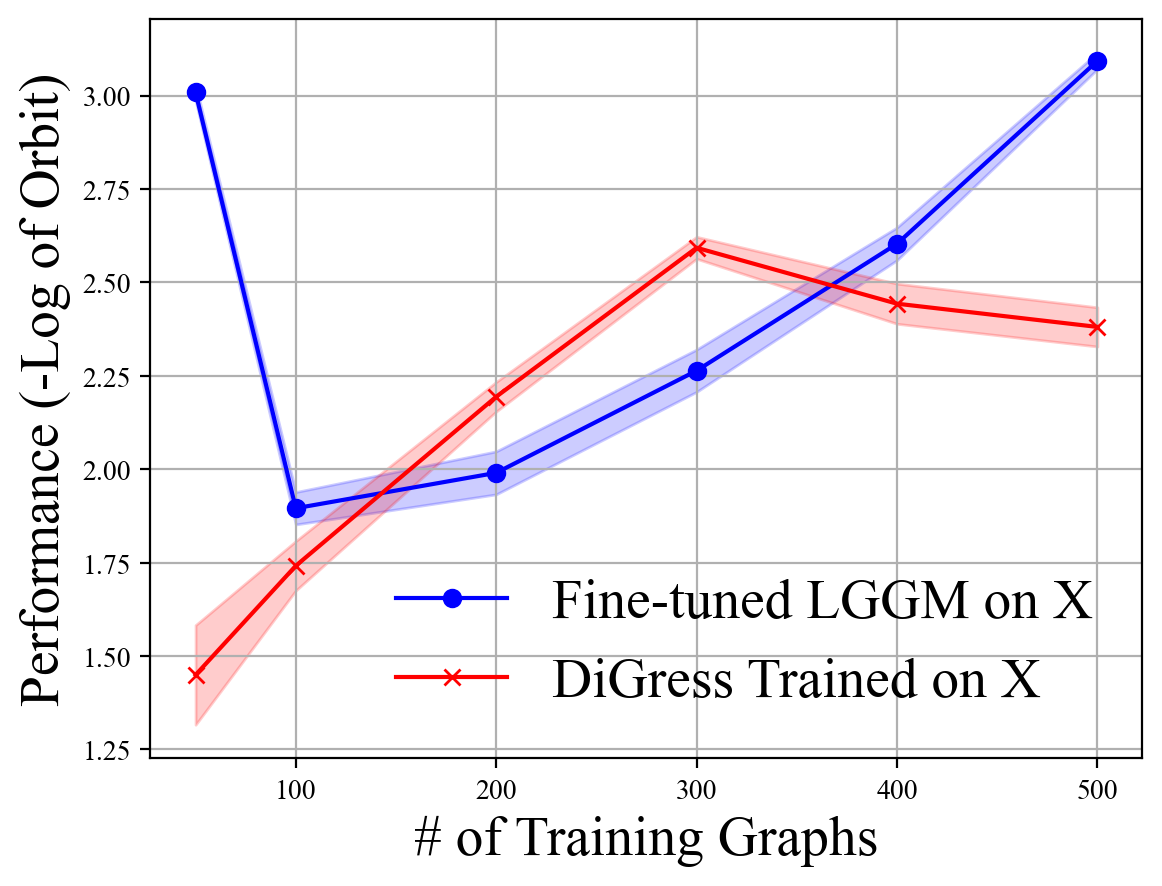}
         \caption{Web-Orb}
         \label{fig-ft-varying-number-web-orb-marginal}
     \end{subfigure}
     \hfill
     \begin{subfigure}[b]{0.24\textwidth}
         \centering
         \includegraphics[width=\textwidth]{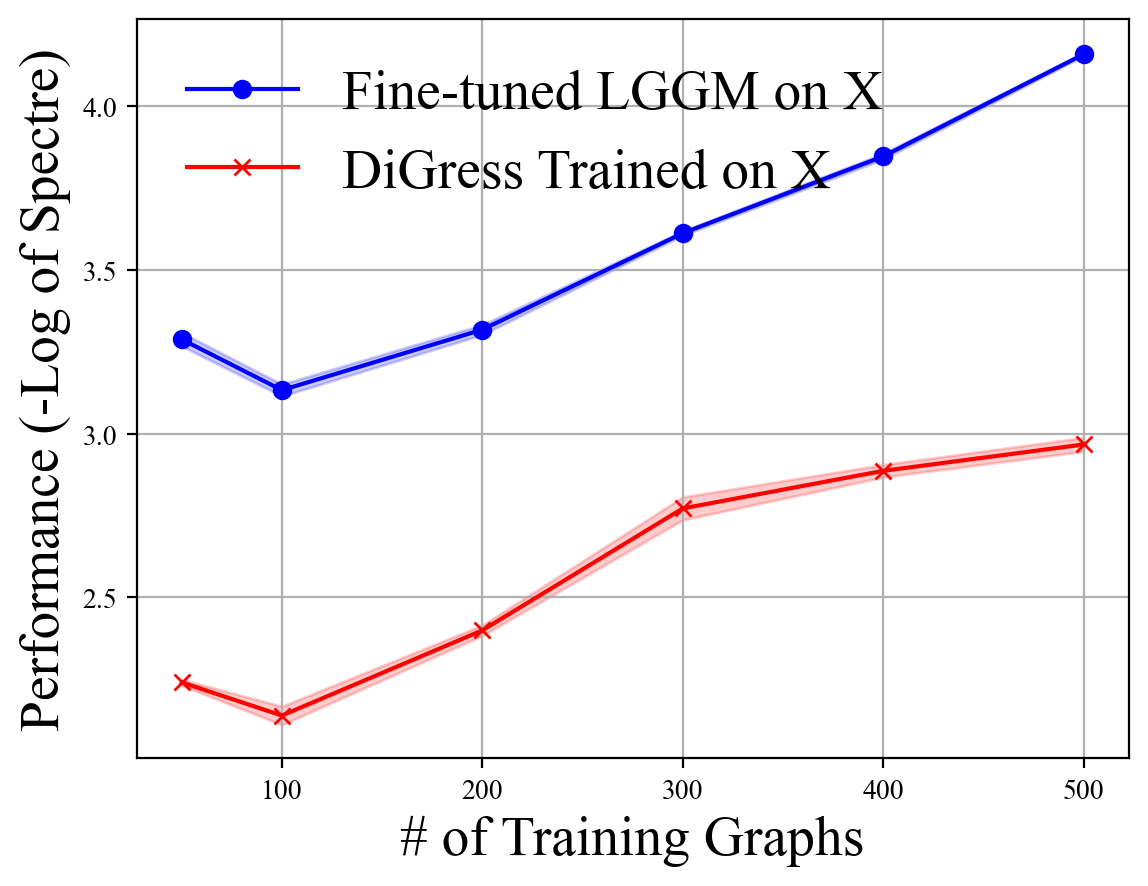}
         \caption{Web-Spec}
         \label{fig-ft-varying-number-web-spec-marginal}
     \end{subfigure}
    \caption{Effect of Number of Training Graphs on Web Graphs.}
    \label{fig-ft-varyingnumber-web-marginal}
\end{figure*}

\begin{figure*}[htbp!]
     \centering
     \begin{subfigure}[b]{0.24\textwidth}
         \centering
         \includegraphics[width=\textwidth]{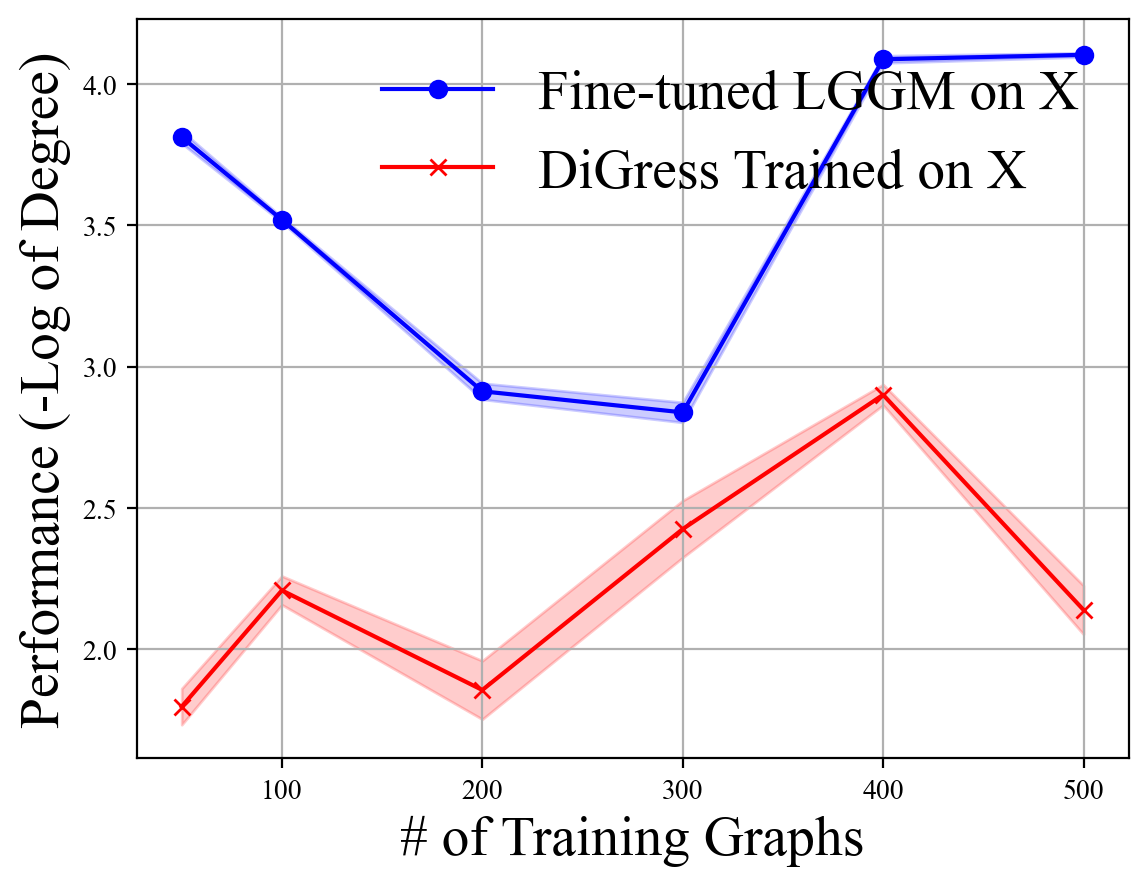}
         \caption{Facebook-DEG}
         \label{fig-ft-varying-number-fb-deg-marginal}
     \end{subfigure}
     \hfill
     \begin{subfigure}[b]{0.24\textwidth}
         \centering
         \includegraphics[width=\textwidth]{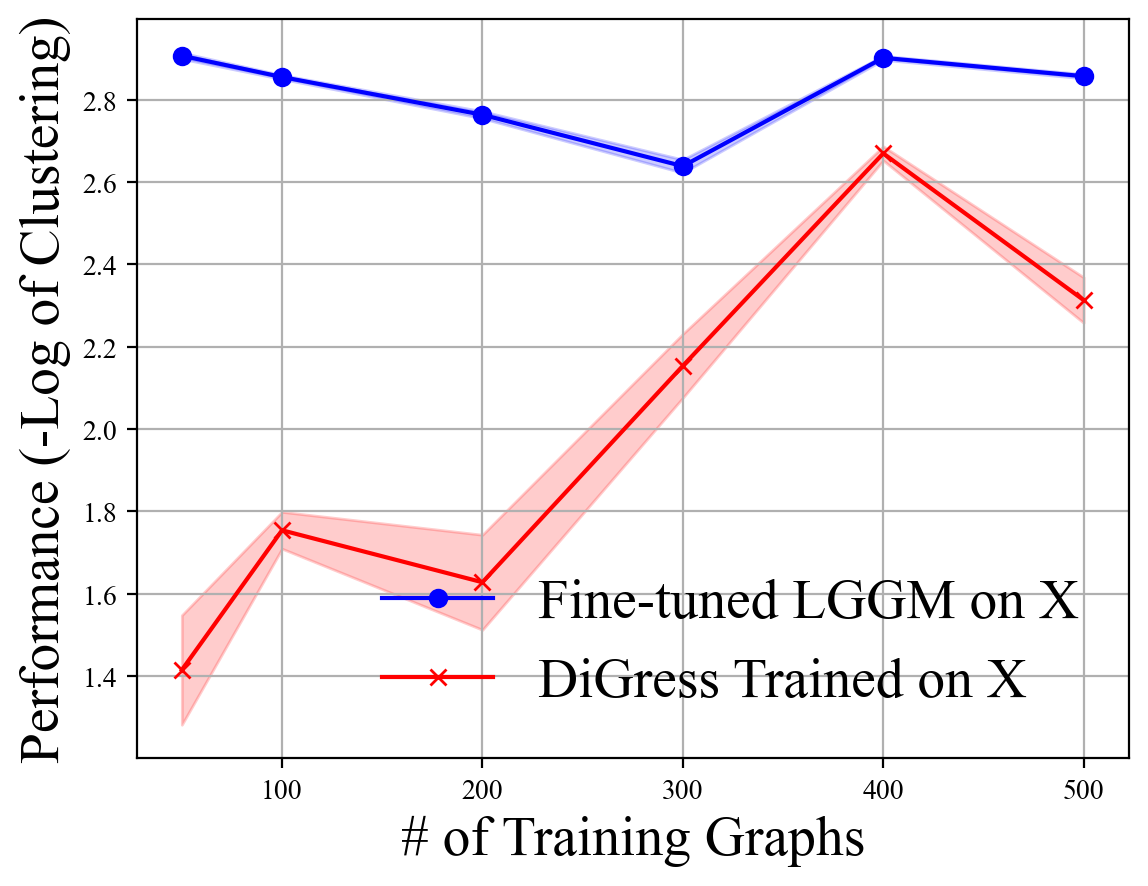}
         \caption{Facebook-CC}
         \label{fig-ft-varying-number-fb-cc-marginal}
     \end{subfigure}
     \hfill
     \begin{subfigure}[b]{0.24\textwidth}
         \centering
         \includegraphics[width=\textwidth]{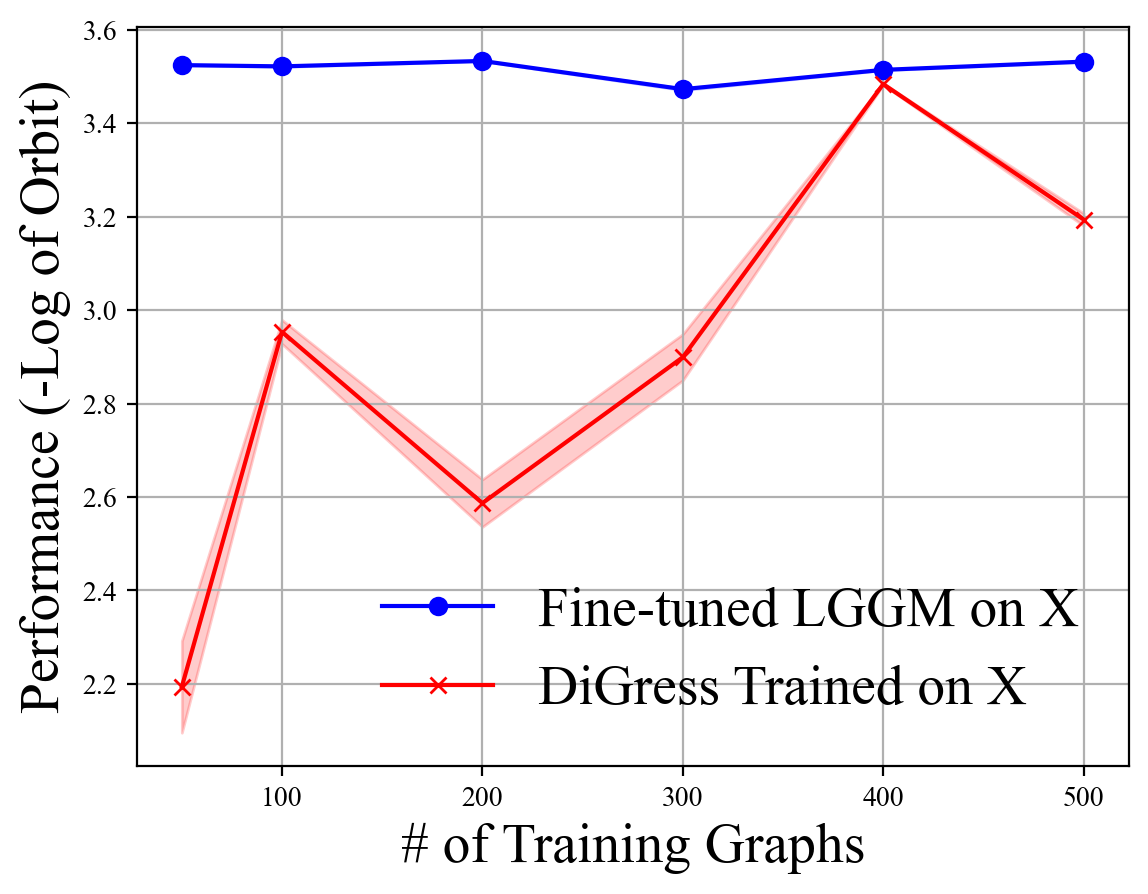}
         \caption{Facebook-Orb}
         \label{fig-ft-varying-number-fb-orb-marginal}
     \end{subfigure}
     \hfill
     \begin{subfigure}[b]{0.24\textwidth}
         \centering
         \includegraphics[width=\textwidth]{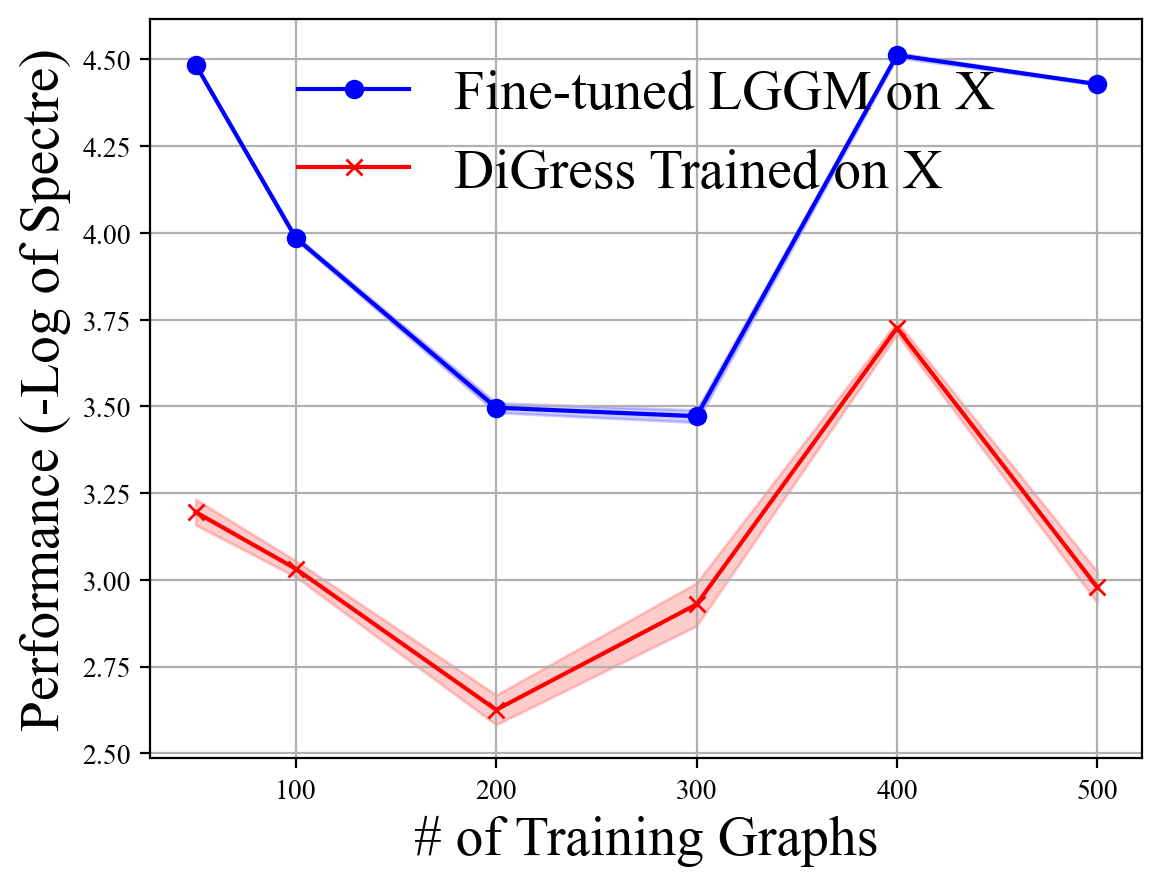}
         \caption{Facebook-Spec}
         \label{fig-ft-varying-number-fb-spec-marginal}
     \end{subfigure}
    \caption{Effect of Number of Training Graphs on Facebook Networks.}
    \label{fig-ft-varyingnumber-fb-marginal}
\end{figure*}

\newpage
\subsection{Sensitive Analysis on Number of Training Data under Uniform Transition Strategy}\label{app-expr-training-data-uniform}

\begin{figure*}[htbp!]
     \centering
     \begin{subfigure}[b]{0.24\textwidth}
         \centering
         \includegraphics[width=\textwidth]{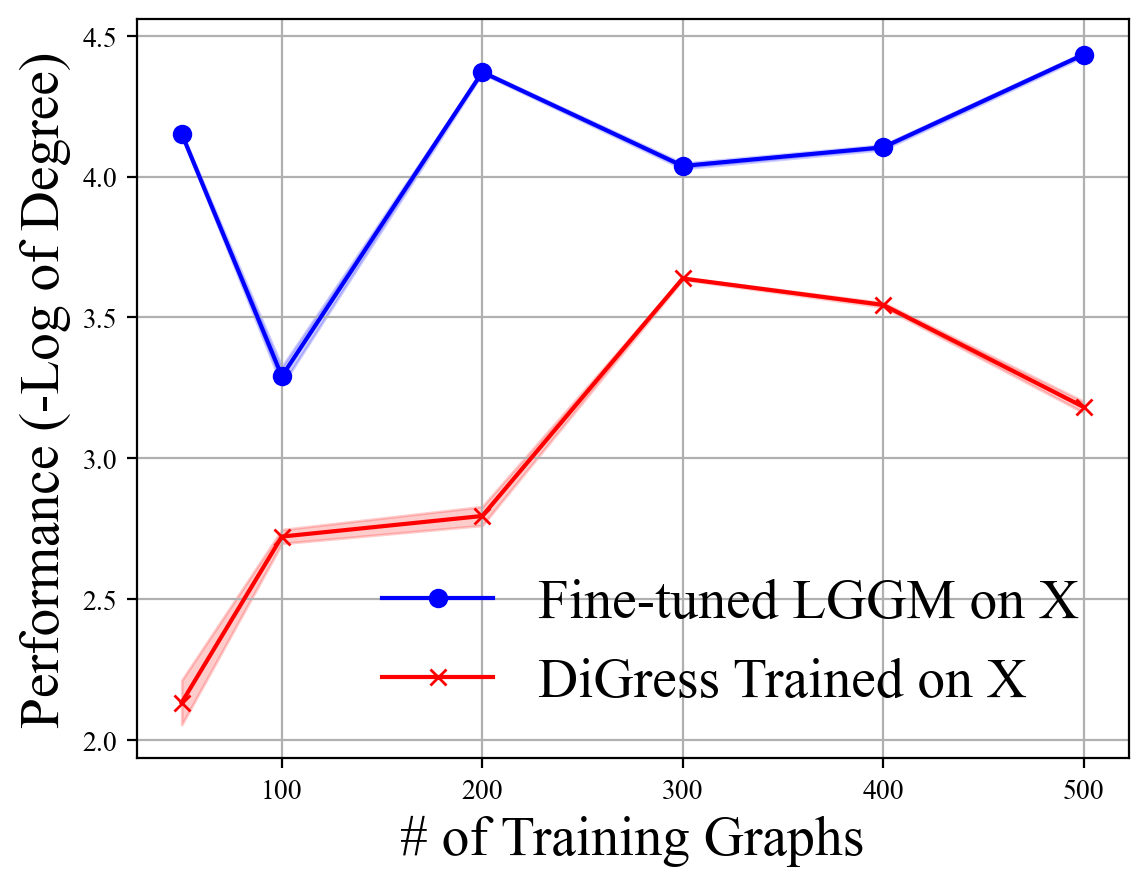}
         \caption{Citation-DEG}
         \label{fig-ft-varying-number-citation-deg-uniform}
     \end{subfigure}
     \hfill
     \begin{subfigure}[b]{0.24\textwidth}
         \centering
         \includegraphics[width=\textwidth]{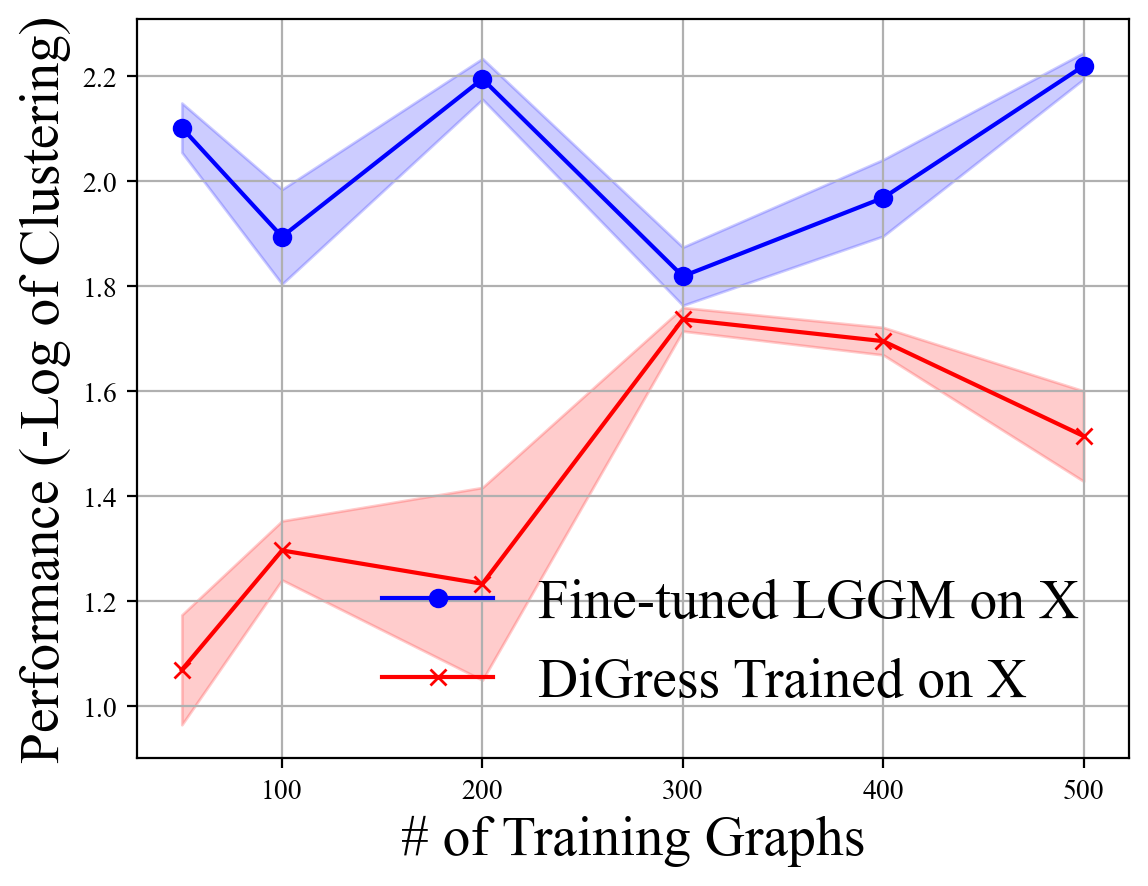}
         \caption{Citation-CC}
         \label{fig-ft-varying-number-citation-cc-uniform}
     \end{subfigure}
     \hfill
     \begin{subfigure}[b]{0.24\textwidth}
         \centering
         \includegraphics[width=\textwidth]{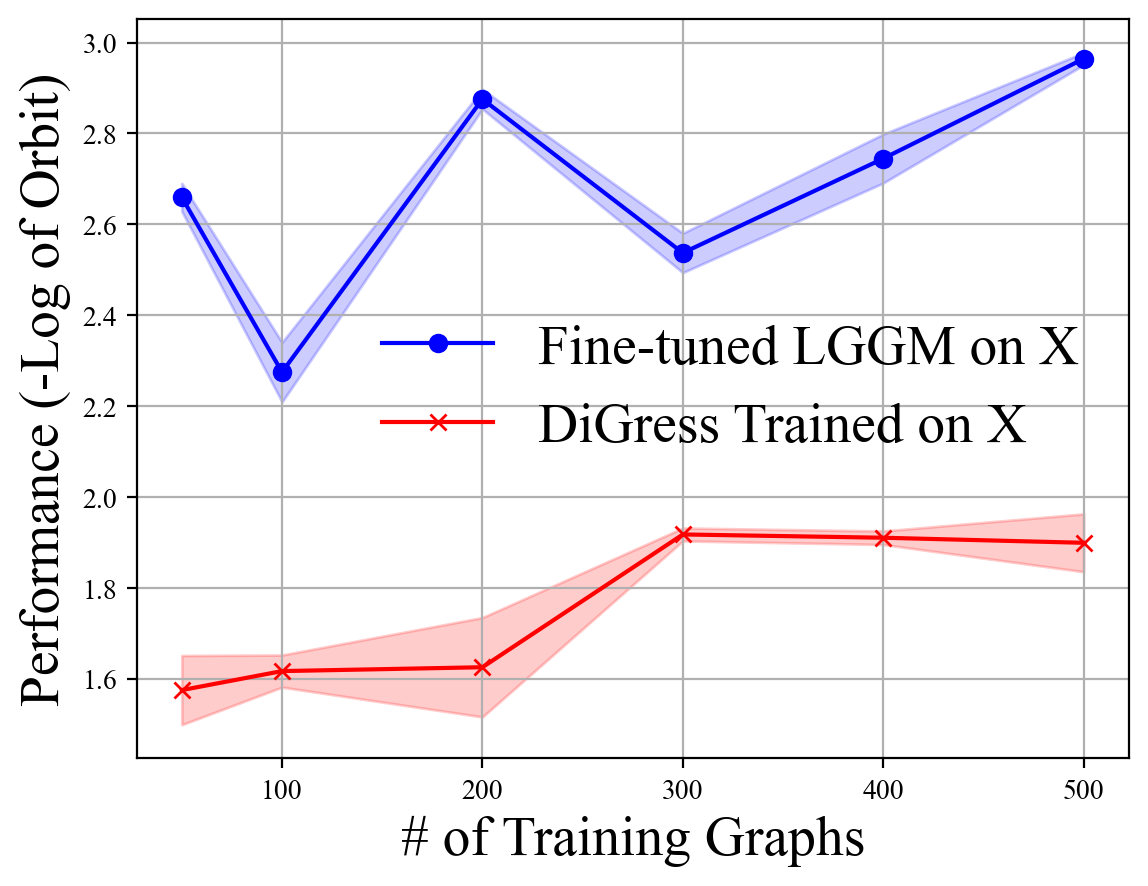}
         \caption{Citation-Orb}
         \label{fig-ft-varying-number-citation-orb-uniform}
     \end{subfigure}
     \hfill
     \begin{subfigure}[b]{0.24\textwidth}
         \centering
         \includegraphics[width=\textwidth]{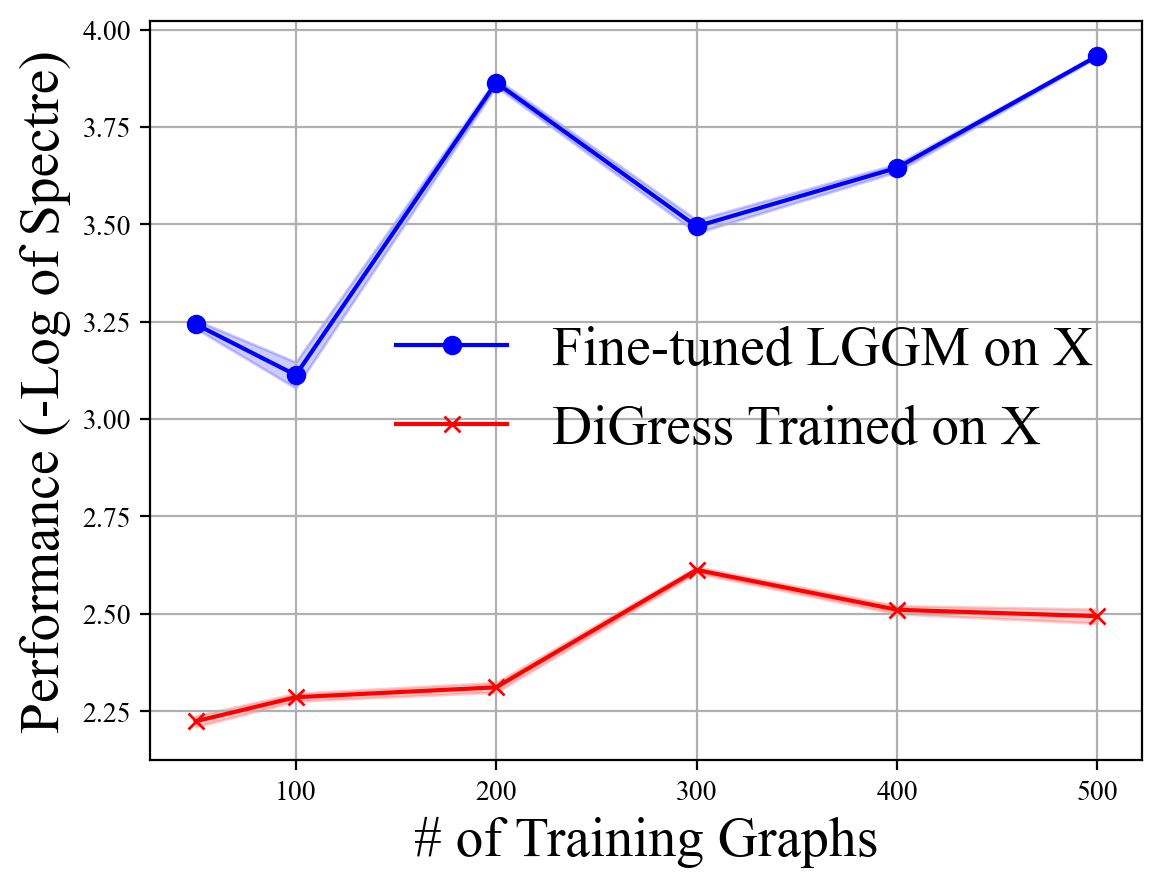}
         \caption{Citation-Spec}
         \label{fig-ft-varying-number-citation-spec-uniform}
     \end{subfigure}
    \caption{Effect of Number of Training Graphs on Citation Networks.}
    \label{fig-ft-varyingnumber-citation-uniform}
\end{figure*}

\begin{figure*}[htbp!]
     \centering
     \begin{subfigure}[b]{0.24\textwidth}
         \centering
         \includegraphics[width=\textwidth]{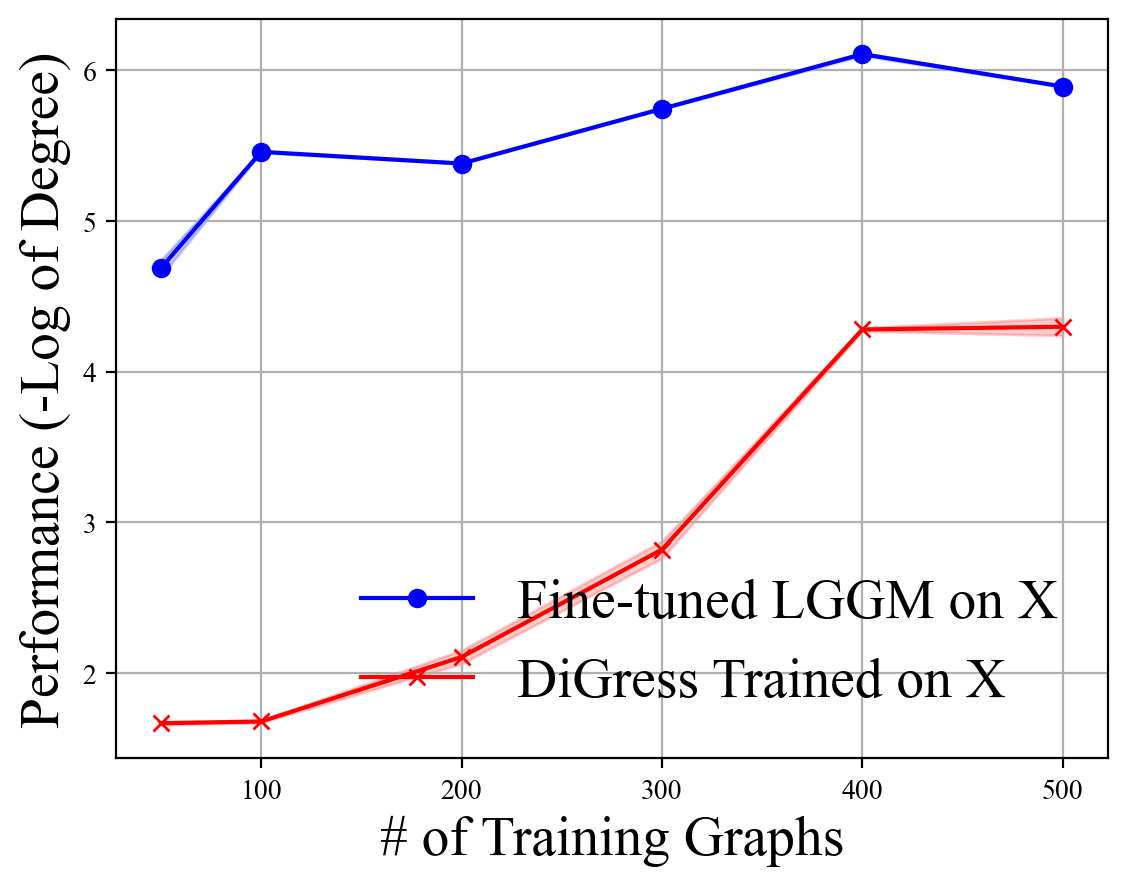}
         \caption{Retweet-DEG}
         \label{fig-ft-varying-number-rt-deg-uniform}
     \end{subfigure}
     \hfill
     \begin{subfigure}[b]{0.24\textwidth}
         \centering
         \includegraphics[width=\textwidth]{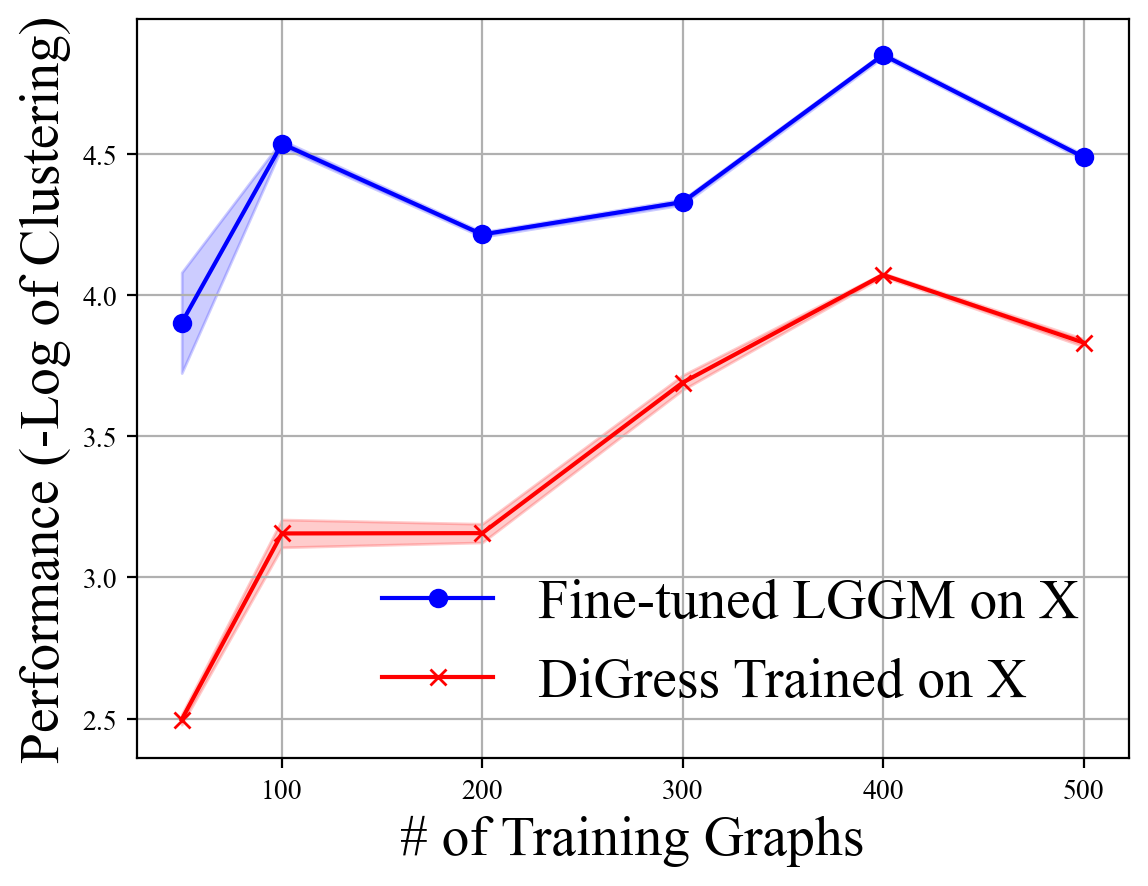}
         \caption{Retweet-CC}
         \label{fig-ft-varying-number-rt-cc-uniform}
     \end{subfigure}
     \hfill
     \begin{subfigure}[b]{0.24\textwidth}
         \centering
         \includegraphics[width=\textwidth]{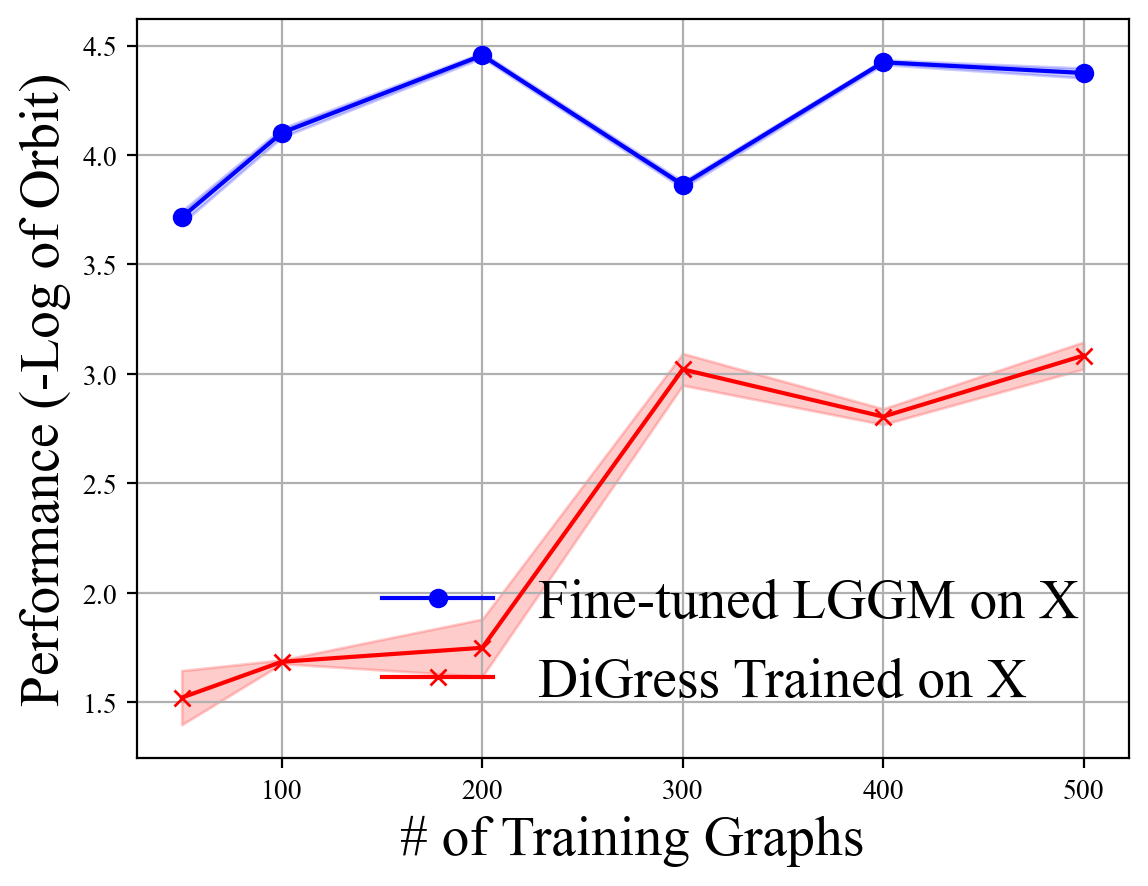}
         \caption{Retweet-Orb}
         \label{fig-ft-varying-number-rt-orb-uniform}
     \end{subfigure}
     \hfill
     \begin{subfigure}[b]{0.24\textwidth}
         \centering
         \includegraphics[width=\textwidth]{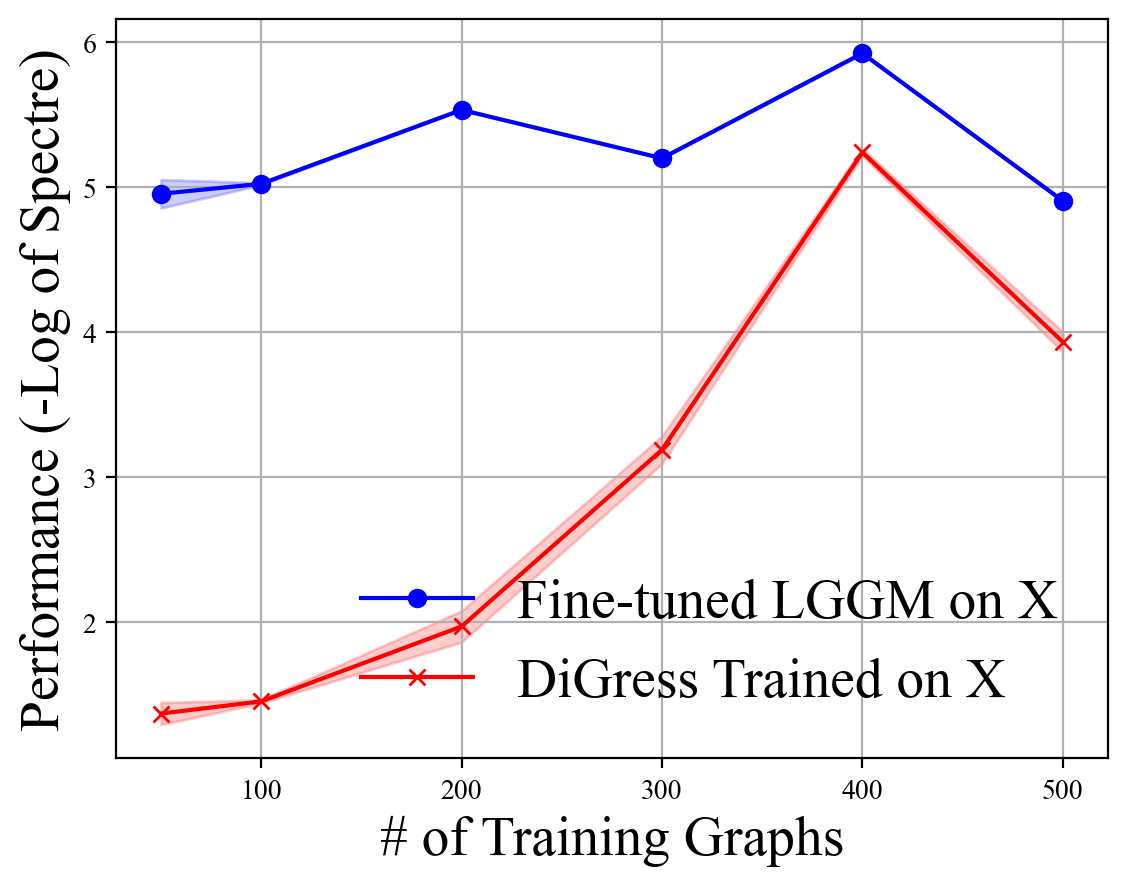}
         \caption{Retweet-Spec}
         \label{fig-ft-varying-number-rt-spec-uniform}
     \end{subfigure}
    \caption{Effect of Number of Training Graphs on Retweet Networks.}
    \label{fig-ft-varyingnumber-rt-uniform}
\end{figure*}

\begin{figure*}[htbp!]
     \centering
     \begin{subfigure}[b]{0.24\textwidth}
         \centering
         \includegraphics[width=\textwidth]{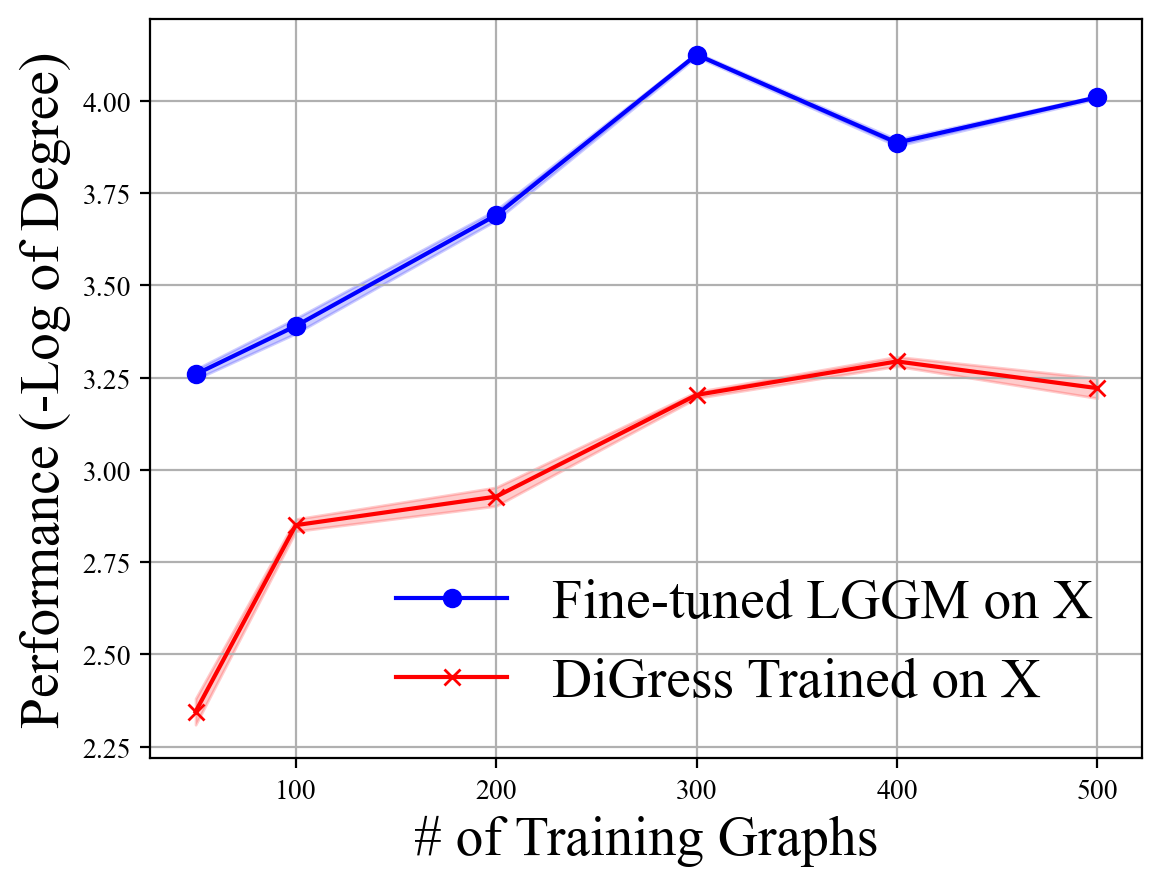}
         \caption{Email-DEG}
         \label{fig-ft-varying-number-email-deg-uniform}
     \end{subfigure}
     \hfill
     \begin{subfigure}[b]{0.24\textwidth}
         \centering
         \includegraphics[width=\textwidth]{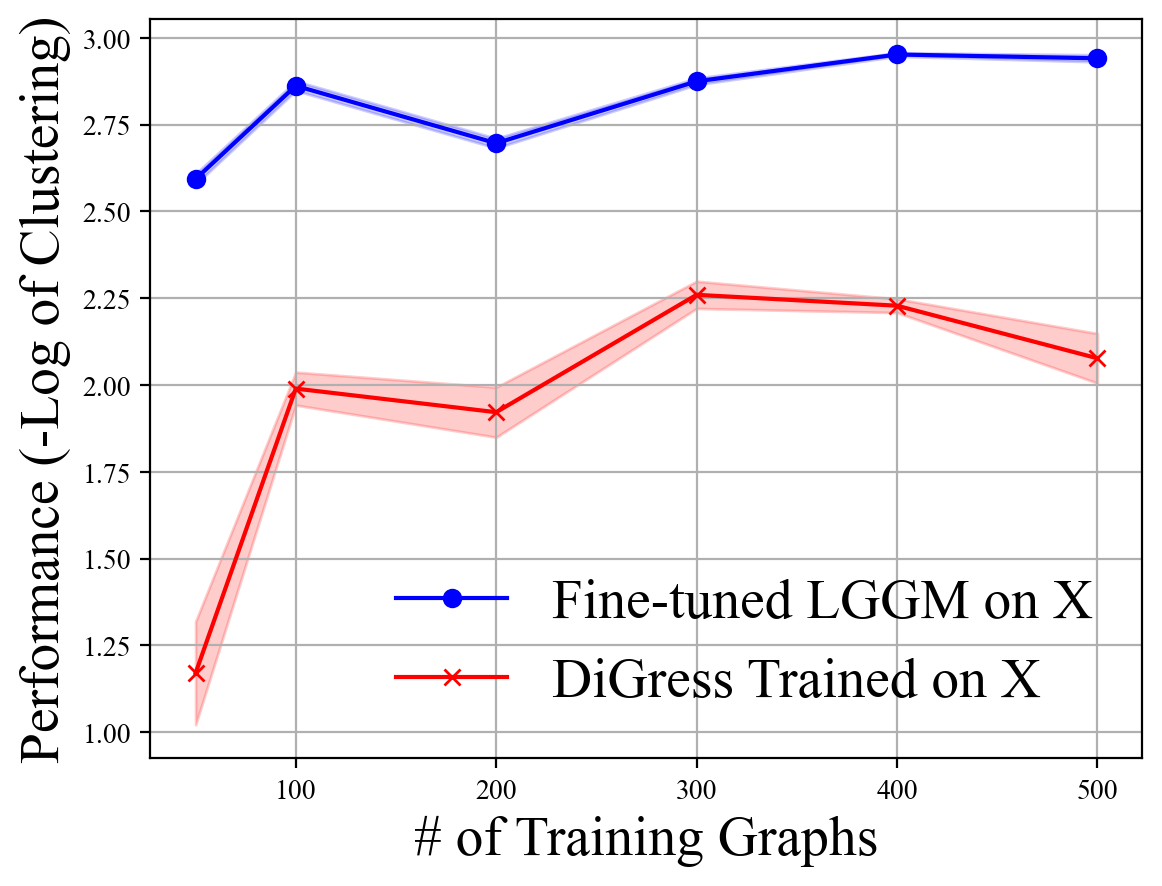}
         \caption{Email-CC}
         \label{fig-ft-varying-number-email-cc-uniform}
     \end{subfigure}
     \hfill
     \begin{subfigure}[b]{0.24\textwidth}
         \centering
         \includegraphics[width=\textwidth]{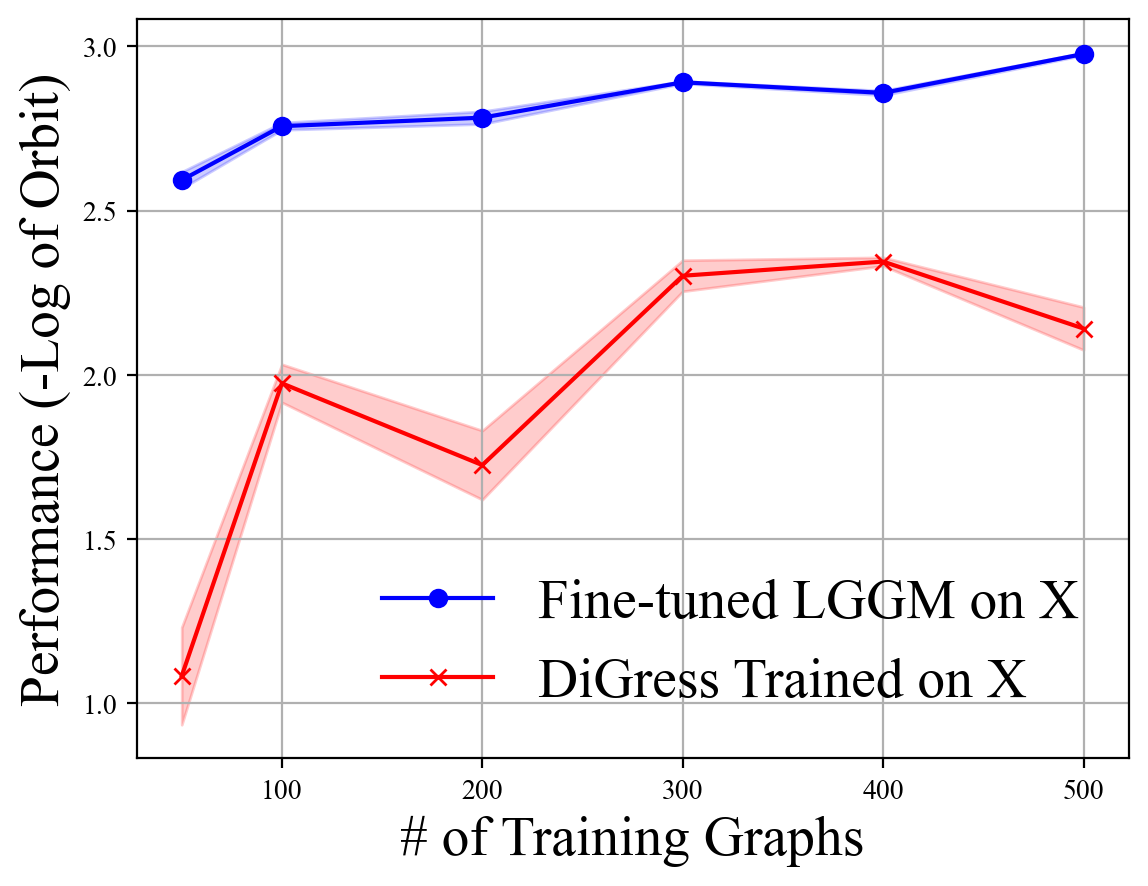}
         \caption{Email-Orb}
         \label{fig-ft-varying-number-email-orb-uniform}
     \end{subfigure}
     \hfill
     \begin{subfigure}[b]{0.24\textwidth}
         \centering
         \includegraphics[width=\textwidth]{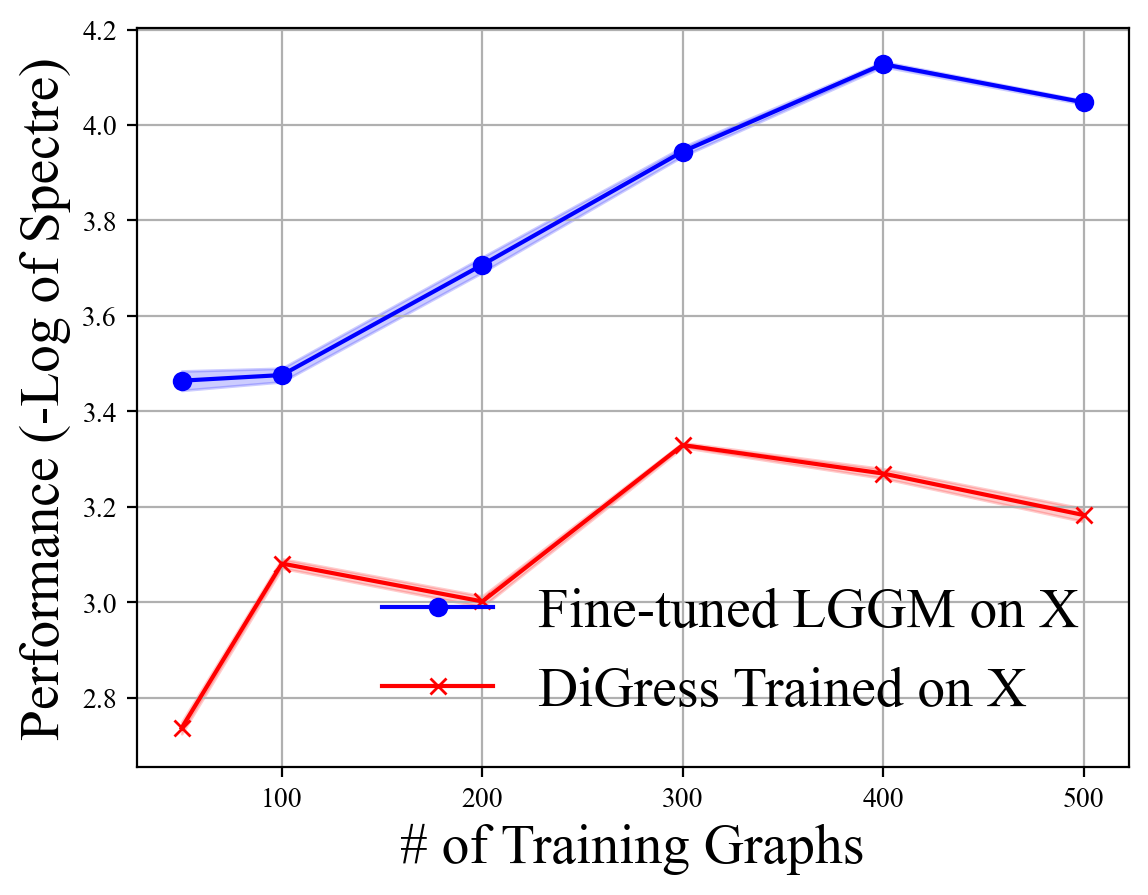}
         \caption{Email-Spec}
         \label{fig-ft-varying-number-email-spec-uniform}
     \end{subfigure}
    \caption{Effect of Number of Training Graphs on Email Networks.}
    \label{fig-ft-varyingnumber-email-uniform}
\end{figure*}

\begin{figure*}[htbp!]
     \centering
     \begin{subfigure}[b]{0.24\textwidth}
         \centering
         \includegraphics[width=\textwidth]{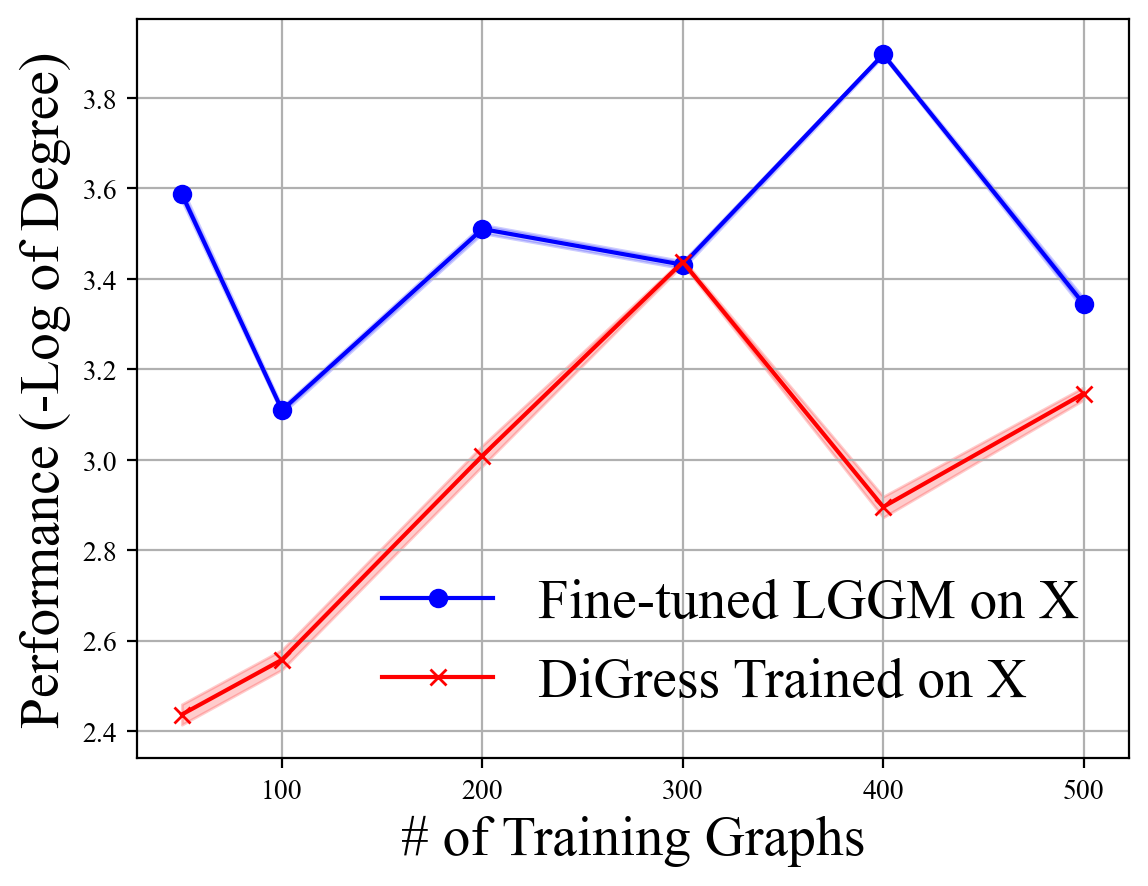}
         \caption{Web-DEG}
         \label{fig-ft-varying-number-web-deg-uniform}
     \end{subfigure}
     \hfill
     \begin{subfigure}[b]{0.24\textwidth}
         \centering
         \includegraphics[width=\textwidth]{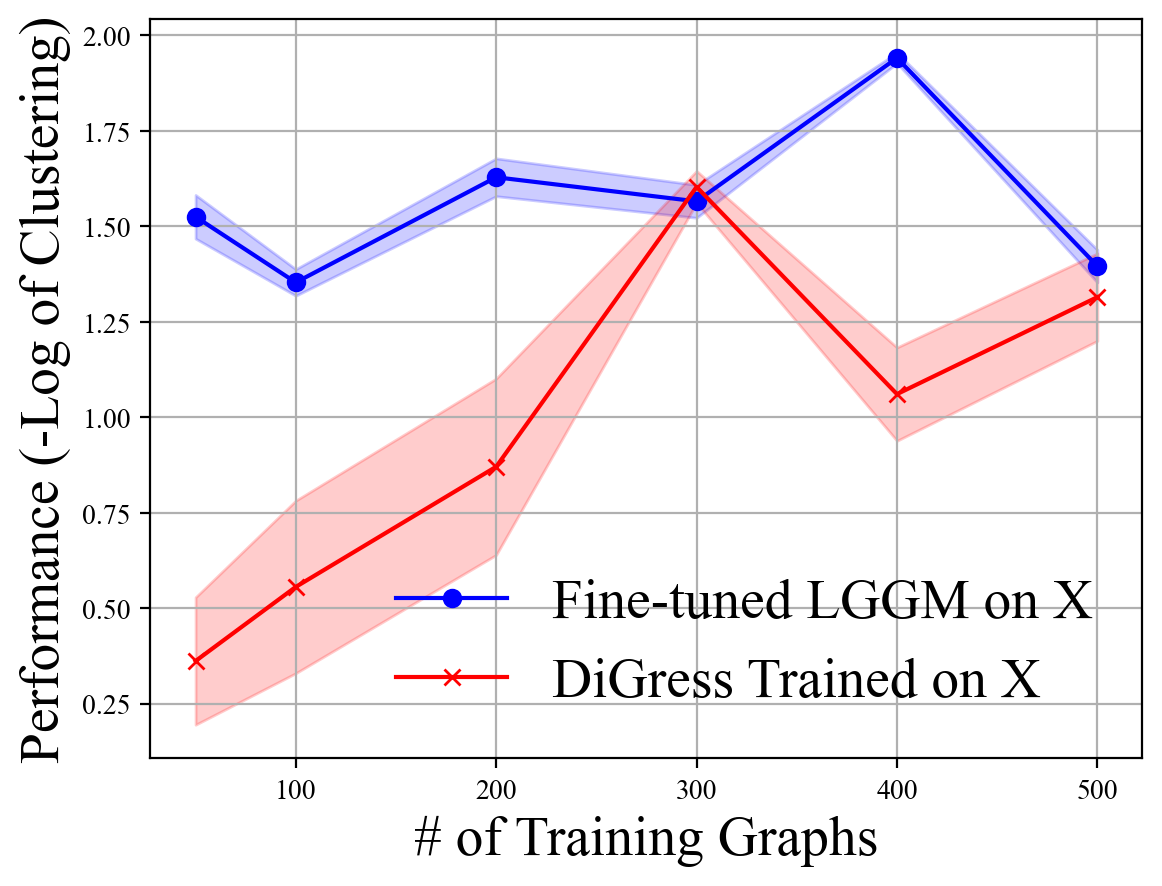}
         \caption{Web-CC}
         \label{fig-ft-varying-number-web-cc-uniform}
     \end{subfigure}
     \hfill
     \begin{subfigure}[b]{0.24\textwidth}
         \centering
         \includegraphics[width=\textwidth]{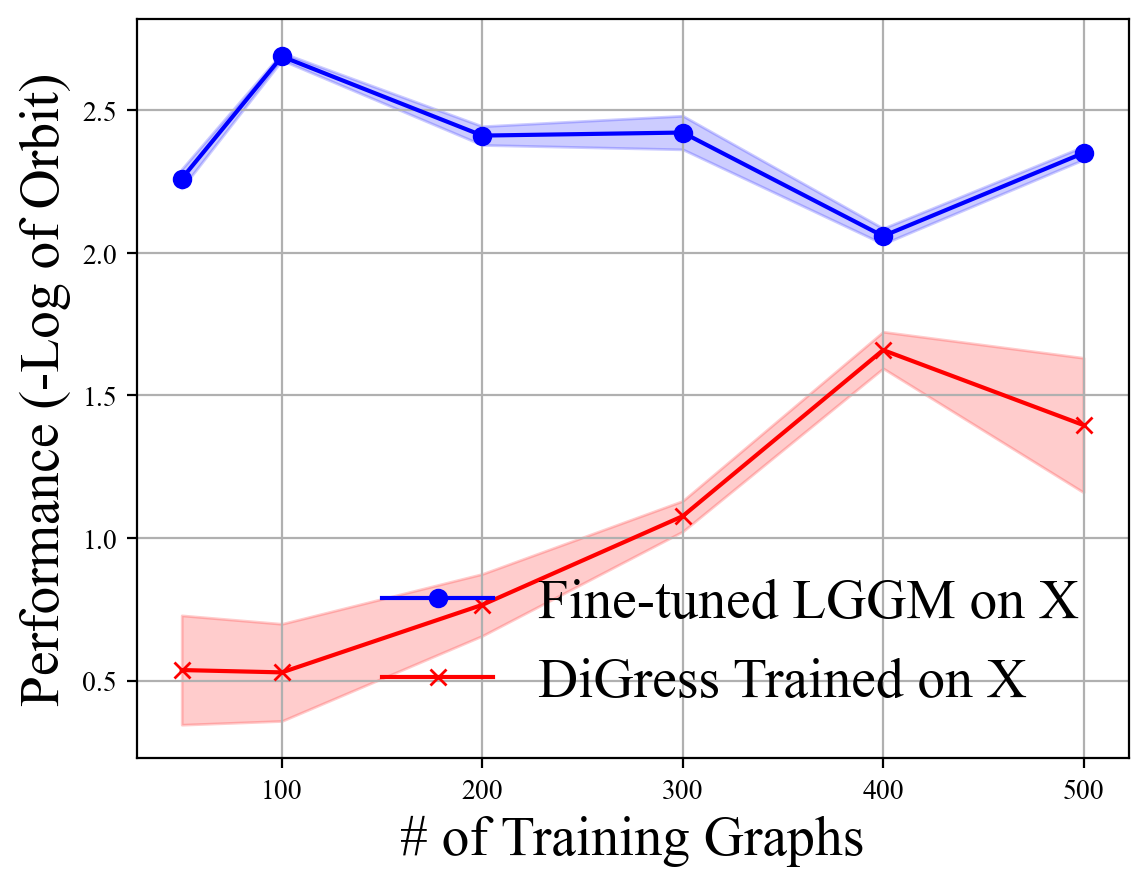}
         \caption{Web-Orb}
         \label{fig-ft-varying-number-web-orb-uniform}
     \end{subfigure}
     \hfill
     \begin{subfigure}[b]{0.24\textwidth}
         \centering
         \includegraphics[width=\textwidth]{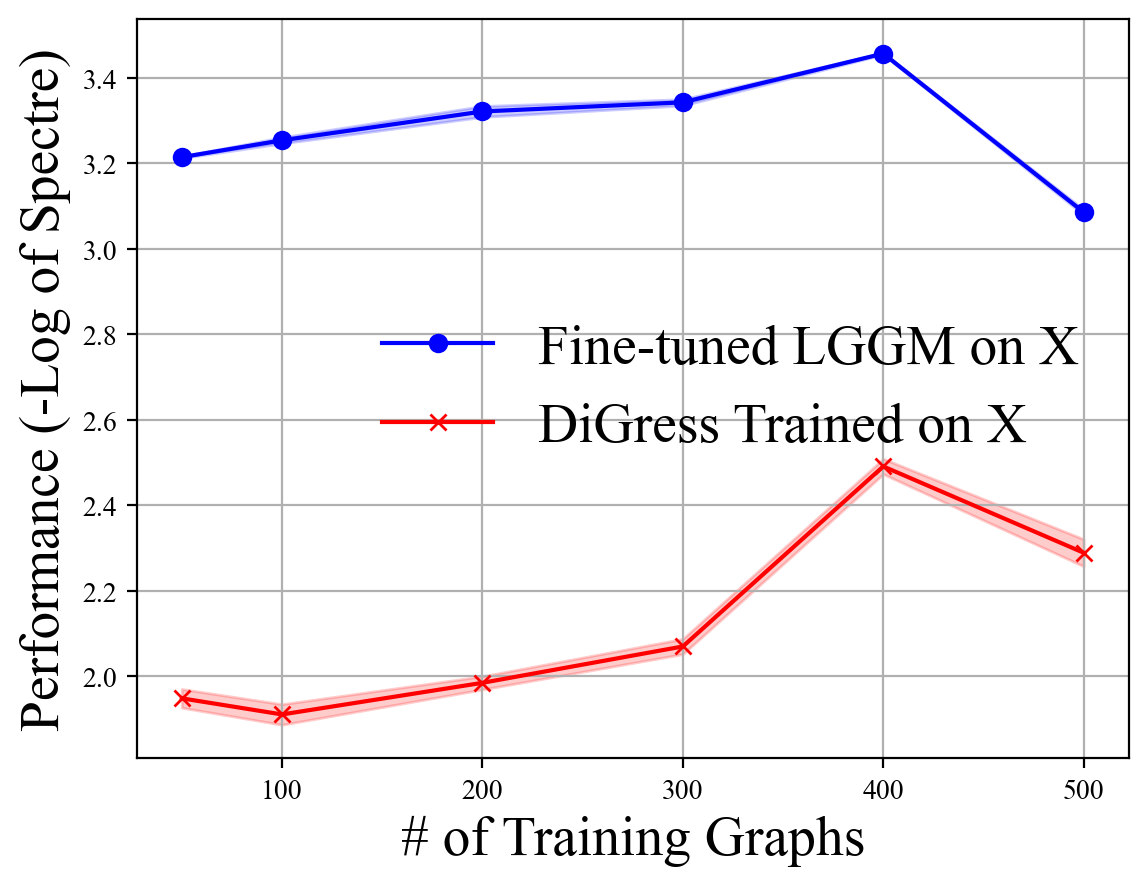}
         \caption{Web-Spec}
         \label{fig-ft-varying-number-web-spec-uniform}
     \end{subfigure}
    \caption{Effect of Number of Training Graphs on Web Graphs.}
    \label{fig-ft-varyingnumber-web-uniform}
\end{figure*}

\begin{figure*}[htbp!]
     \centering
     \begin{subfigure}[b]{0.24\textwidth}
         \centering
         \includegraphics[width=\textwidth]{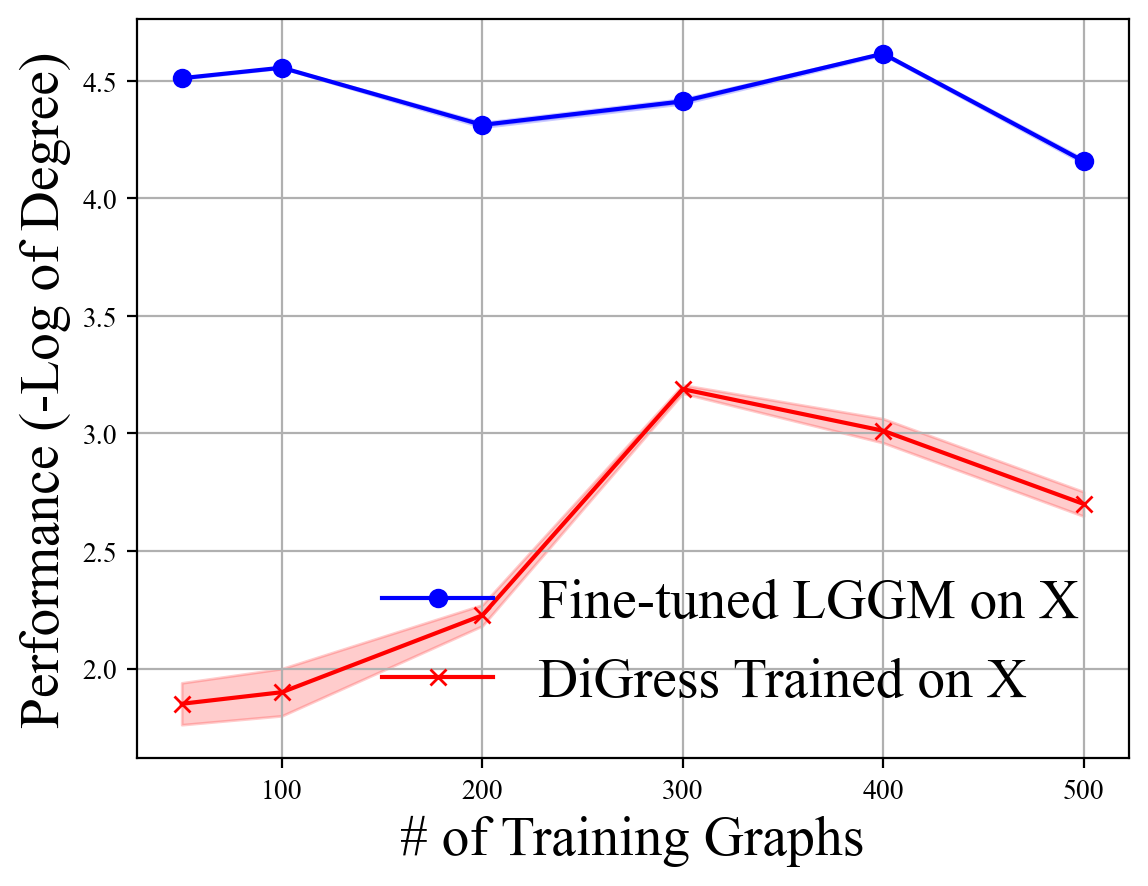}
         \caption{Facebook-DEG}
         \label{fig-ft-varying-number-fb-deg-uniform}
     \end{subfigure}
     \hfill
     \begin{subfigure}[b]{0.24\textwidth}
         \centering
         \includegraphics[width=\textwidth]{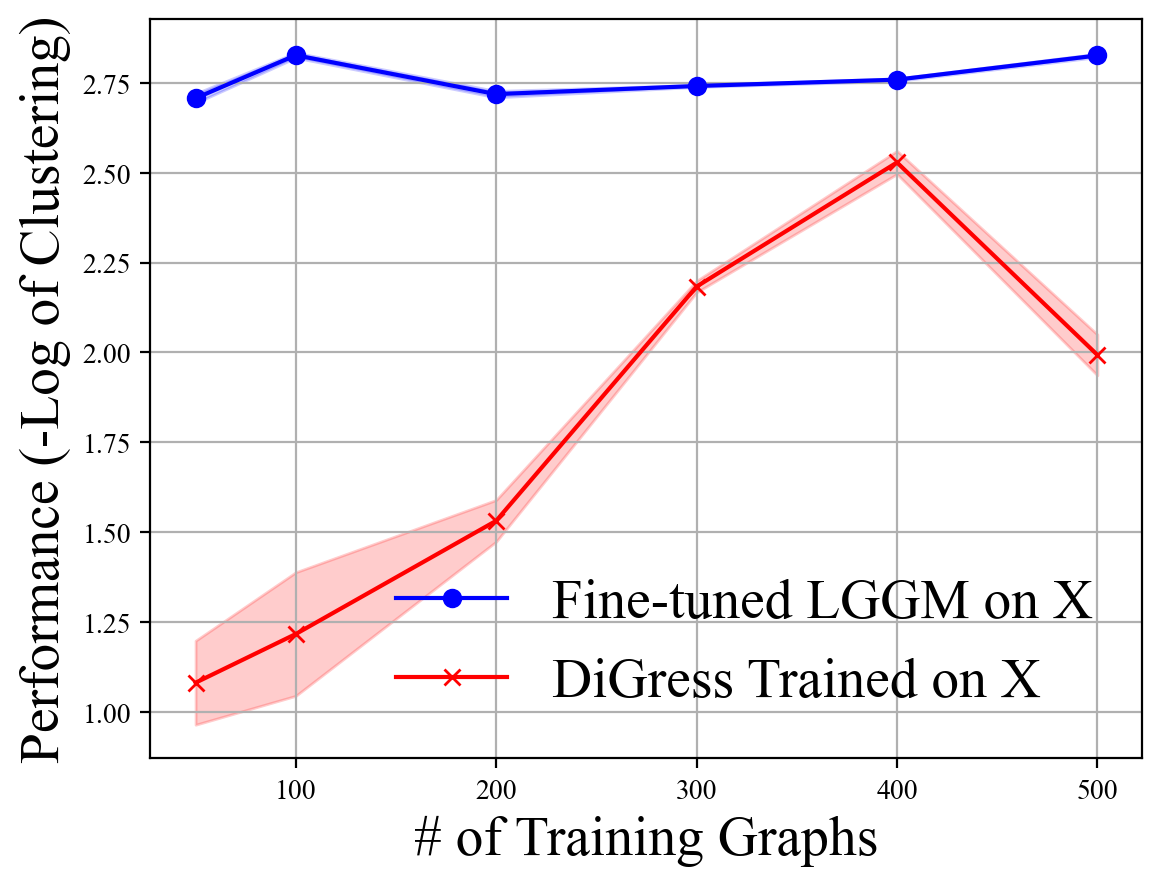}
         \caption{Facebook-CC}
         \label{fig-ft-varying-number-fb-cc-uniform}
     \end{subfigure}
     \hfill
     \begin{subfigure}[b]{0.24\textwidth}
         \centering
         \includegraphics[width=\textwidth]{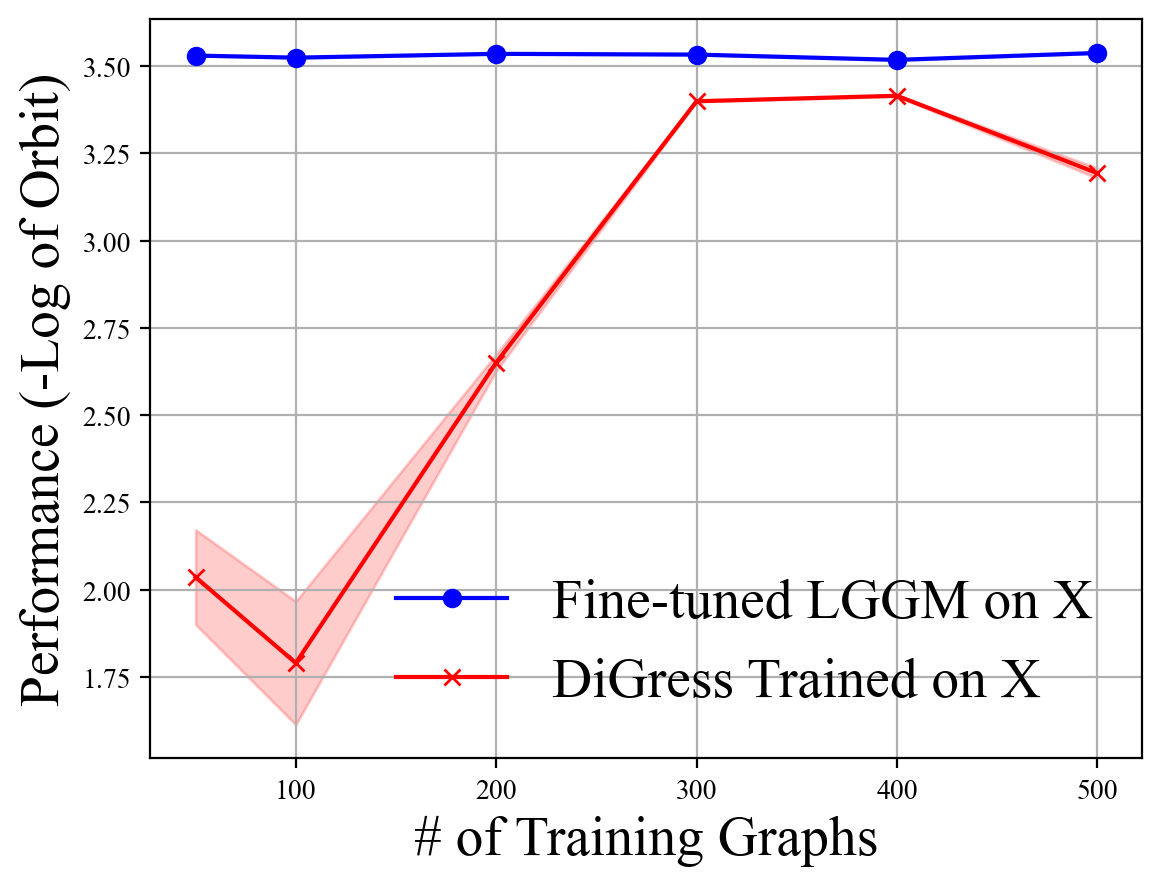}
         \caption{Facebook-Orb}
         \label{fig-ft-varying-number-fb-orb-uniform}
     \end{subfigure}
     \hfill
     \begin{subfigure}[b]{0.24\textwidth}
         \centering
         \includegraphics[width=\textwidth]{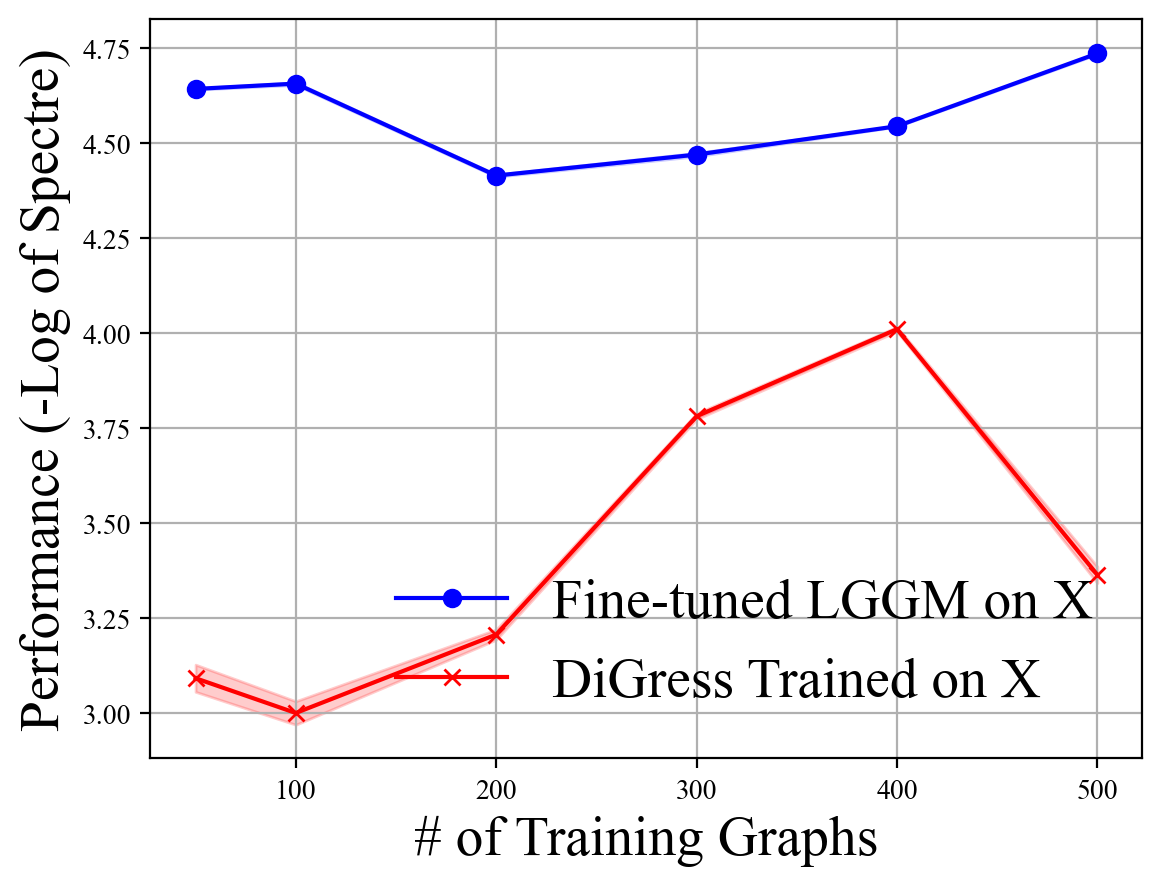}
         \caption{Facebook-Spec}
         \label{fig-ft-varying-number-fb-spec-uniform}
     \end{subfigure}
    \caption{Effect of Number of Training Graphs on Facebook Networks.}
    \label{fig-ft-varyingnumber-fb-uniform}
\end{figure*}

\subsection{Comparing the Domain as the Textual Condition before/after shuffling}\label{app-text2graph-shuffle}
\subsubsection{Domain Specific Transition}
\begin{table*}[htbp!]
\scriptsize
\setlength\tabcolsep{4.5pt}
\caption{Comparing the performance of graph generation between LGGM trained on graphs from all domains with and without domain/name as textual conditions under domain-specific transition.}
\label{app-table-text2graph-shuffle-marginal}
\centering
\begin{tabular}{llcccc|llcccc}
\toprule
\textbf{Domain} & \textbf{Method} & \textbf{DEG} & \textbf{CC} & \textbf{Spec} & \textbf{Orb} & \textbf{Domain} & \textbf{Method} & \textbf{DEG} & \textbf{CC} & \textbf{Spec} & \textbf{Orb} \\
\toprule
\multirow{2}{*}{\textsc{FB}} & LGGM-T2G$^\text{D}$ & \textbf{0.1533} & \textbf{0.1894} & \textbf{0.0817} & \textbf{0.0492} & \multirow{2}{*}{\textsc{BIO}} & LGGM-T2G$^\text{D}$ & \textbf{0.1313} & \textbf{0.5111} & \textbf{0.1340} & \textbf{0.3736}\\
& LGGM-T2G$^\text{D}$* & 0.2323 & 0.2618 & 0.1590 & 0.0923 &  & LGGM-T2G$^\text{D}$* & 0.1762 & 0.5887 & 0.1460 & 0.4929\\
\midrule
\multirow{2}{*}{\textsc{ASN}} & LGGM-T2G$^\text{D}$ & \textbf{0.0429} & \textbf{0.4742} & \textbf{0.0949} & \textbf{0.0401} & \multirow{2}{*}{\textsc{ECON}} & LGGM-T2G$^\text{D}$ & 0.2346 & \textbf{0.1572} & \textbf{0.1550} & \textbf{0.0579}\\
& LGGM-T2G$^\text{D}$* & 0.0891 & 0.5725 & 0.1446 & 0.0610 & & LGGM-T2G$^\text{D}$* & \textbf{0.2029} & 0.3393 & 0.2298 & 0.0579\\
\midrule
\multirow{2}{*}{\textsc{Email}} & LGGM-T2G$^\text{D}$ & \textbf{0.0874} & \textbf{0.3238} & \textbf{0.1472} & \textbf{0.2869} & \multirow{2}{*}{\textsc{RT}} & LGGM-T2G$^\text{D}$ & \textbf{0.0050} & \textbf{0.0940} & \textbf{0.0415} & \textbf{0.2870}\\
& LGGM-T2G$^\text{D}$* & 0.2169 & 0.7497 & 0.2825 & 0.8397 & & LGGM-T2G$^\text{D}$* & 0.0240 & 0.1023 & 0.1374 & 0.4123\\
\midrule
\multirow{2}{*}{\textsc{Web}} & LGGM-T2G$^\text{D}$ & \textbf{0.1253} & \textbf{0.9088} & \textbf{0.1156} & \textbf{0.3884} & \multirow{2}{*}{\textsc{Col}} & LGGM-T2G$^\text{D}$ & \textbf{0.1301} & \textbf{0.9384} & \textbf{0.1963} & \textbf{0.2032}\\
& LGGM-T2G$^\text{D}$* & 0.1464 & 0.9776 & 0.1460 & 0.4211 & & LGGM-T2G$^\text{D}$* & 0.1529 & 0.9684 & 0.2313 & 0.2089\\
\midrule
\multirow{2}{*}{\textsc{ROAD}} & LGGM-T2G$^\text{D}$ & \textbf{0.0112} & \textbf{0.1611} & \textbf{0.0298} & \textbf{0.0120} & \multirow{2}{*}{\textsc{Eco}} & LGGM-T2G$^\text{D}$ & \textbf{0.0575} & \textbf{0.2976} & \textbf{0.0585} & 0.2580\\
& LGGM-T2G$^\text{D}$* & 0.0365 & 0.2430 & 0.0605 & 0.0500 & & LGGM-T2G$^\text{D}$* & 0.1964 & 0.3330 & 0.1438 & \textbf{0.2574}\\
\midrule
\multirow{2}{*}{\textsc{Power}} & LGGM-T2G$^\text{D}$ & \textbf{0.0194} & \textbf{0.6031} & \textbf{0.0286} & \textbf{0.0193} & \multirow{2}{*}{\textsc{Citation}} & LGGM-T2G$^\text{D}$ & 0.1636 & \textbf{0.8868} & 0.2036 & 0.6142\\
& LGGM-T2G$^\text{D}$* & 0.0434 & 0.6721 & 0.0626 & 0.0231 & & LGGM-T2G$^\text{D}$* & \textbf{0.1615} & 0.9553 & \textbf{0.1903} & \textbf{0.6078}\\
\midrule
\multirow{2}{*}{\textsc{All}} & LGGM-T2G$^\text{D}$ & \textbf{0.0968} & \textbf{0.4621} & \textbf{0.1072} & \textbf{0.2158} & & & & & \\
& LGGM-T2G$^\text{D}$* & 0.1399 & 0.5636 & 0.1611 & 0.2937 & & & & & \\
\bottomrule
\end{tabular}

\end{table*}

\subsubsection{Uniform Transition}

\begin{table*}[htbp!]
\scriptsize
\setlength\tabcolsep{4.5pt}
\caption{Comparing the performance of graph generation between LGGM trained on graphs from all domains with and without domain/name as textual conditions under uniform transition strategy.}
\label{app-table-text2graph-shuffle-uniform}
\centering
\begin{tabular}{llcccc|llcccc}
\toprule
\textbf{Domain} & \textbf{Method} & \textbf{DEG} & \textbf{CC} & \textbf{Spec} & \textbf{Orb} & \textbf{Domain} & \textbf{Method} & \textbf{DEG} & \textbf{CC} & \textbf{Spec} & \textbf{Orb} \\
\toprule
\multirow{2}{*}{\textsc{FB}} & LGGM-T2G$^\text{D}$ & \textbf{0.1561} & \textbf{0.1639} & \textbf{0.0924} & \textbf{0.0417} & \multirow{2}{*}{\textsc{BIO}} & LGGM-T2G$^\text{D}$ & \textbf{0.0099} & \textbf{0.1286} & \textbf{0.0303} & \textbf{0.1366}\\
& LGGM-T2G$^\text{D}$* & 0.3018 & 0.4207 & 0.2069 & 0.2622 &  & LGGM-T2G$^\text{D}$* & 0.0754 & 0.2889 & 0.0881 & 0.2783\\
\midrule
\multirow{2}{*}{\textsc{ASN}} & LGGM-T2G$^\text{D}$ & \textbf{0.0318} & 0.2821 & \textbf{0.0606} & \textbf{0.0631} & \multirow{2}{*}{\textsc{ECON}} & LGGM-T2G$^\text{D}$ & \textbf{0.0665} & \textbf{0.0594} & \textbf{0.0650} & \textbf{0.0586}\\
& LGGM-T2G$^\text{D}$* & 0.0637 & \textbf{0.1561} & 0.1416 & 0.2351 & & LGGM-T2G$^\text{D}$* & 0.1035 & 0.0736 & 0.0971 & 0.0922\\
\midrule
\multirow{2}{*}{\textsc{Email}} & LGGM-T2G$^\text{D}$ & \textbf{0.0469} & \textbf{0.0982} & \textbf{0.0484} & \textbf{0.0505} & \multirow{2}{*}{\textsc{RT}} & LGGM-T2G$^\text{D}$ & \textbf{0.0468} & \textbf{0.0955} & \textbf{0.0729} & \textbf{0.0393}\\
& LGGM-T2G$^\text{D}$* & 0.1107 & 0.2322 & 0.1315 & 0.1692 & & LGGM-T2G$^\text{D}$* & 0.1399 & 0.3913 & 0.2441 & 0.2497\\
\midrule
\multirow{2}{*}{\textsc{Web}} & LGGM-T2G$^\text{D}$ & \textbf{0.0255} & \textbf{0.0737} & \textbf{0.0354} & \textbf{0.1856} & \multirow{2}{*}{\textsc{Col}} & LGGM-T2G$^\text{D}$ & \textbf{0.0395} & \textbf{0.3110} & \textbf{0.1146} & \textbf{0.1823}\\
& LGGM-T2G$^\text{D}$* & 0.0485 & 0.0830 & 0.1340 & 0.2669 & & LGGM-T2G$^\text{D}$* & 0.0323 & 0.4972 & 0.1159 & 0.5375\\
\midrule
\multirow{2}{*}{\textsc{ROAD}} & LGGM-T2G$^\text{D}$ & \textbf{0.0088} & 0.1225 & \textbf{0.0399} & \textbf{0.0155} & \multirow{2}{*}{\textsc{Eco}} & LGGM-T2G$^\text{D}$ & \textbf{0.2160} & \textbf{0.2917} & \textbf{0.1203} & \textbf{0.2569}\\
& LGGM-T2G$^\text{D}$* & 0.0453 & \textbf{0.1005} & 0.1257 & 0.3803 & & LGGM-T2G$^\text{D}$* & 0.3722 & 0.3210 & 0.2226 & 0.2771\\
\midrule
\multirow{2}{*}{\textsc{Power}} & LGGM-T2G$^\text{D}$ & \textbf{0.0162} & \textbf{0.1131} & \textbf{0.0479} & \textbf{0.1786} & \multirow{2}{*}{\textsc{Citation}} & LGGM-T2G$^\text{D}$ & \textbf{0.0101} & \textbf{0.1025} & \textbf{0.0315} & \textbf{0.0651}\\
& LGGM-T2G$^\text{D}$* & 0.0225 & 0.1533 & 0.1264 & 0.2957 & & LGGM-T2G$^\text{D}$* & 0.0375 & 0.2454 & 0.0699 & 0.1363\\
\midrule
\multirow{2}{*}{\textsc{All}} & LGGM-T2G$^\text{D}$ & \textbf{0.0562} & \textbf{0.1535} & \textbf{0.0633} & \textbf{0.1061} & & & & & \\
& LGGM-T2G$^\text{D}$* & 0.1128 & 0.2469 & 0.1420 & 0.2650 & & & & & \\
\bottomrule
\end{tabular}

\end{table*}